\PassOptionsToPackage{numbers,sort&compress}{natbib}
\PassOptionsToPackage{unicode}{hyperref}
\PassOptionsToPackage{naturalnames}{hyperref}
\documentclass{article}

\usepackage{microtype}
\usepackage{graphicx}
\usepackage{subfigure}
\usepackage{caption}
\usepackage[table]{xcolor}
\usepackage{multirow}
\usepackage[normalem]{ulem}
\usepackage{xurl}
\usepackage{booktabs} 
\usepackage[preprint]{neurips_2026}
\usepackage[numbers]{natbib}
\usepackage[utf8]{inputenc} 
\usepackage[T1]{fontenc}    
\usepackage{hyperref}       
\usepackage{amsfonts}       
\usepackage{nicefrac}       

\usepackage{amsmath}
\usepackage{amssymb}
\usepackage{mathtools}
\usepackage{amsthm}

\usepackage[capitalize,noabbrev]{cleveref}
\usepackage[textsize=tiny]{todonotes}

\usepackage{silence}
\theoremstyle{plain}
\newtheorem{theorem}{Theorem}[section]

\newtheorem{lemma}[theorem]{Lemma}

\theoremstyle{definition}
\newtheorem{definition}[theorem]{Definition}

\theoremstyle{remark}

\title{How Controlling the Variance can Improve Training Stability of Sparsely Activated DNNs and CNNs}

\author{%
  Emily Dent, Jared Tanner\\
  Mathematical Institute\\
  University of Oxford\\
  Oxford, UK \\
  \texttt{{emily.dent@maths.ox.ac.uk, jared.tanner}@maths.ox.ac.uk}\\
}

\begin{document}

\maketitle




\begin{abstract}
    The Edge-of-Chaos (EoC) theory developed for the random initialization of deep networks allows more efficient training by both preserving information in the initial outputs of the network and minimising exploding or vanishing gradients through characterisation of the intermediate layers as Gaussian processes. This EoC theory provides formulae for the choice of the initialisation distribution variances of the weights and biases. For activations which are approximately linear around the origin, the EoC theory typically encourages the Gaussian process variance to converge towards zero with increasing depth. Here we consider the less studied setting of highly sparsity inducing activations where a large region of values near the origin are set to zero. In this setting we prove a new phenomenon whereby initialisations leading to larger fixed Gaussian processes are beneficial to training stability. This theory informs a new, yet simple, initialisation strategy that allows training DNNs and CNNs with as large as 90\% sparsity in the hidden layers.
\end{abstract}

\section{Introduction}

Training of DNNs and CNNs with activations that induce high degrees of sparsity, e.g. 85\% and greater, raises some novel initialisation challenges and nonlinear activation designs \cite{price_deep_2023}. In its simplest form, the Edge-of-Chaos (EoC) theory considers networks with pre-activation hidden layers given by $h^{(\ell)}=W^{(\ell)} \phi(h^{(\ell -1)}) +b^{(\ell)}$ and studies properties of $h^{(\ell)}$ as a function of $\phi(\cdot)$. This leads to a prescribed an initialisation strategy for the weights and biases. In particular, EoC theory describes the limiting distribution of $h^{(\ell)}$ as a mean zero Gaussian process with variance $\mathbb{E}\left[{h^{(\ell)}_i}^2\right]$ and explains how the variance evolves through layers $\ell$. 
This manuscript extends the EoC theory to the study of activations which are zero around the origin by showing for exemplar $\phi(\cdot)$ with control of the parameter $\mathbb{E}\left[{h^{(\ell)}_i}^2\right]$ it is possible to improve the training accuracy and stability of networks and enable even greater sparsity. 

Inspired by the infinite width limit Gaussian behaviour of the hidden layers, $h_i^{(\ell)} \sim \mathcal{N} \left( 0, q^{(\ell)} \right)$, and the supposition that the largest entries are most important, consider activations $\phi \left( \cdot \right)$ for which $\phi \left( z \right) = 0$ for $|z| \leq \tau$, where $z$ contains the active outputs $h$ and $\tau$ is a parameter determining the level of sparsity. For clarity of exposition we focus on one such illustrative example as in \cite{price_deep_2023}, $\text{CReLU}_{\tau, m}$. 
\begin{align}
\label{crelu_eq}
            \text{CReLU}_{\tau, m}(x) &=
            \begin{cases}
            0, & \text{if $x<\tau$}\\
            x-\tau, & \text{if $\tau\leq x \leq \tau+m$}\\
            m, & \text{if $x>\tau+m$},
            \end{cases}
\end{align}
Here $m$ is a hyper parameter chosen to balance the expressivity and ease of training for large valued $\tau$. 


 The phenomenon of training instability occurs for other similar activations which are zero around the origin. For exploration of further sparsifying activation functions, see \cref{app:var_map_cst} and \cref{app:further_sparsifying_activations}. Difficulties in training occur for the sparsifying activation functions as a result of instability in the empirical $\tilde{q}^{(\ell)}$. The EoC theory typically advocates that $q^{(\ell)}$ converges independently to a non-zero limit $q^*$ or decays slowly to zero. However, here a new phenomenon is shown whereby increasing $q^*$ both reduces the sensitivity of the scale of the back-propagated errors to the empirical hidden layer variance, as well as improves the stability of $\tilde{q}^{(\ell)}$ itself. Experiments show that increasing $q^*$ allows more consistent training of DNNs and CNNs with high sparsity levels, achieving sparsity up to 90\%.



\subsection{Related work}\label{sec:related_work}
  
There is a range of related work which mathematically characterises the behaviour of the hidden layers and proposes methods to improve stability and training of large neural networks. As mentioned, \cite{price_deep_2023} pioneered the results extended here. The remainder of this section reviews other methods to further stabilize information propagation on the Edge-of-Chaos, and some adjacent lines of work.


Traditional EoC theory presumes a rapid and stable convergence of $\mathbb{E} \left[ {h_i^{(\ell)}}^2 \right]$ to $q^*$ in depth and then focuses on higher order phenomena. In particular, through concentration of \emph{all} singular values of the input-output Jacobian towards 1, see \cite{pennington_resurrecting_2017, pennington_emergence_2018, xiao_dynamical_2018}. Specifically, these works analyse the variance of the spectrum of the input-output Jacobian $JJ^T$, denoted as $\sigma_{JJ^T}$. To better concentrate the spectrum of $JJ^T$ they propose methods so that $\sigma_{JJ^T} \rightarrow 0$. The first approach proposes taking $q^* \rightarrow 0$ to cause $\sigma_{JJ^T} \rightarrow 0$, \cite{pennington_emergence_2018}. In \cite{martens_rapid_2021, roberts_principles_2022, murray_activation_2022} they also advocate taking $q^* \rightarrow 0$, for certain activation functions. Further methodologies proposed in \cite{murray_activation_2022, martens_rapid_2021} show that information propagation is improved through a linearized region around the origin of the activation function.  These mentioned works are to the best of our knowledge the only other cases to consider varying $q^*$, all of which only consider the limit of a decreasing, or arbitrarily chosen $q^*$. A further in depth analysis of the EoC can be found in \cite{yaida_non-gaussian_2020, roberts_principles_2022, hanin_random_2023}. 

 There are a range of motivations for the particular case of activation sparsity, one such property is the immediate computational reducibility for entire rows in the matrix multiplication. An incentivizing example for this efficiency improvement comes from the identification of naturally occurring sparsity leading to brain efficiency improvements in neuroscience, \cite{attwell_energy_2001}. Separately, sparsity is also shown to improve the ability of models to generalize and learn \cite{frankle_lottery_2018, muthukumar_sparsity-aware_2023, awasthi_learning_2024}; a phenomenon colloquially known as the \emph{lottery ticket}. Due to this intuition, some recent empirical work considers activation sparsity in attention, see \cite{zhang_relu2_2024, luo_sparsing_2025, you_spark_2025}; in \cite{you_spark_2025} they find competitive activation sparsity up to 92\%.  Activation sparsity is one method to improve computational efficiency, however there is a range of other techniques including: compact architectures design \cite{howard_mobilenets_2017}, low-rank weight matrices \cite{osawa_accelerating_2017} such as in the popular LoRA framework \cite{hu_lora_2021}, low-precision quantisation \cite{weng_neural_2023} such as in recent GPU hardware, and pruning weight matrices such as \cite{blalock_what_2020}.

Extensions to the EoC initialisation have been explored for other architectures including ResNets \cite{yang_mean_2017} and Recurrent Neural Networks \cite{xie_slow_2025}. The perhaps natural extension of this work to ResNets is not explored for the precise reason that the skip connections would negate the induced sparsity in the activation function, and therefore the proposed efficiency improvements. Transformer architectures are also not considered here, as they are a separate line of analysis with additional complexity, for some recent work on information propagation in transformers see \cite{cowsik_geometric_2024}.


 An interesting parallel related literature, is the study of high variance weight initialisation leading to faster training, as a \emph{lazy training} method, this was first mentioned in \cite{chizat_lazy_2019}, and has further been explored in \cite{domine_exact_2023, domine_lazy_2024}.

\subsection{Main contributions}

We revisit EoC initialization in the setting of sparsity-inducing activation functions and identify the fixed-point variance $q^*$ as a critical, yet underexplored, parameter governing training stability. In this regime, the activation function lacks linearity around the origin, and consequently existing methodologies \cite{martens_rapid_2021, roberts_principles_2022, murray_activation_2022}, which typically advocate for $q^*\rightarrow0$, are not applicable. In contrast, we show that increasing $q^*$ can substantially improve trainability for activations that are zero in a neighborhood of the origin.
\cref{sec:eoc_theory} develops the theoretical basis for this effect, identifying key control metrics and deriving explicit characterizations for the exemplar activation $\text{CReLU}_{\tau, m}$. In \cref{sec:experiments}, experiments then show proof of concept for this theory.

\section{Edge-of-Chaos}\label{sec:eoc_theory}

The following first reviews the EoC theory. Explicitly considering the limiting behaviour for Deep Neural Networks (DNNs), a DNN has hidden layers with activation function $\phi$ given by,
\begin{equation}
\begin{aligned}
        \label{hidden_layer}
        h_i^{(\ell)} &= \sum_{j=1}^{N_l}{W_{ij}}^{(\ell)} x_j^{(\ell-1)} + b_i^{(\ell)},  \\
        x_i^{(\ell)} &= \phi \left( h_i^{(\ell)} \right),
\end{aligned}
\end{equation}
where $\ell$ is the layer number, $N_{\ell}$ is the width of layer $\ell$, and ${W_{ij}}^{(\ell)}$ and $b_i^{(\ell)}$ are the weights and biases for layer $\ell$ respectively.  Much of the analysis for Convolutional Neural Networks (CNNs) is similar, \cite{xiao_dynamical_2018}.
The Edge-of-Chaos initialisation is then given by $W_{ij}^{(\ell)} \sim \mathcal{N} \left( 0, \frac{\sigma_w^2}{N_{\ell}} \right)$ and $b_i^{(\ell)} \sim \mathcal{N} \left( 0, \sigma_b^2 \right)$, where $\sigma_w^2$ and $\sigma_b^2$ are selected according to roles that are known to improve information propagation in deep neural networks, \cite{schoenholz_deep_2017}. 
 \\
 In the infinite width limit these hidden layers are known to behave as Gaussian random variables \cite{neal_bayesian_1996, matthews_gaussian_2018, lee_deep_2018}, further \cite{poole_exponential_2016} characterised this initialisation admitting hidden layers with distribution $h_i^{(\ell)} \sim \mathcal{N} \left( 0, q^{(\ell)} \right)$, with $q^{(\ell)}$ evolving through the iterative mapping,
 \begin{align}
     q^{(\ell)} &= V \left( q^{(\ell-1)} \right):= {\sigma_w^2}\int_{\mathbb{R}} \left[ \phi \left(\sqrt{q^{(\ell-1)}}z \right) \right] ^2\gamma(dz) + \sigma_b^2,  \label{eq:vmap_general}
 \end{align}
where $\gamma \left( dz \right)=  e^{-z^2/2}/ \sqrt{2 \pi}$.
In a stable regime, $V(q)$ admits a fixed point $q^*$, and for $V'(q^*) < 1$, $V \left(q^{(\ell)} \right) \rightarrow q^*$ as $l \rightarrow \infty$. 

Further consider the correlation map between two inputs, 
\begin{align}
    \rho^{(\ell)} &= R_\phi(\rho^{(\ell-1)}) := \frac{1}{q^*}\biggl({\sigma_w^2}\int_{\mathbb{R}}\int_{\mathbb{R}}  \phi \left(u_1 \right) \phi \left(u_2 \right)\gamma(dz_1)\gamma(dz_2)  + \sigma_b^2\biggr), \label{correlation_map}
\end{align}
where, $u_1 = \sqrt{q^*}z_1$, $u_2 = \sqrt{q^*}\left( \rho^{(\ell-1)}z_1 + \sqrt{1-{\rho^{(\ell-1)}}^2}z_2 \right)$, $\rho^{(\ell-1)} = {q^{(\ell-1)}_{ab}}/\sqrt{{q^{(\ell-1)}_{aa}q^{(\ell-1)}_{bb}}}$.
$\rho^{(\ell)}$ is the correlation at layer $\ell$, and $z_1$, $z_2$ are independent Gaussian random variables.
The correlation map \eqref{correlation_map} will always admit a fixed point at $\rho=1$. The EoC initialisation is defined by the derivative of the correlation map $R_\phi$ at $1$;
\begin{equation}
\label{chi_general_eq}
    {\chi}_1 \left(q^* \right)  = R'_{\phi}\left( 1 \right) := \sigma_w^2 \int \left[ \phi' \left( \sqrt{q^*}z \right) \right]^2 \gamma \left(dz \right),
\end{equation}
being equal to 1. 

$\chi_1(q^*)$ is also equal to the (first-order) multiplicative perturbation across a single layer, \cite{poole_exponential_2016}; therefore, $\chi_1(q^*) > 1$ results in exploding perturbations and in deeper neural networks $\chi_1(q^*) < 1$ is overly conservative, driving the network towards a constant function, whereas setting $\chi_1(q^*) = 1$, the Edge-of-Chaos, results in stable perturbations, and was shown to prevent vanishing/exploding gradients \cite{schoenholz_deep_2017}. Further study of $\chi_1$ can be found in \cite{poole_exponential_2016, schoenholz_deep_2017, xiao_dynamical_2018}, see also \cref{app_further_on_chi}.

Typically, the literature assumes $q^{(\ell)} \rightarrow q^*$ almost immediately and so $\chi_{1, \phi}^{(\ell)}$ is assumed equal for all $\ell$ and takes value \eqref{chi_general_eq}. However, this is not reasonably the case in practice, and $q^{(\ell)}$ is observed to vary stochastically around the point $q^*$. Hence,

\begin{equation}
    \chi_{1, \phi}^{(\ell)} ={\chi}_{1, \phi}\left(q^{(\ell)}\right) := \sigma_w^2 \int \left[ \phi' \left( \sqrt{{q}^{(\ell)}}z \right) \right]^2 \gamma \left(dz \right),
\end{equation}
is a function of $q^{(\ell)}$, any lack of symmetry about $q^*$ in the dependence of $\chi_1$ as a function of $q^{(\ell)}$ may therefore create a growth in $\prod_{\ell=1}^L {\chi}_{1, \phi}\left(q^{(\ell)}\right)$ which can be exponential in depth and directly informs the accumulated growth of errors through the network. 

\begin{definition}[Training Stability of a Neural Network in the Large Width Limit]{\label{def:training_stability}}
    A neural network is considered stable during training if $\prod_{\ell=1}^L\chi_{1, \phi}^{(\ell)}$ is order one independent of $L$; as from the critical stability condition from the EoC regime, $\chi_1=1$. That is the multiplicative expected squared singular value of the input-output Jacobian through all layers of the network is of order 1.
\end{definition}

The process of choosing parameters for the EoC typically follows as 1. Let $q^*=1$, or some other arbitrary value, 2. Set $\chi_1(q^*) = 1$ to find $\sigma_w^2$ from \eqref{chi_general_eq}, 3. Set $q^* = V(q^*)$ to calculate $\sigma_b^2$ from \eqref{eq:vmap_general}. 

We propose simply increasing $q^*$ from 1 to 3 with no adverse affects. In particular, for our activation function $\text{CReLU}_{\tau, m}$ we have the additional parameters $\tau, m$, as such the complete initialization scheme is:
\begin{enumerate}
    \item Set $q^*$ (e.g. 3), and desired sparsity $s$ (e.g. 0.85).
    \item Calculate $\tau$, e.g. for $\text{CReLU}$ this is $\sqrt{q^*}\Phi^{-1}(s)$.
    \item Set $V'(q^*) = 0.7$ to calculate $m$ numerically using root finding methods, e.g. for $\text{CReLU}$ this is equation \eqref{eq:vprime_crelu}, where $\sigma_w^2 =  2/\left[\text{erf}\left(\frac{\tau+m}{\sqrt{2q}}\right)- \text{erf}\left(\frac{\tau}{\sqrt{2q}}\right)\right]$.
    \item Set $\sigma_w^2=1/\int [\phi'(z\sqrt{q^*})]^2 \gamma(dz)$ so that $\chi_{1, \phi}(q^*) = 1$.
    \item Set $\sigma_b^2=q^*-\sigma_w^2 \int [\phi(z\sqrt{q^*})]^2 \gamma(dz)$ so that $q^* = V(q^*)$.
    \item Initialize $W^{(\ell)}_{ij}\sim\mathcal{N}\left(0,\frac{\sigma_w^2}{N_\ell}\right)$ and $b_i^{(\ell)}\sim \mathcal{N}(0,\sigma_b^2)$.
\end{enumerate}

This is a demonstration for our exemplar sparsifying activation function $\text{CReLU}_{\tau, m}$, however see \cref{app:var_map_cst} and \cref{app:further_sparsifying_activations} for analysis of further activation functions.

 The benefit of an increased $q^*$ is a result of the dependence of $\chi_{1, \phi}^{(\ell)}$ on $q^{(\ell)}$, which in itself is determined by the mapping $q^{(\ell+1)} = V(q^{(\ell)})$, the training stability can be improved in the large width limit three-fold, as defined in Definition \ref{def:training_stability}. First, through refinement of the fixed-point mapping dynamics by improving the symmetry of $V(q^*)$, leading to a more even distribution of our empirical hidden layer variances. Second, by reducing the size of the finite dimensional correction to our variance map. Third, by decreasing the sensitivity of $\chi_{1, \phi}^{(\ell)}$ to $q^{(\ell)}$, that is $\chi'_{1, \phi}(q^{(\ell)})$. Sections \ref{sec:reducing_vprimeprime}, \ref{subsec:finite_dim_correction}, \ref{sec:sensitivity_of_chi}, respectively show these three control metrics can be achieved by simply increasing our $q^*$ at initialization through exposition of the sparsifying activation function $\phi = \text{CReLU}_{\tau, m}$.

\subsection{{\texorpdfstring{Concentrating $q^{(\ell)}$}{Concentrating q}}}\label{sec:reducing_vprimeprime}

As mentioned, in the deterministic view of the iterative variance map, $q^{(\ell)}$ will always converge to the fixed point where $V'(q^*)<1$,  under the Law of Large Numbers. However in actuality $q^{(\ell)}$ has variability due to the infinite limit assumption, and is a random variable itself. 
For larger values of $V''(q^*)$, where there is a strong loss of symmetry around the fixed point, the iterative map cannot sufficiently control the perturbations, see \cref{fig:vcrelu_vary_m}. The perturbations around $q^*$ are then asymmetrical, resulting in an unstable regime for the variance of the hidden layers $q^{(\ell)}$, since the $\prod_{\ell=1}^L{\chi}_{1, \phi}(q^{(\ell)})$ accumulates values biased to be greater than $1$, and grows exponentially with depth. This loss of symmetry constrains how large $m$ can be selected which limits the expressivity of the network before instability occurs, and therefore the lack of trainability found in \cite{price_deep_2023}.

\begin{figure}[htbp!]
    \centering
\subfigure[]{\includegraphics[height = 3.75cm, width=0.3\linewidth]{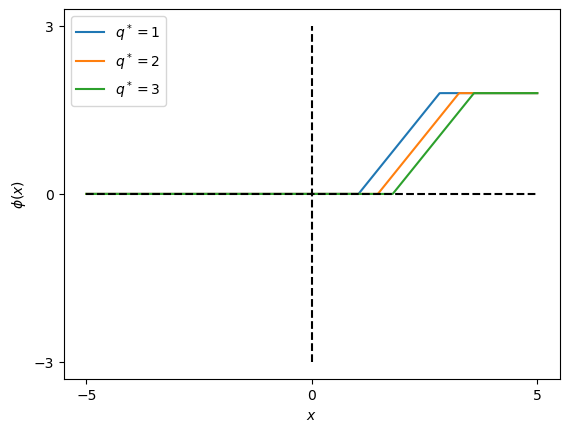}
    }
    \subfigure[]{\includegraphics[height = 3.75cm, width=0.3\linewidth]{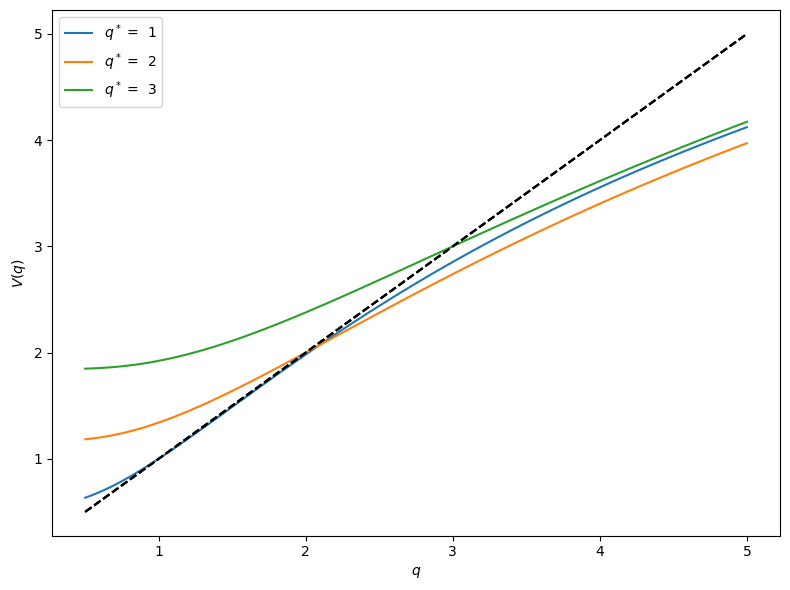}}
    \subfigure[]{\includegraphics[trim =0cm 0cm 0cm 0.7cm, clip, height = 3.75cm, width=0.3\linewidth]{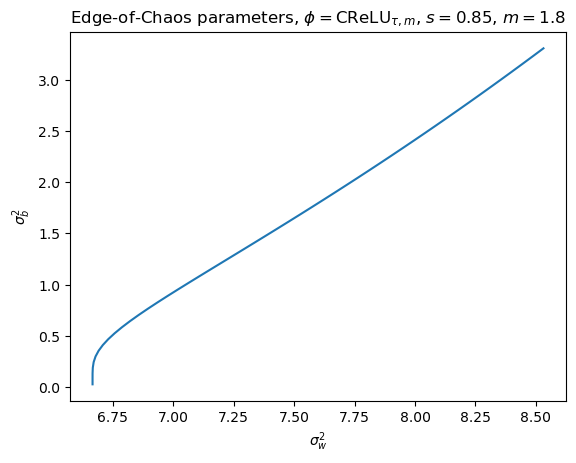}}
    \caption{(a) $\text{CReLU}_{\tau, m}(x)$, for $m=1.8$, $s=0.85$ with varying values of $q^*=\{ 1, 2, 3\}$, (b) $V_{\text{CReLU}_{\tau, m}}(q)$, for $m=1.8$, $s=0.85$ with varying values of $q^*=\{ 1, 2, 3\}$, with vertical axis is $V(q)$ and horizontal axis $q$. (c) EoC variances $\sigma_b^2$ against $\sigma_w^2$ for a range of $q^* \in [0.05, 3]$. In (b)  observe as $q^*$ increases the symmetry of $V(q)$ improves around the respective fixed point $q^*$.}
    \label{fig:vcrelu_vary_m}
\end{figure}


To improve the concentration of $q^{(\ell)}$, consider the variance map $V(q)$ and $V''(q^*)$ as a heuristic for the lack of symmetry. First see,

\begin{multline}
    V'_{\text{CReLU}_{\tau, m}}(q) = \sigma_w^2 \left[   \frac{1}{2} \text{erf}\left( \frac{\tau +m}{\sqrt{2q}} \right) - \frac{1}{2} \text{erf}\left( \frac{\tau }{\sqrt{2q}} \right)
    - \frac{m}{\sqrt{2 \pi q}} \text{exp}\left( -\frac{ (\tau +m)^2}{{2q}} \right)  \right]. \label{eq:vprime_crelu}
\end{multline}



Then the expression for $V''_{\text{CReLU}_{\tau, m}}(q^*)$ is,
\begin{equation}
    V''_{\text{CReLU}_{\tau, m}}(q) = \frac{\sigma_w^2}{\sqrt{8 \pi q^5}} \text{exp}\left( -\frac{ (\tau +m)^2}{{2q}} \right)
    \left[q\tau  {\exp\left({\frac{2m\tau+m^2}{2q}}\right)} 
     - q\tau- m(\tau+m)^2 \right].
\end{equation}
The exact derivation for $V''_{\text{CReLU}_{\tau, m}}(q)$ is uninformative, for interest see \cref{appendix_further_deriv_derivations}.

\begin{figure}[htbp!]
\centering
\subfigure[]{\includegraphics[height = 3.75cm, width=0.3\linewidth]{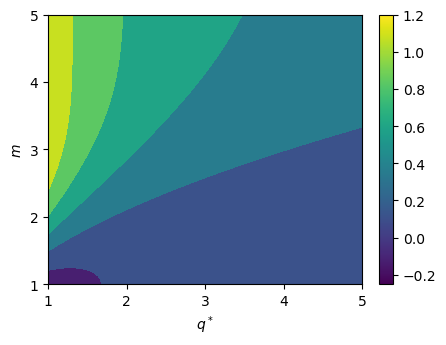}}
    \subfigure[]{\includegraphics[height = 3.75cm, width=0.3\linewidth]{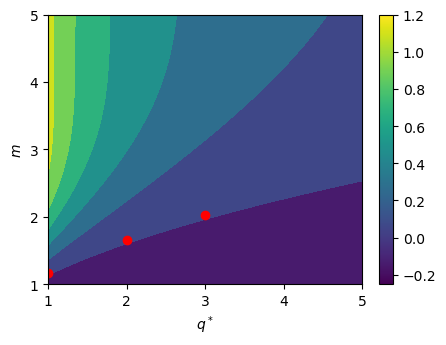}}
    \subfigure[]{\includegraphics[height = 3.75cm, width=0.3\linewidth]{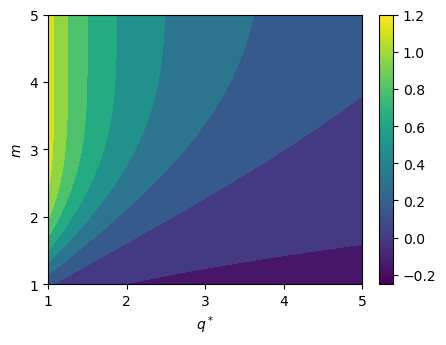}}
    \caption{$V''_{\text{CReLU}_{\tau, m}}(q^*)$, for a range of $s = \{0.8, 0.85, 0.9\}$, the plots of these three sparsities are from left to right, each with horizontal axis $q^*$ and vertical axis $m$. For fixed value $m$ by increasing $q^*$, $V''_{\text{CReLU}_{\tau, m}}(q^*)$ reduces, and this holds across all three sparsity levels.  The highlighted points in plot $s=0.85$ correspond to fixing $V'(q^*)=0.7$ for $q^*=\{ 1, 2, 3\}$, and have values 0.023, 0.013, 0.008 respectively.}
    \label{fig:vprimeprime_crelu_heatmap_multi_s}
\end{figure}

\cref{fig:vprimeprime_crelu_heatmap_multi_s} further shows that for a fixed sparsity and maximum value $m$, increasing $q^*$ decreases the curvature of the variance map for the chosen activation functions - further stabilizing the initialisation.

\subsection{Finite dimensional correction}\label{subsec:finite_dim_correction}

This section explicitly examines the finite dimensional correction of the variance map, as in \cite{roberts_principles_2022}. Let $\tilde{q}^{(\ell)}$ be the the finite dimensionally corrected variance.
In particular, consider the variance $\tilde{q}^{(\ell)}$ and the fourth moment $r^{(\ell)}$ of the hidden layers as an infinite series of sub-leading order corrections. That is,
\begin{align}
    \tilde{q}^{(\ell)} &= \tilde{q}^{\{0\}(\ell)} + \frac{1}{n_{\ell-1}}\tilde{q}^{\{1\}(\ell)} + \mathcal{O}\left( \frac{1}{n^{2}}\right), \\
    r^{(\ell)} &= r^{\{0\}(\ell)} + \frac{1}{n_{\ell-1}}r^{\{1\}(\ell)} + \mathcal{O}\left( \frac{1}{n^{2}}\right).
\end{align}


Denote $\tilde{q}^{\{1\}(\ell)}$ as the next-leading-order metric, which is the first-order correction to the variance. Where $r^{(\ell)}$ describes the size of fluctuations away from the mean variance $q^{(\ell)}$, as it controls the near-Gaussianity of the layer distribution. Note, the leading order term to $\tilde{q}^{(\ell)}$ fully describes the infinite-width limit to the pre-activation distribution, $q^{(\ell)} = \tilde{q}^{\{0\}(\ell)}$.

Now introducing the operator $\left\langle \cdot \right\rangle_{q}$ for ease of notation,
\begin{equation}
    \left\langle f(z) \right\rangle_{q} = \frac{1}{\sqrt{2\pi q}}\int_{\mathbb{R}} f(z)e^{-\frac{z^2}{2q}}dz.
\end{equation}
In \cite{roberts_principles_2022} they find the iterative mapping \eqref{eq:vmap_general} and the equivalent following recursions,
\begin{align}
    r^{(\ell+1)} &= V' \left( q^{(\ell)}\right)^2r^{(\ell)}+ \sigma_w^4 \left( \left\langle \phi^4(z) \right\rangle_{q^{(\ell)}} - \left\langle \phi^2(z) \right\rangle^2_{q^{(\ell)}}\right), \label{eq:fourth_mom_recursion}\\
    \tilde{q}^{\{1\}(\ell+1)} &=  V' \left( q^{(\ell)}\right) \tilde{q}^{\{1\}(\ell)} + \frac{1}{2} V''\left( q^{(\ell)}\right) r^{(\ell)} \label{eq:nlo_recursion}.
\end{align}

Initialising at the fixed point $q^{(1)} = q^*$, such that $q^* =V \left( q^{*}\right)$, gives $q^{(\ell)} = q^*$ for $l\geq 1$. These results allow the proof of the following \cref{thm:finite_dim_nlo}, for the full proof see \cref{appendix_finite_correction}.

\begin{theorem}
\label{thm:finite_dim_nlo}
Where the recursive relations \eqref{eq:vmap_general}, \eqref{eq:fourth_mom_recursion}, \eqref{eq:nlo_recursion} hold, assuming $0<V' \left( q^{*}\right) <1$, $\tilde{q}^{\{1\} (1)}=0$ and $r^{(1)}=0$, for $\ell \geq3$,
\begin{align}
    \left| \tilde{q}^{\{1\}(\ell)} \right| \leq \frac{\sigma_w^4}{2}\frac{\left|V''\left( q^{*}\right)\right|  \left| \left\langle \phi^4(z) \right\rangle_{q^{*}} - \left\langle \phi^2(z) \right\rangle^2_{q^{*}}\right|}{\left(1-V' \left( q^{*}\right)\right)^2\left(1+V' \left( q^{*}\right) \right)}. \label{abs_upper_bound_nlo}
\end{align}

\end{theorem}

\begin{figure}
    \centering
    \subfigure[]{\includegraphics[height = 3.75cm, width=0.3\linewidth]{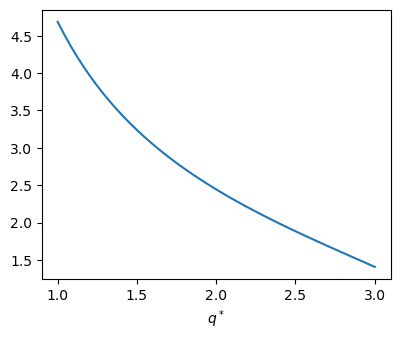}}
    \subfigure[]{\includegraphics[height = 3.75cm, width=0.3\linewidth]{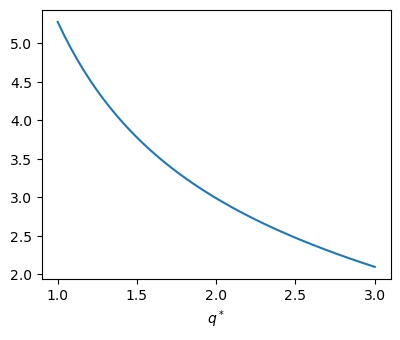}}
    \subfigure[]{\includegraphics[height = 3.75cm, width=0.3\linewidth]{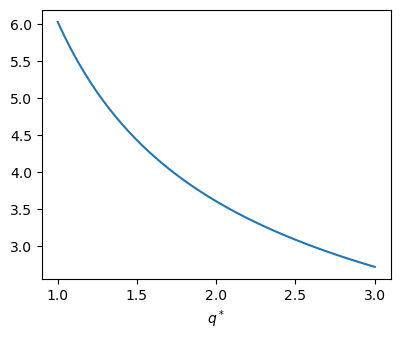}}
    \caption{Plots of the $\log$ upper bound for the absolute value of the next-leading-order term of the layer variance, $\left|\tilde{q}^{\{1\}(\ell)}\right|$, against $q^*$, as in \eqref{abs_upper_bound_nlo}, for a range of $s = \{0.8, 0.85, 0.9\}$ from left to right, and fixed $m=3$. By increasing $q^*$ the upper bound for the next-leading-order term $\left|\tilde{q}^{\{1\}(\ell)}\right|$ reduces exponentially.} 
    \label{fig:abs_upper_bound_nlo_crelu}
\end{figure}

Plotting the upper bound of \cref{thm:finite_dim_nlo} in \cref{fig:abs_upper_bound_nlo_crelu}, the benefit of increasing $q^*$ for larger sparsity levels is apparent. By increasing the variance of the hidden layers the next-leading-order correction can be more tightly bounded, and $\tilde{q}^{(\ell)}$ more accurately follows the recursion $q^{(\ell+1)} = V(q^{(\ell)})$ in \eqref{eq:vmap_general} .

\subsection{\texorpdfstring{Sensitivity of $\chi_1^{(\ell)}$}{Sensitivity of chi1(q)}}\label{sec:sensitivity_of_chi}

Despite an asymmetry of $V(q)$ around $q^*$ causing a skewed distribution of $\tilde{q}^{(\ell)}$, the explicit reason for an inability to train comes from $\prod_{\ell=1}^L \chi_1\left(q^{(\ell)} \right)$, therefore the accrued larger values of $\chi_1\left(q^{(\ell)} \right)$ could further be decreased by reducing the sensitivity to ${q}^{(\ell)}$, i.e. reducing $\left|\frac{d\chi_1\left(q \right)}{dq}\right|$. The sensitivity of $\chi_{1}(q)$ to perturbations in $q$, to the best of our knowledge, has gone undiscussed as a means to regain trainability. 

Considering, 
\begin{align}
    \chi_{1, \text{CReLU}{\tau, m}}(q) &= \frac{\sigma_w^2}{2}\left[\text{erf}\left(\frac{\tau+m}{\sqrt{2q}}\right)- \text{erf}\left(\frac{\tau}{\sqrt{2q}}\right)\right], \label{eq:chi_crelu}
\end{align}
note that,
\begin{align}
    \chi'_{1, \text{CReLU}{\tau, m}}(q) &=  \frac{ \sigma_w^2}{\sqrt{8 \pi q^3}} \text{exp}\left( -\frac{ (\tau +m)^2}{{2q}} \right)  \left[ \tau{\exp\left({\frac{2m\tau+m^2}{2q}}\right)} 
     - (\tau+m)\right],
\end{align}
again full derivations can be found in \cref{appendix_further_deriv_derivations}.

\begin{figure}[hbtp]
    \centering
    \subfigure[]{\includegraphics[height = 3.75cm, width=0.3\linewidth]{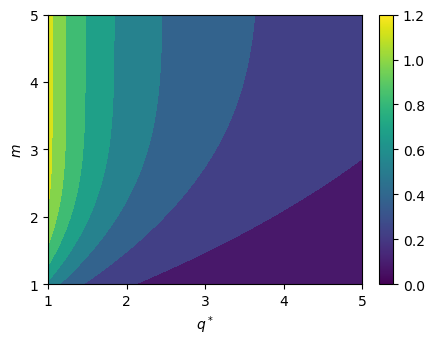}}
    \subfigure[]{\includegraphics[height = 3.75cm, width=0.3\linewidth]{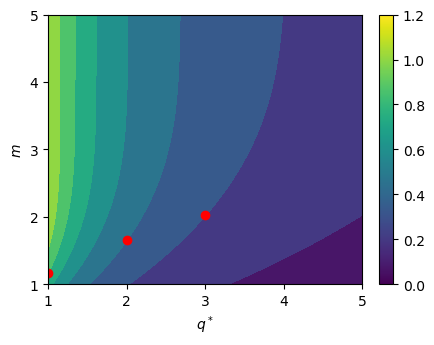}}
    \subfigure[]{\includegraphics[height = 3.75cm, width=0.3\linewidth]{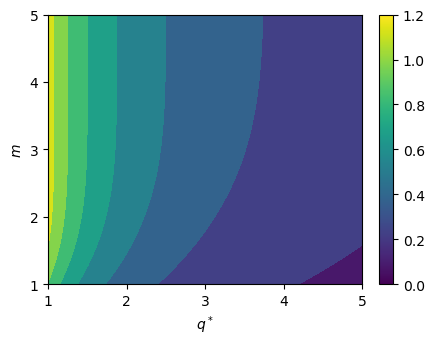}}
    \caption{$\chi'_{1, \text{CReLU}_{\tau, m}}(q^*)$, for a range of $s = \{0.8, 0.85, 0.9\}$, the plots of these three sparsities are from left to right then top to bottom, with horizontal axis $q^*$ and vertical axis $m$. For fixed value $m$ by increasing $q^*$ $\chi'_{1, \text{CReLU}_{\tau, m}}(q^*)$ reduces, and this holds across all three sparsity levels. The highlighted points in plot $s=0.85$ correspond to fixing $V'(q^*)=0.7$ for $q^* = \{ 1, 2, 3\}$ and have values 0.6, 0.3, 0.2 respectively. }
    \label{fig:chiprime_crelu_heatmap_multi_s}
\end{figure}

\cref{fig:chiprime_crelu_heatmap_multi_s} shows that $\chi'_{1, \text{CReLU}_{\tau, m}}(q^*)>0$ can be reduced at the fixed point $q^*$ by increasing this fixed point, allowing the use of $q^*$ as a parameter to reduce the sensitivity to perturbations of $\chi_{1, \text{CReLU}_{\tau, m}}$.

\subsection{Summary of findings}
Increasing $q^*$ improves the symmetry of the variance map around the fixed point, the fluctuation of the next-leading-order term, and further reduces the sensitivity of $\chi'_{1, \phi}(q^*)$ for our chosen $\phi$. In fact, Figures \ref{fig:vprimeprime_crelu_heatmap_multi_s}, \ref{fig:abs_upper_bound_nlo_crelu} and \ref{fig:chiprime_crelu_heatmap_multi_s} show that even small increases in $q^*$ to $2$ or $3$ allow for great improvements across all three considered control metrics. To this endeavor, the largely ignored parameter $q^*$ is shown to be a valuable tool for activations where $V'(q^*)$ approaches 1, especially where previous techniques are ruled out by not being approximately linear around the origin. 

\section{Experiments}\label{sec:experiments}

In order to most clearly demonstrate our theorised improvements of training stability, experiments are conducted on DNNs and CNNs of substantial depth and moderate width for the simple training tasks, classification of MNIST \cite{lecun_gradient-based_1998} and CIFAR10 \cite{krizhevsky_learning_2009}. Further experiments on Tiny ImageNet \cite{deng_imagenet_2009} can be found in \cref{app:tiny_imagenet}. The greater depth exacerbates the exponential growth of instability while the moderate width retains the stochasticity of the variance, see Sections \ref{sec:reducing_vprimeprime}-\ref{sec:sensitivity_of_chi}.  Specifically, the DNN has depth 100 and width 300 while the CNN has depth 50 and 300 channels per layer. The $\text{CReLU}_{\tau,m}$ nonlinear activation is used in all but the final softmax classification layer.  These network hyper parameters are chosen to isolate the instability mechanism and demonstrate the training capacity between varying $q^{*}$ rather than maximum achievable accuracy. 
  Throughout this section the networks are initialised with fixed point variances $q^*=\{1,2,3\}$ and the training for fixed sparsity $s$ and variance slope $V'(q^*)$ is compared. As the variance slope, $V'(q^*)$, prescribes the first-order behaviour of the variance map, this is used as a principled quantity to determine the choice of $m$.
  To further isolate the impact of varying $q^*$, the networks are trained with SGD. Training is for 200 and 300 epochs using learning rates $10^{-4}$ and $10^{-3}$ for the DNNs and CNNs respectively. The CNNs are also trained with cosine learning rate schedule. The data is split with 20\% retained for testing, and 10\% of the training data retained for validation, all final results are reported on the test data set. Each data input is normalized such that it has mean zero and variance $q^*$, the initialisation of the first layer is chosen to preserve this. 
  Each experiment is run on a single NVIDIA A100 or H100 for DNNs and CNNs respectively, and implemented in Pytorch Lightning.

 \cref{table_crelu_dnn} replicates the results in \cite{price_deep_2023} for $q^{*}=1$ and extends the experiments to $q^*=2$ and $3$, as is demonstrated in Figures \ref{fig:vprimeprime_crelu_heatmap_multi_s} and \ref{fig:chiprime_crelu_heatmap_multi_s} this small increase to $q^*$ allows relatively large reduction in $V''(q^*)$ and $\chi'_{1, \phi}(q^*)$. Sparsity up to 0.8 trains stably for the range of choice of $V'(q^*)=
\{0.5,0.7, 0.9\}$. However, for larger sparsity $s=0.85$ the network only retains near full accuracy of 91\% when $V'(q^*)$ is chosen for more careful associated values of $m$ and for $s=0.9$ accuracy is reduced and the choice of $m$ becomes more sensitive, smaller values of $m$ decrease the expressivity of the network whilst larger values can lose training stability. Increasing $q^*$ to 2 and 3 both reduces the sensitivity to the choice of $m$ and retains higher accuracy.

\cref{table_crelu_cnn_cifar_median_across_seeds_test} is analogous to  \cref{table_crelu_dnn}  but for the CNN where again the experiments with $q^*=1$ are shown to have final accuracy with greater sensitivity to $V'(q^*)$ and increasing $q^*$ to 2 and 3 generally improves the ability of the nets to train and increases the final accuracy. Here the median test accuracy is reported, as these tests have greater sensitivity in ability to train.  

The observed sparsity of the trained nets is also shown in \cref{table_crelu_dnn} and \cref{table_crelu_cnn_cifar_median_across_seeds_test}, and is generally consistent with the desired sparsity $s$. For anomalous cases where a decrease in sparsity is observed these are consistent with occurrences whereby the network fails to train, in these cases where exploding gradients appear the network no longer behaves with hidden layers in line with the expected Gaussian distribution, and hence the reduction in sparsity appears. Though it is worth noting we would generally advocate for somewhat smaller $V'(q^*)=0.7$ to further control the contraction of $q^{(\ell)}$ towards $q^*$.

\begin{table*}[ht!]
\centering
\tiny
\caption{DNN trained on MNIST, 100 layers and 300 width, with activation function $\text{CReLU}_{\tau, m}$ across five seeds for each set of parameters. Sparsity up to $s=0.8$ trains stably for the range of $V'(q^*)$ and $q^*$, however for larger sparsity $s=\{0.85, 0.9\}$, see improved sensitivity to the choice of $m$ through $V'(q^*)$ and higher accuracy through increased $q^*$. }
\begin{tabular}
{p{0.75cm}p{0.3cm}p{0.6cm}|p{0.3cm}p{0.65cm}|p{0.4cm}p{0.6cm}|p{0.3cm}p{0.65cm}|p{0.4cm}p{0.6cm}|p{0.3cm}p{0.65cm}|p{0.4cm}p{0.6cm}}
\toprule
\multicolumn{15}{c}{\textbf{DNN on MNIST}} \\ 
\midrule
 &   & & \multicolumn{4}{|c|}{$q^*=1$} & \multicolumn{4}{|c|}{$q^*=2$} & \multicolumn{4}{|c}{$q^*=3$} \\
\midrule
  & $s$ & $V'(q^*)$  & 
 $m$ & $V''(q^*)$  &\textbf{Acc.} & \textbf{Sparsity} & $m$ & $V''(q^*)$ &\textbf{Acc.} & \textbf{Sparsity} & $m$ & $V''(q^*)$ &\textbf{Acc.} & \textbf{Sparsity} \\ 
\midrule
\multirow{1}{*}{$\text{ReLU}$}  &0.5& &&& 0.94& 0.50&&&0.94&0.50&&&0.94&0.50\\
\midrule
 \multirow{7}{*}{$\text{CReLU}_{\tau, m}$} 
&0.8& 0.5 &0.89&-0.24& 0.90 & 0.80 & 1.26 & -0.12& 0.91 & 0.80 &1.54&-0.08 & 0.91 & 0.80 \\
     \rowcolor{lightgray}\cellcolor{white}&\cellcolor{white}&0.7 &1.27&-0.12& 0.91 & 0.80 & 1.79 & -0.06& 0.91 & 0.80 &2.19&-0.04 & 0.90 & 0.80 \\
    &&0.9 &1.85&  0.21& 0.91 & 0.80 & 2.62 &  0.11 & 0.91 & 0.80&3.21&0.07  & 0.91 & 0.80 \\
\cline{2-15}
&0.85&0.5 &0.81&-0.14& 0.79 & 0.85 & 1.14 & -0.07 & 0.86 & 0.85 &1.40 &-0.05 & 0.90 & 0.85 \\
     \rowcolor{lightgray}\cellcolor{white}&\cellcolor{white}&0.7 &1.17& 0.02& 0.91 & 0.85 & 1.66 &  0.01 & 0.92 & 0.85&2.03&0.01  & 0.91 & 0.85 \\
    &&0.9 &1.74& 0.41& 0.11 & 0.77 & 2.46 &  0.20 & 0.11 & 0.76 &3.01&0.13 & 0.27 & 0.80 \\
\cline{2-15}
&0.9& 0.5 &  0.72 &0.90& 0.41 & 0.90  & 1.02 & 0.00 &0.49 & 0.90 & 1.25 & 0.00& 0.54 & 0.90  \\
 \rowcolor{lightgray}\cellcolor{white}&\cellcolor{white}& 0.7 & 1.06 & 0.12 & 0.75 & 0.90  & 1.50 & 0.06 & 0.61 & 0.90  & 1.84 & 0.04 &0.89 & 0.90  \\
 && 0.9 & 1.61 & 0.34 & 0.10 & 0.84 & 2.28 & 0.17 & 0.11 & 0.80 & 2.79 & 0.11 & 0.10 & 0.77 \\

\end{tabular}
\label{table_crelu_dnn}
\end{table*}

\begin{table*}[ht!]
\centering
\tiny
\caption{CNN trained on CIFAR10, 50 layers and 300 channels, with activation function $\text{CReLU}_{\tau, m}$ across four seeds for each set of parameters. For higher sparsity levels $s=\{0.85, 0.9\}$, see improved sensitivity to the choice of $m$ through $V'(q^*)$ and higher accuracy through increased $q^*$. 
}
\begin{tabular}
{p{0.75cm}p{0.3cm}p{0.6cm}|p{0.3cm}p{0.65cm}|p{0.4cm}p{0.6cm}|p{0.3cm}p{0.65cm}|p{0.4cm}p{0.6cm}|p{0.3cm}p{0.65cm}|p{0.4cm}p{0.6cm}}
\toprule
\multicolumn{15}{c}{\textbf{CNN on CIFAR10}} \\ 
\midrule
 &   & & \multicolumn{4}{|c|}{$q^*=1$} & \multicolumn{4}{|c|}{$q^*=2$} & \multicolumn{4}{|c}{$q^*=3$} \\
\midrule
  & $s$ & $V'(q^*)$  & 
 $m$ & $V''(q^*)$  &\textbf{Acc.} & \textbf{Sparsity} & $m$ & $V''(q^*)$ &\textbf{Acc.} & \textbf{Sparsity} & $m$ & $V''(q^*)$ &\textbf{Acc.} & \textbf{Sparsity} \\ 
\midrule
\multirow{1}{*}{$\text{ReLU}$}  &0.5& &&& 0.71& 0.50&&&0.75&0.50&&&0.71&0.50\\
\midrule
 \multirow{7}{*}{$\text{CReLU}_{\tau, m}$} 
&0.8& 0.5 & 0.89 & -0.24  & 0.57 & 0.80 & 1.26 & -0.12  & 0.58 & 0.80   & 1.54 & -0.08  & 0.59  & 0.80  \\
 \rowcolor{lightgray}\cellcolor{white}&\cellcolor{white}&0.7 & 1.27 & -0.12  & 0.58 & 0.80 & 1.79 & -0.06   & 0.60& 0.80   & 2.19 & -0.04  & 0.59  & 0.80  \\
 
   && 0.9 & 1.85 & 0.21   & 0.54 & 0.80 & 2.62 & 0.11   & 0.16 & 0.81   & 3.21 & 0.07   & 0.28  & 0.81  \\
\cline{2-15}
&0.85& 0.5 & 0.81 & -0.14 & 0.51 & 0.85 & 1.14 & -0.07 & 0.52  & 0.85   & 1.40 & -0.05  & 0.54  & 0.85  \\
\rowcolor{lightgray}\cellcolor{white}&\cellcolor{white}& 0.7 & 1.17 & 0.02  & 0.49 & 0.85 & 1.66 & 0.01   & 0.56 & 0.85   & 2.03 & 0.01   & 0.56  & 0.85  \\
    && 0.9 & 1.74 & 0.41  & 0.50 & 0.84 & 2.46 & 0.20   & 0.53 & 0.84   & 3.01 & 0.13   & 0.51  & 0.85  \\
\cline{2-15}
  &0.9& 0.5 & 0.72 & 0.00 & 0.41 & 0.90 & 1.02 & 0.00   & 0.43 & 0.89   & 1.25 & 0.00   & 0.46  & 0.90  \\
\rowcolor{lightgray}\cellcolor{white}&\cellcolor{white}&0.7 & 1.06 & 0.23 & 0.29 & 0.90 & 0.12 & 1.50   & 0.45 & 0.89   & 1.84 & 0.08   & 0.24  & 0.90  \\
     && 0.9 & 1.61 & 0.69 & 0.22 & 0.87 & 2.28 & 0.34   & 0.42 & 0.88  & 2.79 & 0.32    & 0.35  & 0.82  \\
\end{tabular}
\label{table_crelu_cnn_cifar_median_across_seeds_test}
\end{table*}

Figures \ref{fig:mnist_training_mean_plots}, \ref{fig:cifar_training_mean_plots} demonstrate the mean training loss and validation accuracy as a function of gradient descent steps for the DNN (across five seeds) and CNN (across four seeds) respectively, at sparsity $s=0.85$, $V'(q^*)=0.7$, standard deviation bars are also shown. The recommended value of $V'(q^*)=0.7$ is used for all plots as it is a value where all networks train. Results for the three representative values of $q=\{1,2,3\}$ are displayed. As indicated by the theory in \cref{sec:eoc_theory}, the training dynamics are generally observed to improve by increasing $q^*$. 
Figures \ref{fig:mean_train_loss_crelu}-\ref{fig:dnn_val_loss_crelu_ind_seed_s_0.9} in \cref{app:more_experiments_dnn} show plots for individual seeds and for the additional configurations in \cref{table_crelu_dnn}; in these plots the importance of larger $q^*$ is shown to become greater as either of sparsity level $s$ and $V'(q^*)$ increase. The improved training speed observed in Figures \ref{fig:mnist_training_mean_plots}, \ref{fig:cifar_training_mean_plots} for greater $q^*$ is apparent, this is in line with the theory in \cref{sec:sensitivity_of_chi} due to the reduction in sensitivity of $\chi_{1,\phi}(q^*)$ for increasing $q^*$, the corresponding values of $\chi'_{1, \phi}(q^*)$ to the parameter choices in Figures \ref{fig:mnist_training_mean_plots} and \ref{fig:cifar_training_mean_plots} are highlighted as red dots in \cref{fig:chiprime_crelu_heatmap_multi_s}.



\begin{figure}[htbp!]
    \centering
    \subfigure[]{\includegraphics[trim = 0cm 0cm 0cm 1.75cm, clip,height=3.45cm, width=0.45\linewidth]{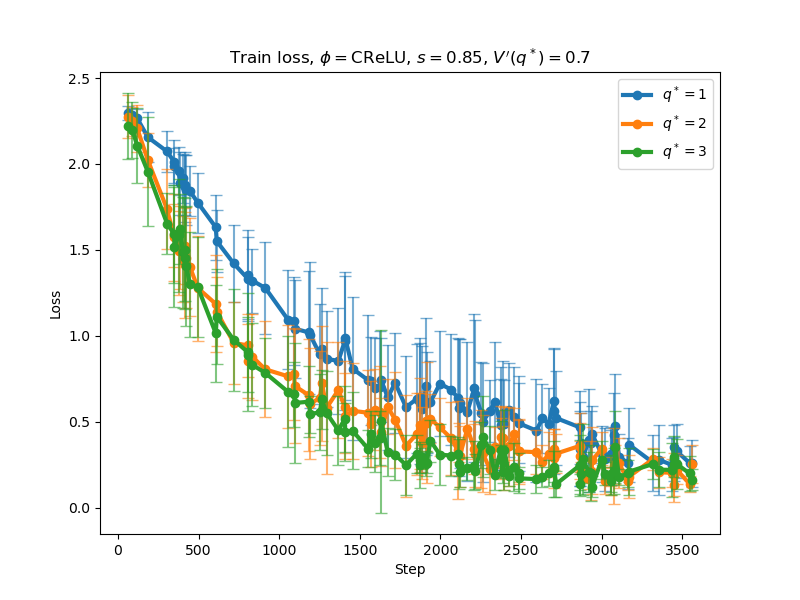}
    }
    \subfigure[]{\includegraphics[trim = 0cm 0cm 0cm 1.75cm, clip, height=3.45cm, width=0.45\linewidth]{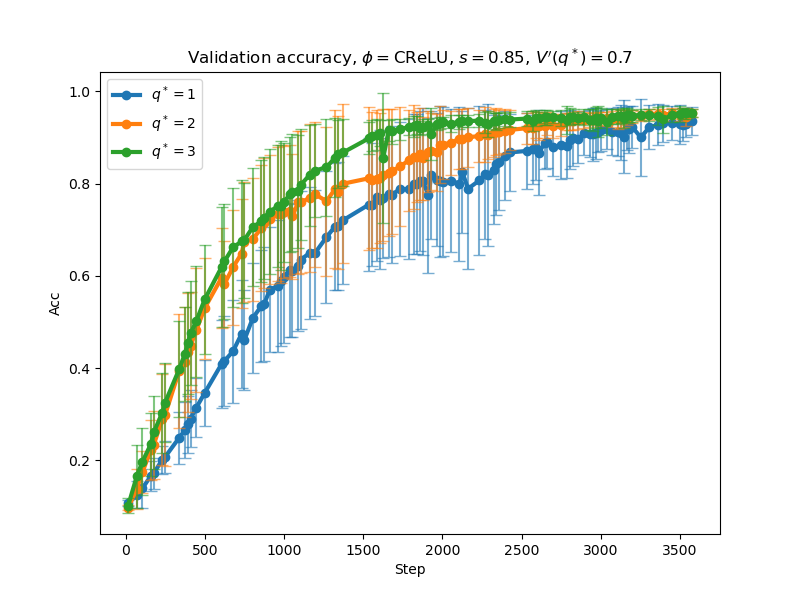}}
     \caption{Mean training loss, (a), and validation accuracy, (b), against step, for a DNN with $\phi = \text{CReLU}_{\tau, m}$ trained on MNIST across five seeds, with standard deviation error bars. $\phi = \text{CReLU}_{\tau, m}$ is such that $s=0.85$ and $V'(q^*) = 0.7$. For increasingly larger of $q^*=\{ 1, 2, 3 \}$, the training loss and validation accuracy converge faster.}
    \label{fig:mnist_training_mean_plots}
\end{figure}

Again, plots of the training loss for all experiments of the CNNs, including of Tiny ImageNet can be found in Figures \ref{fig:cifar_training_mean_plots}-\ref{fig:mean_train_loss_crelu_tinyimagenet} in \cref{more_experiments_cnn}.


\begin{figure}[htbp!]
    \centering
\subfigure[]{\includegraphics[trim = 0cm 0cm 0cm 1.75cm, clip, height=3.45cm,  width=0.45\linewidth]{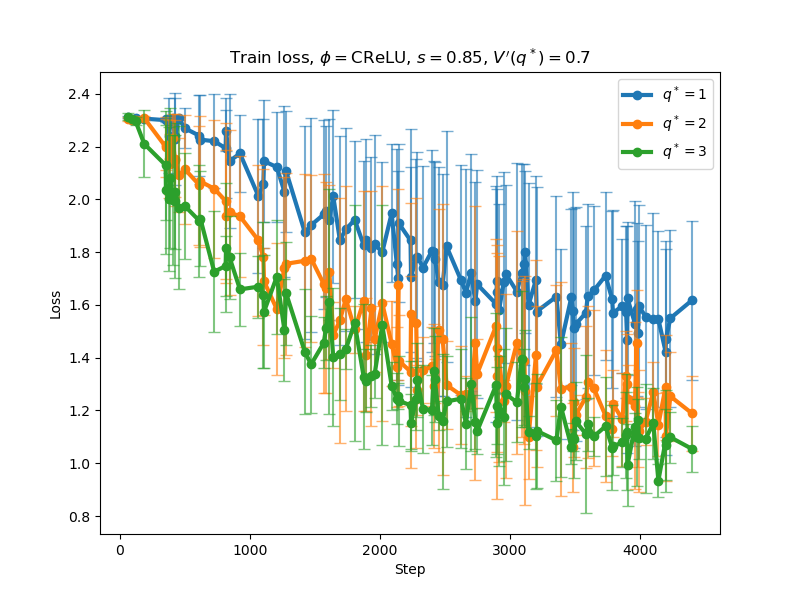}
    }
    \subfigure[]{\includegraphics[trim = 0cm 0cm 0cm 1.75cm, clip, height=3.45cm,  width=0.45\linewidth]{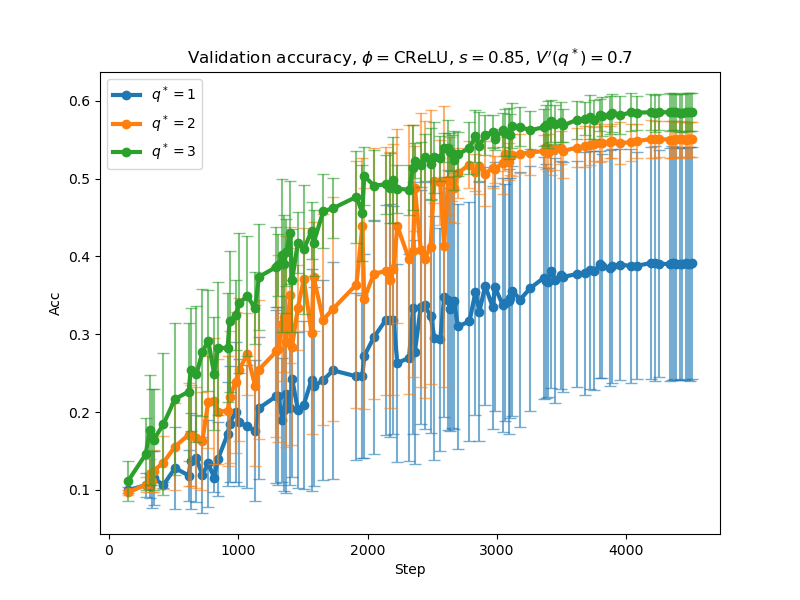}}
    \caption{Mean training loss, (a), and validation accuracy, (b), against step, for a CNN with $\phi = \text{CReLU}_{\tau, m}$ trained on CIFAR10 across four seeds, with standard deviation error bars. $\phi = \text{CReLU}_{\tau, m}$ is such that $s=0.85$ and $V'(q^*) = 0.7$. For increasingly larger of $q^*=\{ 1, 2, 3 \}$, the training loss and validation accuracy converge faster.}
    \label{fig:cifar_training_mean_plots}
\end{figure}

\section{Conclusion and further extensions}

Previous Edge-of-Chaos analysis has primarily focused on how the correlation of nearby points evolves through the layers of a deep neural network, \cite{poole_exponential_2016}. For unitary weight matrices and activations that are approximately linear around the origin these correlation dynamics can be improved by having $q^{(\ell)}$ decreasing towards zero; the dynamic isometry property \cite{martens_rapid_2021, roberts_principles_2022, murray_activation_2022}.  Dynamic isometry isn’t possible for activations that are zero around the origin and new issues of training instability arise \cite{price_deep_2023}. This motivates revisiting the role of the Gaussian process variance and showing for the first time that this parameter is a valuable tool for improving the training dynamics of DNNs and CNNs with such sparsifying nonlinear activations.  Faster training, less hyper-parameter sensitivity, and greater overall accuracy are obtained for even greater 90\% hidden layer sparsities.

\paragraph{Future work:} A natural extension to this work is to consider more complex architectures such as transformers and state-space models.  Both cases include MLP type layers which could similarly benefit from the hidden layer sparsity, these are not considered here as they are a separate line of analysis. Further experiments demonstrating the interplay of our approach with pruned, quantized, and low-rank weight matrices is needed; as well as further engineering to improve and deploy the preliminary computational efficiency experiments for activation sparsity seen in \cref{app:sparse_matvec}.

\paragraph{Limitations:}\label{limitations} The EoC theory relies upon the infinite width limit assumption, hence these results mostly apply for models with sufficiently large width. For clarity of analysis, we evaluate our approach on simple classification tasks (MNIST, CIFAR10, Tiny ImageNet) and deep proxy architectures: DNNs (width 300, depth 100) and CNNs (300 channels, depth 50). While these configurations are not optimized for practical deployment, they provide a setting in which the training dynamics and stability effects predicted by our theory are most clearly observable. 

\section*{Acknowledgments}
The authors would acknowledge support from His Majesty’s Government in the development of this research.
Jared Tanner is supported by the UK Engineering and Physical Sciences Research Council (EPSRC) through the grant EP/Y028872/1.
The authors are also grateful to Thiziri Nait Saada of the University of Oxford for early discussions of this work.







\clearpage
\bibliography{references}
\bibliographystyle{ieeetr}

\newpage
\appendix
\onecolumn

\section{\texorpdfstring{Further on $\chi_1$}{Further on chi1}}\label{app_further_on_chi}
To further understand the importance of $\chi_1$ consider, as in \cite{xiao_dynamical_2018}, the input-output Jacobian,
\begin{equation}
    J = \frac{\partial x^{(L)}}{\partial x^{(0)}} = \prod_{l=1}^L D^{(\ell)}W^{(\ell)}.
\end{equation}
The Jacobian is the product of the linear operators used in computing each layer's error vector in backpropagation, and therefore is closely related to training.

To show how $\chi_1$ relates to the scaling of perturbations across a single layer in training, the second moment of the errors distribution, $\tilde{v}^{(\ell)} = \mathbb{E} \left[ {({\delta}_i^{\ell}})^2 \right]$ was shown in \cite{schoenholz_deep_2017} to evolve according to a recurrent relation dependent on $\chi_1$, 
\begin{equation}
    \tilde{v}^{(\ell)} = \tilde{v}^{(\ell+1)} \frac{N_{\ell+1}}{N_{\ell}}\chi_1.
\end{equation}
Therefore for a network where $N_{\ell} = N$ $\forall l \in \{ 1, \dots, L\}$, $\tilde{v}^{(\ell)} = \tilde{v}^{(0)} \left( \chi_1\right)^{\ell}$. 

Assuming a normalized input so that $q^{(1)}=q^*$, the moments of $D^{(\ell)}$ are given by,
\begin{equation}
    \mu_k = \int \left( \phi' \left( \sqrt{q^*}z \right) \right)^{2k} \gamma(dz). \label{eq:moments_of_diag_jacobian_general}
\end{equation}
(Note: $\sigma_w^2 \mu_1 = \chi_1$)
The quantities $\mu_1, \mu_2$ are then related to the first two moments of the spectrum of $JJ^T$, $m_1, m_2$, by the following,
\begin{align}
m_1 &= \left( \sigma_w^2 \mu_1 \right)^L, \\
m_2 &= \left( \sigma_w^2 \mu_1 \right)^{2L}L \left( \frac{\mu_2}{\mu_1^2} + \frac{1}{L} -1-s_1 \right),
\end{align}
where $s_1$ is the first moment of $WW^T$ $S-$transformed, and so we have identified $m_1 = (\chi_1)^L$. For our case of initialisation, $W$ is Gaussian with mean $0$ and variance $\sigma_w^2/N_{\ell-1}$, $s_1=-1$.

Since $m_1 = \left( \chi_1 \right)^L$, we can interpret $\chi_1$ as the growth of the error vector per layer.
The motivation of concentrating $\chi_1$ around $1$ is therefore to ensure for large $L$ the growth of the error vectors doesn't explode or collapse in expectation.

To further stabilize our network, we can concentrate the variance of $JJ^T$,
\begin{equation}
    \sigma_{JJ^T} = m_2 - m_1^2 = L \left( \frac{\mu_2}{\mu_1^2} - 1 - s_1 \right).
\end{equation}
As the variance $\sigma_{JJ^T
}^2$ increases with depth $L$ this may become problematic for large depth. Methods have been proposed to control $\sigma_{JJ^T}^2$, for example in \cite{pennington_emergence_2018} they proposed allowing $q^* \rightarrow 0$ at a rate to ensure $\sigma_{JJ^T}$ converges to a nonzero constant in the large depth limit.
Further studies show in the case of $\phi = \text{ReLU}$, there is only one EoC point and so $\sigma_{JJ^T}$ is fixed with no way to improve it \cite{pennington_emergence_2018}. If instead $\phi \left( \cdot \right)$ is approximately linear around the origin, with orthonormal activations then $s_1=0$, and it is possible to have $\frac{\mu_2}{\mu_1^2} \geq 1$ approach $1$ in order to mitigate the growth of $\sigma_{JJ^T}$ \cite{murray_activation_2022}, such as $\frac{\mu_2}{\mu_1^2} \sim L^{-1}$ \cite{martens_rapid_2021}.

In our case, for $\phi \left(  \cdot \right)$ where $\phi \left( z\right) =0 $ with $|z| \leq \tau$, no such $L$ independence is possible; however $\frac{\mu_2}{\mu_1^2} -1 -s_1$ remains a function of $m, \tau, q^*$.

\section{\texorpdfstring{Derivation of further derivatives for $\text{CReLU}_{\tau, m}$}{Derivation of further derivatives for CReLU}}{\label{appendix_further_deriv_derivations}}

We know,
\begin{align}
    V'_{\text{CReLU}_{\tau, m}}(q) &= \sigma
    _w^2\left( \frac{1}{2} \text{erf} \left( \frac{\tau+m}{\sqrt{2q}}\right)-\frac{1}{2} \text{erf} \left( \frac{\tau}{\sqrt{2q}}\right)-\frac{m}{\sqrt{2\pi q}}\text{exp}\left(- \frac{(\tau+m)^2}{2q}\right)\right), \\
    &= \sigma_w^2 \left( a_1 + a_2 + a_3\right),
\end{align}

which we write as,
\begin{align}
    V'_{\text{CReLU}_{\tau, m}}(q) &= \sigma_w^2 \left( a_1 + a_2 + a_3\right),
\end{align}
defining,
\begin{align}
    a_1 &=\frac{1}{2} \text{erf} \left( \frac{\tau+m}{\sqrt{2q}}\right), \\
    a_2 &=-\frac{1}{2} \text{erf} \left( \frac{\tau}{\sqrt{2q}}\right), \\
    a_3&= -\frac{m}{\sqrt{2\pi q}}\text{exp}\left(- \frac{(\tau+m)^2}{2q}\right).
\end{align}

It follows that,
\begin{align}
    \frac{da_1}{dq} &= -\frac{\tau+m}{\sqrt{8\pi q^3}}\exp\left({-\frac{(\tau+m)^2}{2q}}\right),\\
    \frac{da_2}{dq}&=\frac{\tau}{\sqrt{8\pi q^3}}\exp\left({-\frac{\tau^2}{2q}}\right), \\
    \frac{da_3}{dq} 
    &= \frac{m}{\sqrt{8\pi q^3}}\left( 1- \frac{(\tau+m)^2}{q}\right)\text{exp}\left(- \frac{(\tau+m)^2}{2q}\right). 
\end{align}

Hence, 
\begin{align}
    V''_{\text{CReLU}_{\tau, m}}(q) &= \frac{\sigma
    _w^2}{\sqrt{8\pi q^3}}\left[ {\tau}\exp\left({-\frac{\tau^2}{2q}}\right)-\left(\tau+m\right)\exp\left({-\frac{(\tau+m)^2}{2q}}\right) \right. \nonumber \\
    &\hspace{45mm}+{m}\left.\left( 1- \frac{(\tau+m)^2}{q}\right)\exp\left({-\frac{(\tau+m)^2}{2q}}\right)\right].
\end{align}


Further, we can calculate $\chi'_{1, \text{CReLU}_{\tau, m}}$ since, $\chi_{1, \text{CReLU}_{\tau, m}} = \sigma_w^2 \left( a_1 + a_2 \right)$, we can immediately see,
\begin{align}
    \chi'_{1, \text{CReLU}_{\tau, m}} &= \frac{\sigma
    _w^2}{\sqrt{8\pi q^3}}\left( {\tau}\exp\left({-\frac{\tau^2}{2q}}\right)-(\tau+m)\exp\left({-\frac{(\tau+m)^2}{2q}}\right)\right).
\end{align}

\section{\texorpdfstring{Variance map analysis for $\text{CST}_{\tau, m}$}{Variance map analysis for CST}}{\label{app:var_map_cst}}

We can easily extend the results for $\text{CReLU}_{\tau, m}$ to $\text{CST}_{\tau, m}$, defined

\begin{align}
    \label{cst_eq}
            \text{CST}_{\tau, m}(x) &=
            \begin{cases}
            0, & \text{if $|x|<\tau$}\\
            x-\text{sign}(x)\tau, & \text{if $\tau\leq |x| \leq \tau+m$}\\
            \text{sign}(x)m, & \text{if $|x|>\tau+m$}.
            \end{cases}
\end{align}

For a visual of $\text{CST}_{\tau, m}$, see \cref{fig:cst_fig}.
\begin{figure}[htbp!]
    \centering
    \includegraphics[width=0.5\linewidth]{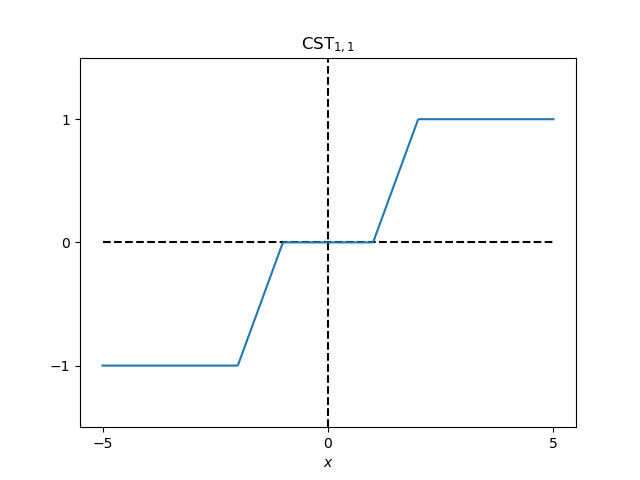}
    \caption{Non-linear activation $\text{CST}_{\tau, m}$,  as defined in \eqref{cst_eq} respectively, with $\tau=1$ and $m=1$.}
    \label{fig:cst_fig}
\end{figure}

It was shown in \cite{price_deep_2023} that,
\begin{align}
    V'_{\text{CST}_{\tau, m}}(q) = 2V'_{\text{CReLU}_{\tau, m}}(q),
\end{align}
and,
\begin{align}
    \chi'_{1, \text{CST}_{\tau, m}}(q) = 2\chi'_{1, \text{CReLU}_{\tau, m}}(q).
\end{align}

Therefore the derivatives of the variance map for $\text{CST}_{\tau, m}$, and derivatives of $\chi_{1, \text{CST}_{\tau, m}}$ are equivalent to $\text{CReLU}_{\tau, m}$ up to a factor of 2. Plots of $V''_{\text{CST}_{\tau, m}}(q^*)$ and $ \chi'_{1, \text{CST}_{\tau, m}}(q^*)$ can be found in Figures \ref{fig:vprimeprime_cst_heatmap_multi_s} and \ref{fig:chiprime_cst_heatmap_multi_s}, respectively.


\begin{figure}[htbp!]
    \centering
    \subfigure[]{\includegraphics[height = 3.75cm, width=0.3\linewidth]{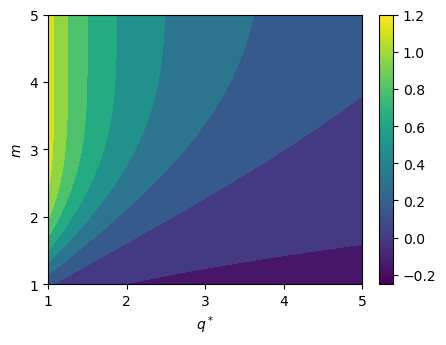}}
    \subfigure[]{\includegraphics[height = 3.75cm, width=0.3\linewidth]{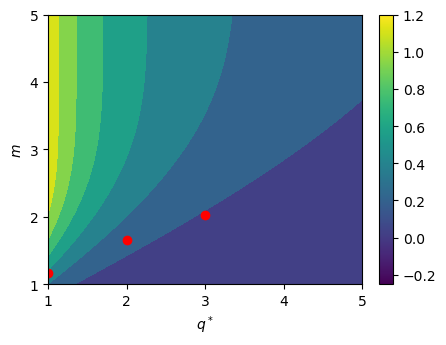}}
    \subfigure[]{\includegraphics[ height = 3.75cm, width=0.3\linewidth]{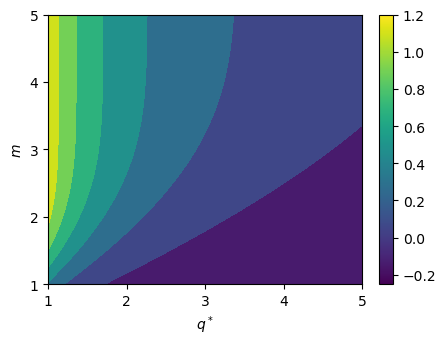}}
    \caption{$V''_{\text{CST}_{\tau, m}}(q^*)$, for a range of $s = \{0.8, 0.85, 0.\}$, the plots of these three sparsities are from left to right, with horizontal axis $q^*$ and vertical axis $m$. For all sparsity levels and large fixed value $m$ by increasing $q^*$, $V''_{\text{CST}_{\tau, m}}(q^*)$ reduces.}
    \label{fig:vprimeprime_cst_heatmap_multi_s}
\end{figure}

\begin{figure}[hbtp]
    \centering
    \subfigure[]{\includegraphics[height = 3.75cm, width=0.3\linewidth]{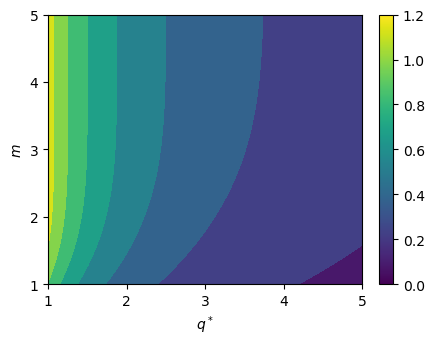}}
    \subfigure[]{\includegraphics[height = 3.75cm, width=0.3\linewidth]{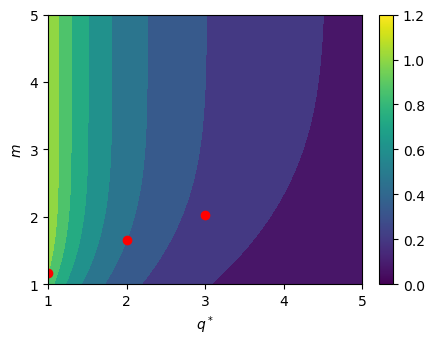}}
    \subfigure[]{\includegraphics[height = 3.75cm, width=0.3\linewidth]{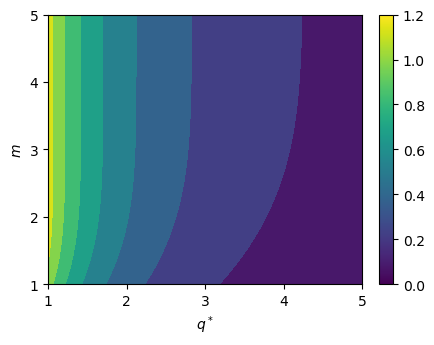}}
    \caption{$\chi'_{1, \text{CST}_{\tau, m}}(q^*)$, for a range of $s = \{0.8, 0.85, 0.9\}$, the plots of these three sparsities are from left to right, with horizontal axis $q^*$ and vertical axis $m$. For fixed value $m$ by increasing $q^*$, $\chi'_{1, \text{CST}_{\tau, m}}(q^*)$ reduces, and this holds across all three sparsity levels.}
    \label{fig:chiprime_cst_heatmap_multi_s}
\end{figure}

Figures \ref{fig:vprimeprime_cst_heatmap_multi_s}, \ref{fig:chiprime_cst_heatmap_multi_s} show that we can, similarly to $\text{CReLU}_{\tau, m}$, decrease the values $V''_{\text{CST}_{\tau, m}}(q^*)$ and $ \chi'_{1, \text{CST}_{\tau, m}}(q^*)$ for fixed $m$ by increasing $q^*$.
\section{Further sparsifying activation functions}{\label{app:further_sparsifying_activations}}

This section repeats the analysis of Sections \ref{sec:reducing_vprimeprime}-\ref{sec:sensitivity_of_chi} for a further three sparsifying activation functions,
\begin{align}
    \label{eq:hardshrink}
            \text{HardShrink}_{\tau, \alpha, \epsilon}(x) &=
            \begin{cases}
            0, & \text{if $x<\tau-\epsilon$}\\
            \alpha(x-\tau+\epsilon), & \text{if $\tau-\epsilon \leq x \leq \tau+\epsilon$}\\
            x+2\alpha \epsilon - \tau - \epsilon, & \text{if $x>\tau+\epsilon$},
            \end{cases}\\
    \label{eq:hardshrink_cap}
            \text{CHardShrink}_{\tau, \alpha, \epsilon, m}(x) &=
            \begin{cases}
            0, & \text{if $x<\tau-\epsilon$}\\
            \alpha(x-\tau+\epsilon), & \text{if $\tau-\epsilon \leq x \leq \tau+\epsilon$}\\
            x+2\alpha \epsilon - \tau - \epsilon, & \text{if $\tau+\epsilon< x<m$}\\
            m+2\alpha \epsilon - \tau - \epsilon & \text{if $x>m$},
            \end{cases}\\
    \label{eq:sawtooth}
\text{Sawtooth}_{\tau, m}(x) &=
            \begin{cases}
            0, & \text{if $|x-\tau|>m$}\\
            -\text{sign}(x-\tau)(x-\tau)+m, & \text{if $|x-\tau|\leq m$}.
            \end{cases}
\end{align}
For a visual of these activation functions see \cref{fig:further_activations}.
The behaviour of the control metrics $V'(q^*)$, $V''(q^*)$ and $\chi_{1, \phi}'(q^*)$ when increasing $q^*$ for sparsifying activation functions \eqref{eq:hardshrink}, \eqref{eq:hardshrink_cap} and \eqref{eq:sawtooth} can be numerically calculated and are plotted in \cref{fig:further_activations_control_metrics}. Observe again the decreased $V''(q^*)$ and $\chi_{1, \phi}'(q^*)$ by increasing $q^*$ for these activations.

\begin{figure}[htbp!]
    \centering
    \subfigure[t][]{
        \includegraphics[width=0.3\linewidth]{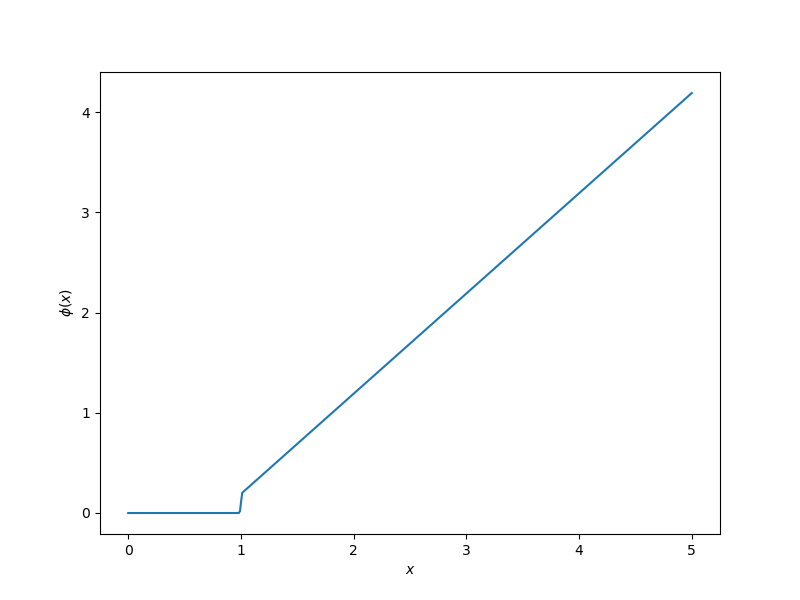}   
    }
    \hfill
    \subfigure[t][]{
        \includegraphics[width=0.3\linewidth]{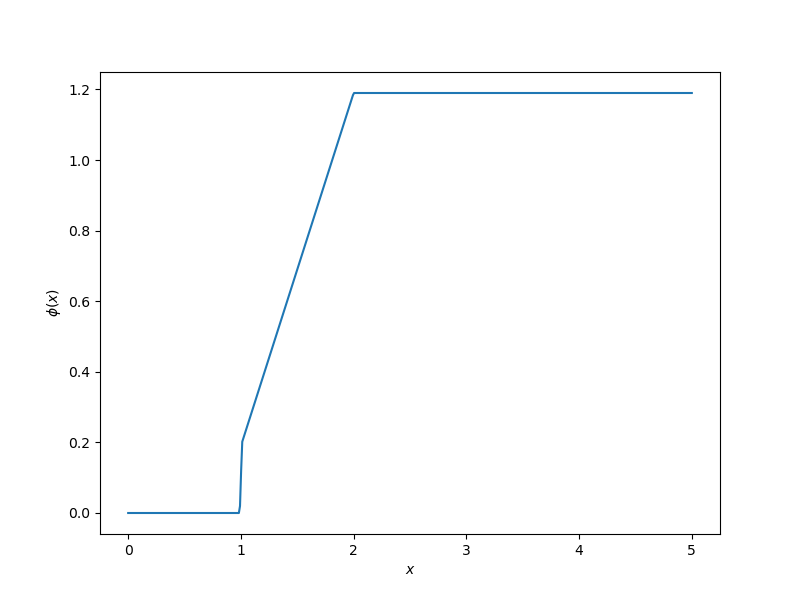}
    }
    \hfill
    \subfigure[t][]{
        \includegraphics[width=0.3\linewidth]{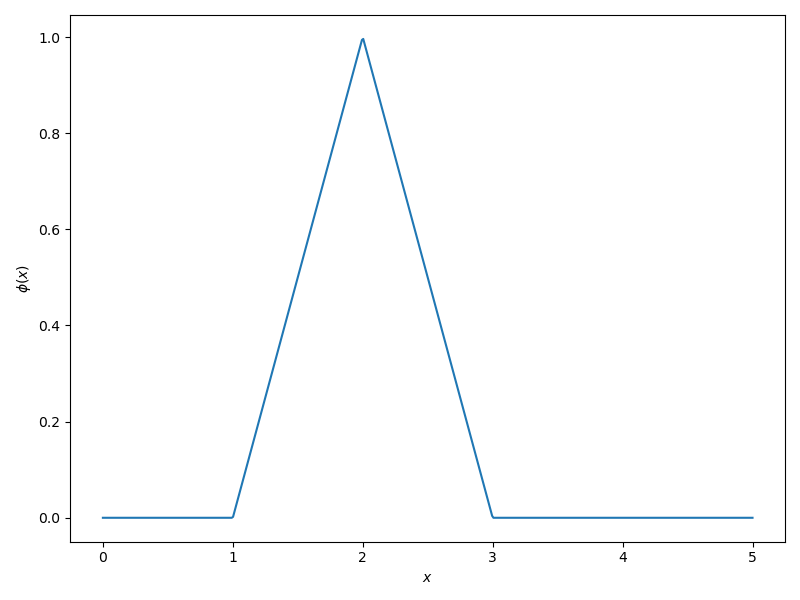}
    }
    \hfill
    \caption{Parameter plots for non-linear activations $\text{HardShrink}_{\tau, \alpha, \epsilon}$, $\text{CHardShrink}_{\tau, \alpha, \epsilon, m}$, $\text{Sawtooth}_{\tau, m}$, columns left to right, as defined in (1), (2), (3) respectively. (a), (b), (c) are exemplar plots of the activations themselves, (a) is $\text{HardShrink}_{\tau, \alpha, \epsilon}$ with $\tau =1$, $\alpha = 10$ and $\epsilon=0.01$. (b) is $\text{CHardShrink}_{\tau, \alpha, \epsilon, m}$ with $\tau =1$, $\alpha = 10$, $\epsilon=0.01$ and $m=2$. (c) is $\text{Sawtooth}_{\tau, m}$ with $\tau =1$, $m=1$. }
    \label{fig:further_activations}
\end{figure}

\begin{figure}[htbp!]
    \centering
    \subfigure[t][]{
        \includegraphics[width=0.3\linewidth]{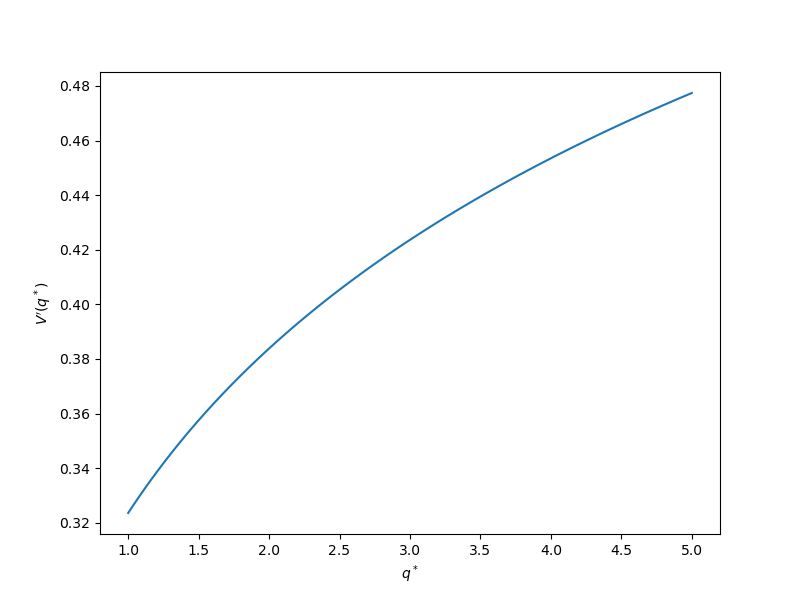}   
    }
    \hfill
        \subfigure[t][]{
        \includegraphics[width=0.3\linewidth]{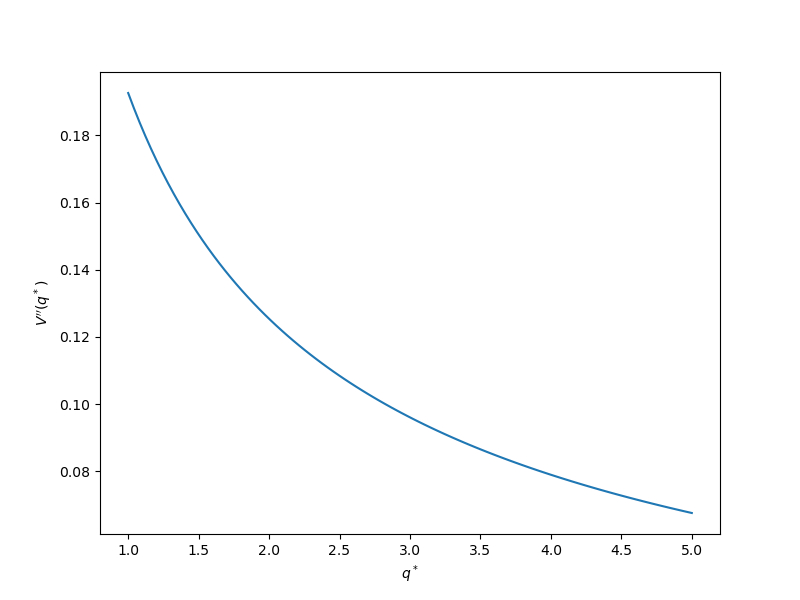}   
    }
    \hfill
    \subfigure[t][]{
        \includegraphics[width=0.3\linewidth]{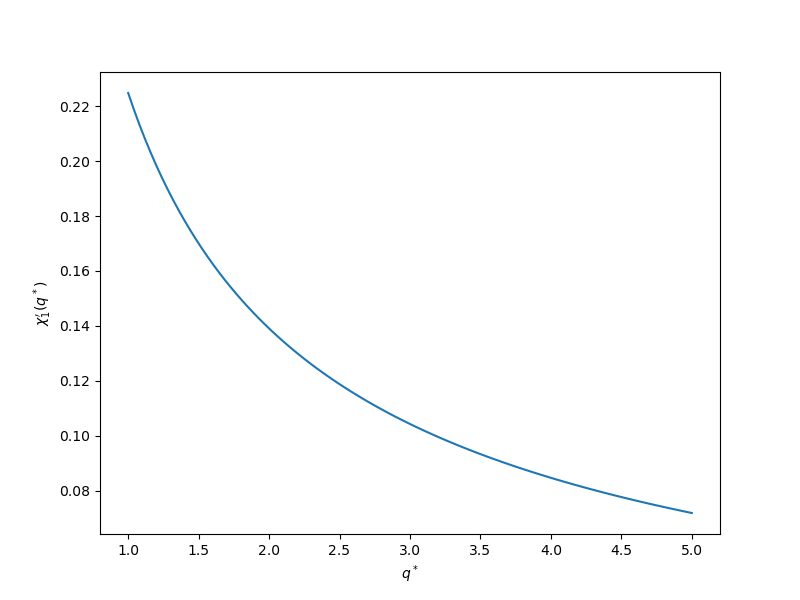} 
    }
    \hfill
    \subfigure[t][]{
        \includegraphics[width=0.3\linewidth]{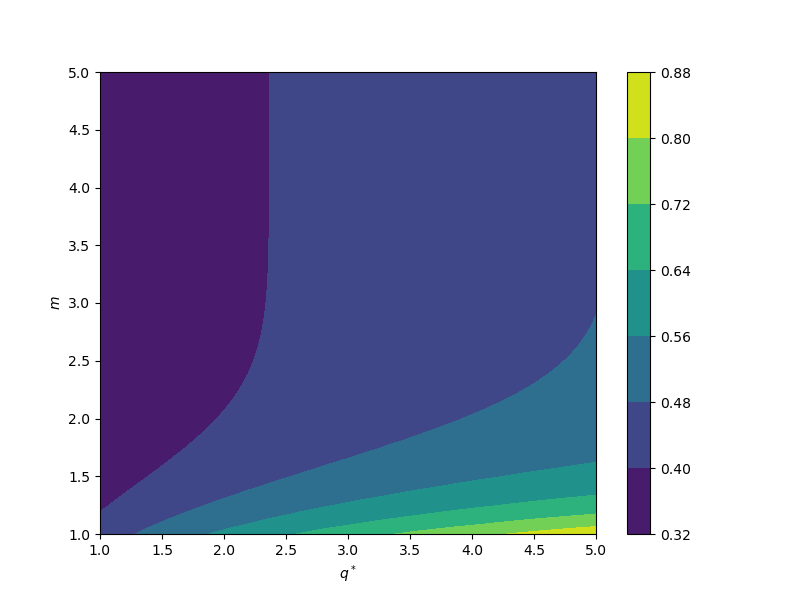}
    }
    \hfill
    \subfigure[t][]{
        \includegraphics[width=0.3\linewidth]{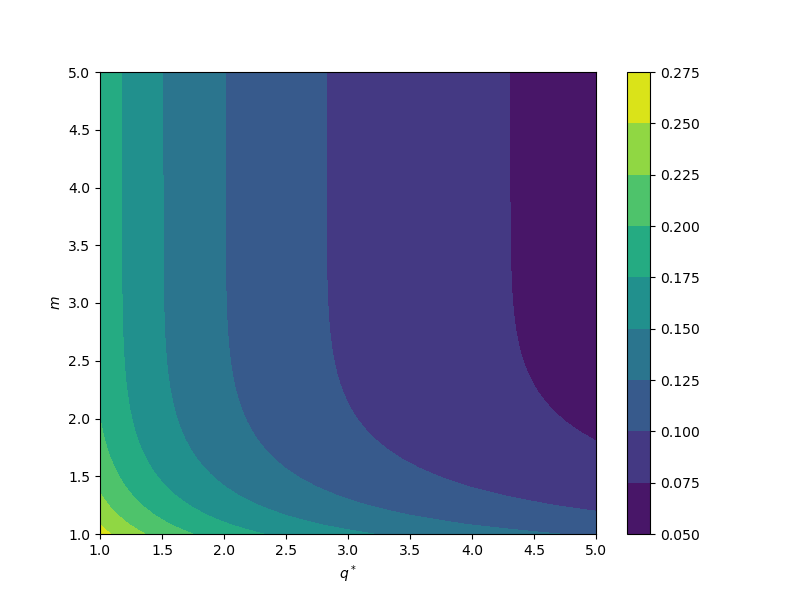}
    }
    \hfill
    \subfigure[t][]{
        \includegraphics[width=0.3\linewidth]{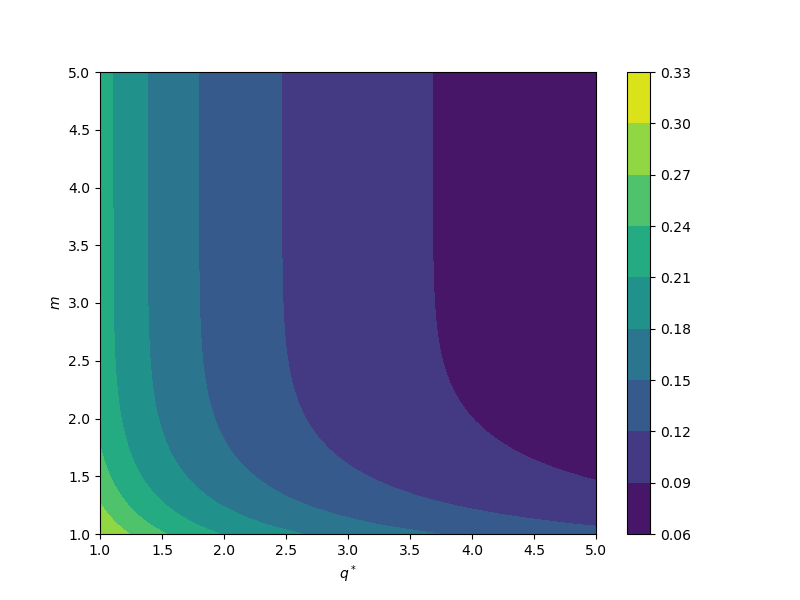}
    }
    \hfill
    \subfigure[t][]{
        \includegraphics[width=0.3\linewidth]{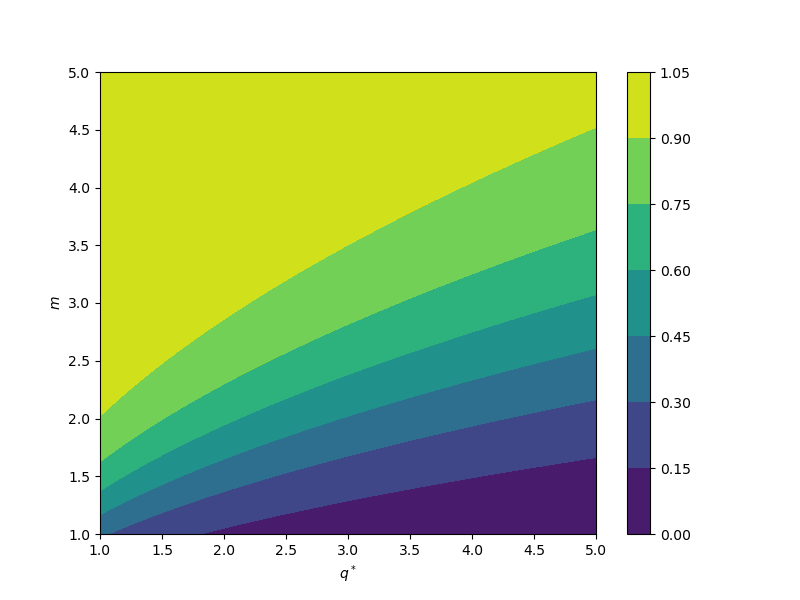}
    }
    \hfill
    \subfigure[t][]{
        \includegraphics[width=0.3\linewidth]{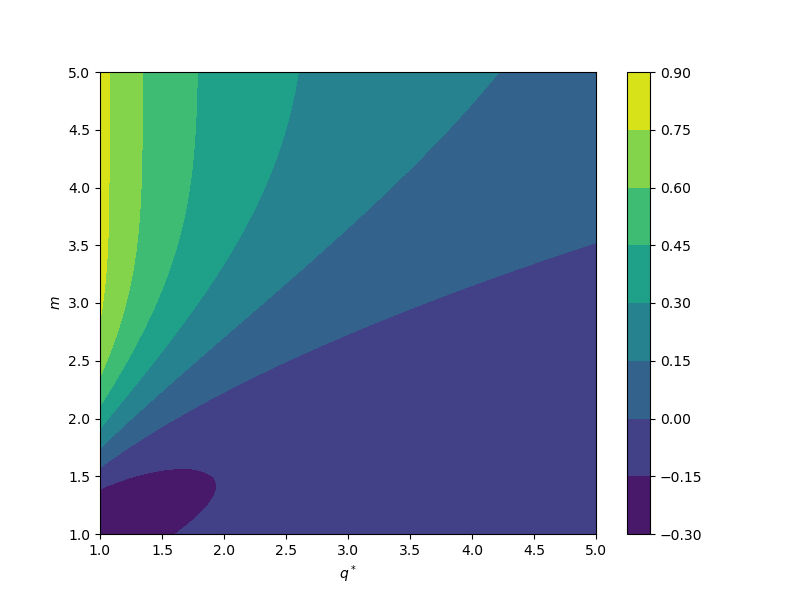}
    }     
    \hfill
    \subfigure[t][]{
        \includegraphics[width=0.3\linewidth]{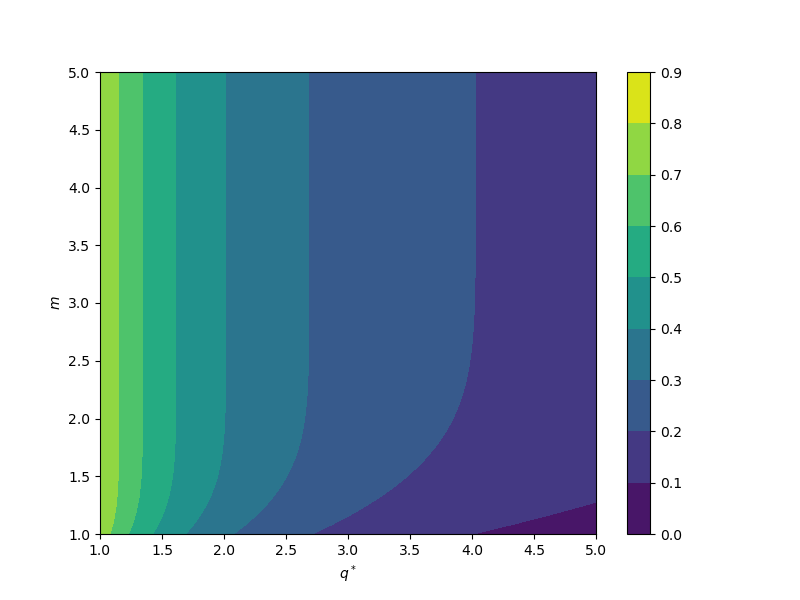}
    }
    \caption{Parameter plots for non-linear activations $\text{HardShrink}_{\tau, \alpha, \epsilon}$,  $\text{CHardShrink}_{\tau, \alpha, \epsilon, m}$, $\text{Sawtooth}_{\tau, m}$, rows top to bottom, as defined in \eqref{eq:hardshrink}, \eqref{eq:hardshrink_cap} and \eqref{eq:sawtooth} respectively. We fix sparsity level $s=0.85$ for all plots, and for $\text{HardShrink}_{\tau, \alpha, \epsilon}$, $\text{CHardShrink}_{\tau, \alpha, \epsilon, m}$, we fix  $\alpha = 10$ and $\epsilon=0.01$. The values plotted are $V'(q^*)$, $V''(q^*)$ and $\chi_1'(q^*)$, columns left to right respectively. (a-c) plots each value respectively against $q^*$ for $\text{HardShrink}_{\tau, \alpha, \epsilon}$. (d-i), are heatmaps of each value with $q^*$ along the horizontal axis and $m$ along the vertical axis, for $\text{CHardShrink}_{\tau, \alpha, \epsilon, m}$, $\text{Sawtooth}_{\tau, m}$. We again see we can decrease $V''(q^*)$ and $\chi_1'(q^*)$ by increasing $q^*$ for these activations.}
    \label{fig:further_activations_control_metrics}
\end{figure}

Note for our activation function $\text{CReLU}_{\tau, m}$, there is an important relationship that appears between $V'_{\phi}(q)$ and $\chi_{1, \phi}(q)$, found in \cite{price_deep_2023},

\begin{align}
    \chi_{1, \text{CReLU}{\tau, m}}(q) &= V'_{\text{CReLU}_{\tau, m}}(q) + \frac{\sigma_w^2m}{\sqrt{2 \pi q}} \text{exp}\left( -\frac{ (\tau +m)^2}{{2q}} \right). \label{chi_vprime_relation_crelu}
\end{align}

Equation \eqref{chi_vprime_relation_crelu} shows the direct relationship between $\chi_{1, \phi}(q)$ and $V'_{\phi}(q^*)$ for the exemplar choice of $\phi = \text{CReLU}{\tau, m}$. 

Then see the relationship between $\chi'_{1, \phi}(q)$ and $V''_{\phi}(q)$,
\begin{align}
     \chi'_{1, \text{CReLU}{\tau, m}}(q^*) &= V''_{\text{CReLU}_{\tau, m}}(q^*) + \frac{\sigma_w^2m}{\sqrt{2 \pi q^*}}\text{exp}\left( -\frac{ (\tau +m)^2}{{2q^*}} \right) \left[-\frac{1}{2q^*}+ \frac{ (\tau +m)^2}{{2{q^*}^2}}\right].
\end{align}

This relationship shows that  by increasing $q^*$, both $V''_{\text{CReLU}_{\tau, m}}(q^*)$ and $\frac{d\chi_{1, \text{CReLU}_{\tau, m}}}{dq} \left( q^* \right)$ can be reduced, \cref{fig:chiprime_crelu_heatmap_multi_s}. Assuming $V''_{\phi}(q^*)>0$ is decreasing, this secondary behaviour of the sensitivity reduction of $\chi_{1, \phi}$ should similarly apply to any activation function where you can write $ \chi'_{1, \phi}(q) = V''_{\phi}(q) + f(q)$, given $f (q)>0$ is decreasing on $q>0$.

\pagebreak
\section{Finite dimensional correction}\label{appendix_finite_correction}

To show the precise appearance of the second derivative in the finite dimensional correction we can look to the work of Chapters 4 and 5, \cite{roberts_principles_2022}, they find the following equivalent recursions,
\begin{align}
    q^{(\ell+1)} &= V\left(q^{(\ell)} \right) \label{eq:app_vmap}\\
    r^{(\ell+1)} &= \left(V' \left( q^{(\ell)}\right)\right)^2r^{(\ell)} + \sigma_w^4 \left( \left\langle \phi^4(z) \right\rangle_{q^{(\ell)}} - \left\langle \phi^2(z) \right\rangle^2_{q^{(\ell)}}\right) \label{eq:app_fourth_mom_recursion}\\
    \tilde{q}^{\{1\}(\ell+1)} &=  V' \left( q^{(\ell)}\right) \tilde{q}^{\{1\}(\ell)} + \frac{1}{2} V''\left( q^{(\ell)}\right) r^{(\ell)} \label{eq:app_nlo_recursion}.
\end{align}
We know from \cite{roberts_principles_2022} that the first-layer pre-activation distribution is exactly Gaussian, hence $\tilde{q}^{\{1\}(1)} =0$ and $ r^{(1)} =0$. Since we initialise at the fixed point $q^{(1)} = q^*$, such that $q^* =V \left( q^{*}\right)$, immediately $q^{(\ell)} = q^*$ for $l\geq 1$. 

The following Lemma \ref{lem:fourth_mom_eq}, Lemma \ref{lem:app_nlo_eq} will prove useful.

\begin{lemma}
\label{lem:fourth_mom_eq}
Where the following recursive relation holds, 
\begin{equation}
    r^{(\ell+1)} = V' \left( q^{*}\right)^2r^{(\ell)} + \sigma_w^4 \left( \left\langle \phi^4(z) \right\rangle_{q^{*}} - \left\langle \phi^2(z) \right\rangle^2_{q^{*}}\right),
\end{equation}
assuming $V'(q^*) \neq 1$ and $r^{(1)}=0$,
\begin{equation}
    r^{(\ell)} = \sigma_w^4 \left( \left\langle \phi^4(z) \right\rangle_{q^{*}} - \left\langle \phi^2(z) \right\rangle^2_{q^{*}}\right) \frac{1- V' \left( q^{*}\right)^{2(\ell-1)}}{1-V' \left( q^{*}\right)^{2}},
\end{equation}
for $l\geq2$.
\end{lemma}

\begin{proof}
    We have a recursive formula for $r^{(\ell)}$ of the form $r^{(\ell+1)} = ar^{(\ell)}+b$, where $a = V' \left( q^{*}\right)^2$ and $b= \sigma_w^4 \left( \left\langle \phi^4(z) \right\rangle_{q^{*}} - \left\langle \phi^2(z) \right\rangle^2_{q^{*}}\right)$, hence for $l\geq2$,
\begin{align}
    r^{(\ell)} = \sigma_w^4 \left( \left\langle \phi^4(z) \right\rangle_{q^{*}} - \left\langle \phi^2(z) \right\rangle^2_{q^{*}}\right) \sum_{0\leq i \leq l-2}  V' \left( q^{*}\right)^{2i}.
\end{align}
Since $V'(q^*)\neq 1$, we can simply apply the geometric series formula to find the result,
\begin{equation}
    r^{(\ell)} = \sigma_w^4 \left( \left\langle \phi^4(z) \right\rangle_{q^{*}} - \left\langle \phi^2(z) \right\rangle^2_{q^{*}}\right) \frac{1- V' \left( q^{*}\right)^{2(\ell-1)}}{1-V' \left( q^{*}\right)^{2}}.
\end{equation}
\end{proof}

\begin{lemma}
    \label{lem:app_nlo_eq}
    Where the following recursive relationship holds,
    \begin{equation}
        \tilde{q}^{\{1\}(\ell+1)} =  V' \left( q^{*}\right) \tilde{q}^{\{1\}(\ell)} + \frac{1}{2} V''\left( q^{*}\right) r^{(\ell)},
    \end{equation}
    and we initialise such that $\tilde{q}^{\{1\}(1)}= 0$ and $r^{(1)}=0$, then
    \begin{equation}
        \tilde{q}^{\{1\}(\ell)} = \frac{1}{2}V''\left( q^{*}\right) \sum_{0\leq i \leq \ell-3} V' \left( q^{*} \right)^i r^{(\ell-i-1)},
    \end{equation}
    for $\ell\geq 3$.
\end{lemma}
\begin{proof}
    We prove this result by induction. To begin, we identify $\tilde{q}^{\{1\}(2)}= V' \left( q^{*}\right) \tilde{q}^{\{1\}(1)} + \frac{1}{2} V''\left( q^{*}\right) r^{(1)} =0 $. Now for $\ell=3$,
\begin{align}
    \tilde{q}^{\{1\}(3)} &=V' \left( q^{*}\right) \tilde{q}^{\{1\}(2)} + \frac{1}{2} V''\left( q^{*}\right) r^{(2)} = \frac{1}{2} V''\left( q^{*}\right) r^{(2)}.
\end{align}
Let us then assume the result holds for $\ell=n$ so,
\begin{align}
    \tilde{q}^{\{1\}(n)} = \frac{1}{2}V''\left( q^{*}\right) \sum_{0\leq i \leq n-3} V' \left( q^{*} \right)^i r^{(n-i-1)}.
\end{align}
Hence for $l=n+1$,
\begin{align}
    \tilde{q}^{\{1\}(n+1)} &=  V' \left( q^{*}\right) \tilde{q}^{\{1\}(n)} + \frac{1}{2} V''\left( q^{*}\right) r^{(n)} \\
    &=  V' \left( q^{*}\right)  \frac{1}{2}V''\left( q^{*}\right) \sum_{0\leq i \leq n-3} V' \left( q^{*} \right)^i r^{(n-i-1)}  + \frac{1}{2} V''\left( q^{*}\right) r^{(n)} \\
    &=  \frac{1}{2}V''\left( q^{*}\right) \sum_{0\leq i \leq n-2} V' \left( q^{*} \right)^i r^{(n-i)},
\end{align}
as required. This completes our proof.
\end{proof}

\begin{proof}[Proof of \cref{thm:finite_dim_nlo}]

We can substitute our equation for $r^{(\ell)}$ from \cref{lem:fourth_mom_eq} into our result for $\tilde{q}^{\{1\}(\ell)}$ found in \cref{lem:app_nlo_eq}. Since $V'\left( q^{*}\right) \neq 1$, for $\ell \geq3$
\begin{align}
    \tilde{q}^{\{1\}(\ell)} &= \frac{\sigma_w^4}{2}V''\left( q^{*}\right)  \left( \left\langle \phi^4(z) \right\rangle_{q^{*}} - \left\langle \phi^2(z) \right\rangle^2_{q^{*}}\right)  \sum_{0\leq i \leq \ell-3} V' \left( q^{*} \right)^i \frac{1- V' \left( q^{*}\right)^{2(\ell-i-1)}}{1-V' \left( q^{*}\right)^{2}} \\
    &= \frac{\sigma_w^4}{2}\frac{V''\left( q^{*}\right)}{1-V' \left( q^{*}\right)^{2}}  \left( \left\langle \phi^4(z) \right\rangle_{q^{*}} - \left\langle \phi^2(z) \right\rangle^2_{q^{*}}\right) \sum_{0\leq i \leq \ell-3} V' \left( q^{*}\right)^i - V' \left( q^{*}\right)^{2\ell-i-2} \\
    &= \frac{\sigma_w^4}{2}\frac{V''\left( q^{*}\right)}{1-V' \left( q^{*}\right)^{2}}  \left( \left\langle \phi^4(z) \right\rangle_{q^{*}} - \left\langle \phi^2(z) \right\rangle^2_{q^{*}}\right) \left[ \frac{1-V' \left( q^{*}\right) ^{\ell-2}}{1-V' \left( q^{*}\right) } \right. \nonumber \\
    &\hspace{75mm}\left. - V' \left( q^{*}\right)^{2\ell-2} \frac{1-V' \left( q^{*} \right)^{2-\ell}}{1-V' \left( q^{*}\right)^{-1}} \right]\\
    &= \frac{\sigma_w^4}{2}\frac{V''\left( q^{*}\right)}{1-V' \left( q^{*}\right)^{2}}  \left( \left\langle \phi^4(z) \right\rangle_{q^{*}} - \left\langle \phi^2(z) \right\rangle^2_{q^{*}}\right) \times \nonumber \\
    &\hspace{50mm }\left[ \frac{1-V' \left( q^{*} \right)^{\ell-2}+ V' \left( q^{*}\right)^{2\ell-1}-V' \left( q^{*} \right)^{\ell+1}}{1-V' \left( q^{*}\right) } \right] \\
    &= \frac{\sigma_w^4}{2}\frac{V''\left( q^{*}\right)}{1-V' \left( q^{*}\right)^{2}}  \left( \left\langle \phi^4(z) \right\rangle_{q^{*}} - \left\langle \phi^2(z) \right\rangle^2_{q^{*}}\right) \times \nonumber \\
    &\hspace{60mm} \left[ \frac{ \left( 1-V' \left( q^{*}\right) ^{\ell-2} \right) \left( 1-V' \left( q^{*} \right)^{\ell+1} \right)}{1-V' \left( q^{*}\right) }  \right].
\end{align}

Further, considering our assumption $0<V' \left( q^{*}\right) <1$ and taking the absolute value, we find,
\begin{align}
    \left| \tilde{q}^{\{1\}(\ell)} \right| \leq \frac{\sigma_w^4}{2}\frac{\left|V''\left( q^{*}\right)\right|}{\left(1-V' \left( q^{*}\right)\right)^2\left(1+V' \left( q^{*}\right) \right)}  \left| \left\langle \phi^4(z) \right\rangle_{q^{*}} - \left\langle \phi^2(z) \right\rangle^2_{q^{*}}\right|,
\end{align}
for $\ell \geq 3$. This completes our proof.
\end{proof}


\clearpage
\section{Sparse MatVec speedup}\label{app:sparse_matvec}

For soundness of our concept improving efficiency we implement a custom CUDA kernel and measure the time speedup and fraction of energy compared to the standard GEMV operation. This code was written by the authors and we expect further improvements could be gained from professional engineers.

\begin{table*}[ht!]
\centering
\caption{The speedup and energy fraction of an $m \times n$ matrix-vector multiplication for varying levels of randomly induced sparsity $s$ in the vector. The comparison is between the standard GEMV matrix-vector multiplication, which doesn't take advantage of sparsity and a custom CUDA kernel which does. The results are reported over 100, 000 iterations, with a warmup of 20 iterations. The Speedup reported is (Dense GEMV time)/(Sparse MatVec time), and the Energy fraction is (Sparse MatVec energy)/(Dense GEMV energy). We can see as expected that for larger sparsity levels we have a larger speedup and a reduced energy consumption, for larger fractions of $m/n$ we also see this behaviour. The implementation of the custom sparse MatVec CUDA kernel relies on contiguous memory, for larger $n$ we see the speedup reduce and the energy fraction increase as the irregular memory access limits the improvements. }
\begin{tabular}
{p{0.4cm}p{0.7cm}|p{1.2cm}p{1.5cm}|p{1.2cm}p{1.5cm}|p{1.2cm}p{1.5cm}}
\toprule
\multicolumn{8}{c}{\textbf{Sparse MatVec Improvements}} \\ 
\midrule
 &   &  \multicolumn{2}{|c|}{$s=0.5$} & \multicolumn{2}{|c|}{$s=0.8$} & \multicolumn{2}{|c}{$s=0.9$} \\
\midrule
  $m/n$& $n$  &\textbf{Speedup} & \textbf{Energy fraction} &\textbf{Speedup} & \textbf{Energy fraction}  &\textbf{Speedup} & \textbf{Energy fraction}  \\ 
\midrule
\multirow[t]{3}{*}{0.5} & 100 & 1.250 & 0.950 & 1.183 & 0.844 & 1.216 & 0.819 \\
 & 300 & 0.806 & 1.211 & 2.097 & 0.470 & 2.286 & 0.431 \\
 & 3000 & 0.272 & 2.134 & 0.932 & 0.578 & 1.637 & 0.327 \\
\cline{1-8}
\multirow[t]{3}{*}{1} & 100 & 1.015 & 0.985 & 1.208 & 0.820 & 1.218 & 0.814 \\
 & 300 & 0.842 & 1.136 & 1.782 & 0.538 & 2.181 & 0.439 \\
 & 3000 & 0.454 & 1.352 & 1.008 & 0.585 & 2.936 & 0.188 \\
\cline{1-8}
\multirow[t]{3}{*}{2} & 100 & 1.620 & 0.613 & 2.054 & 0.485 & 2.164 & 0.459 \\
 & 300 & 0.879 & 1.044 & 1.689 & 0.545 & 2.157 & 0.424 \\
 & 3000 & 0.527 & 1.205 & 1.090 & 0.556 & 1.883 & 0.301 \\

\end{tabular}
\label{matvec_speedup}
\end{table*}

\clearpage

\section{Further DNN experiment results}\label{app:more_experiments_dnn}
\setlength{\tabcolsep}{6pt}

This section provides more figures for experiments of DNNs on the training loss behaviour of increasing $q^*$ for exemplar activation functions $\text{CReLU}_{\tau, m}$ and $\text{CST}_{\tau, m}$; this is for further detail of experiments in \cref{sec:experiments}. Each DNN trained on MNIST of depth 100 and width 300 takes approximately 2 hours to train on a single A100 with 40GB of RAM, total compute time for all DNN experiments in this section is estimated at 35 days of compute.

\subsection{\texorpdfstring{$\text{CReLU}_{\tau, m}$}{CReLU}}
\subsubsection{Mean training loss plots}

\cref{fig:mean_train_loss_crelu} provides the mean training loss for each parameter set across five different seeds, and demonstrates the improvement in mean training dynamics, with faster training for high sparsity levels, particularly for cases $s=\{0.85, 0.9\}$, $V'(q^*) = \{ 0.5, 0.7\}$; and the recovery of ability to train for $s=0.85$, $V'(q^*)=0.9$.
\begin{figure}[htbp!]
    \centering
    \subfigure[$s=0.6$, $V'(q^*)=0.5$]{
        \includegraphics[trim = 0cm 0cm 0cm 1.75cm, clip,height = 2cm, width=0.3\linewidth]{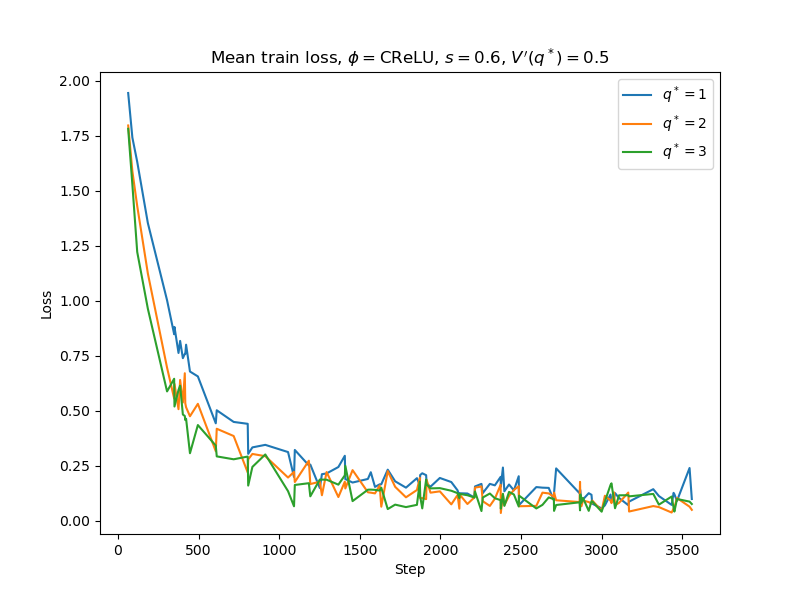}
    }
    \hfill
    \subfigure[$s=0.6$, $V'(q^*)=0.7$]{
        \includegraphics[trim = 0cm 0cm 0cm 1.75cm, clip,height = 2cm, width=0.3\linewidth]{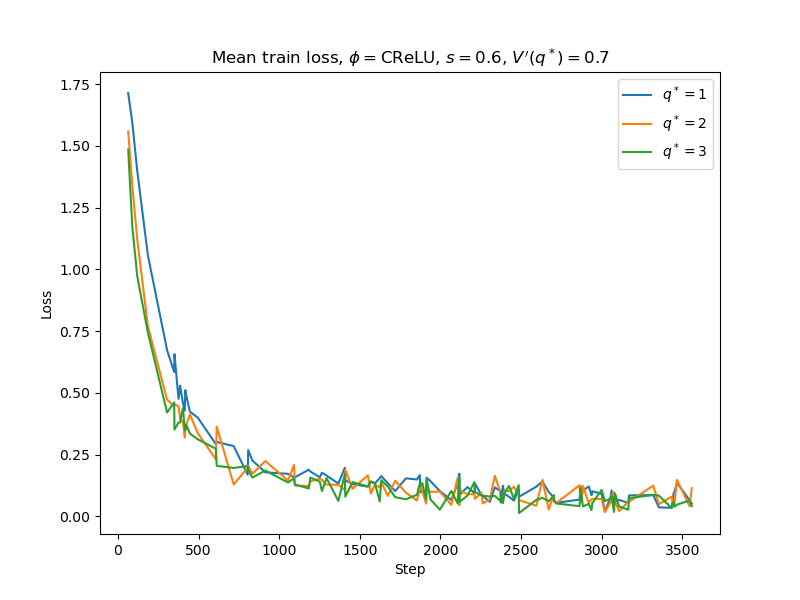}
    }
    \hfill
    \subfigure[$s=0.6$, $V'(q^*)=0.9$]{
        \includegraphics[trim = 0cm 0cm 0cm 1.75cm, clip,height = 2cm, width=0.3\linewidth]{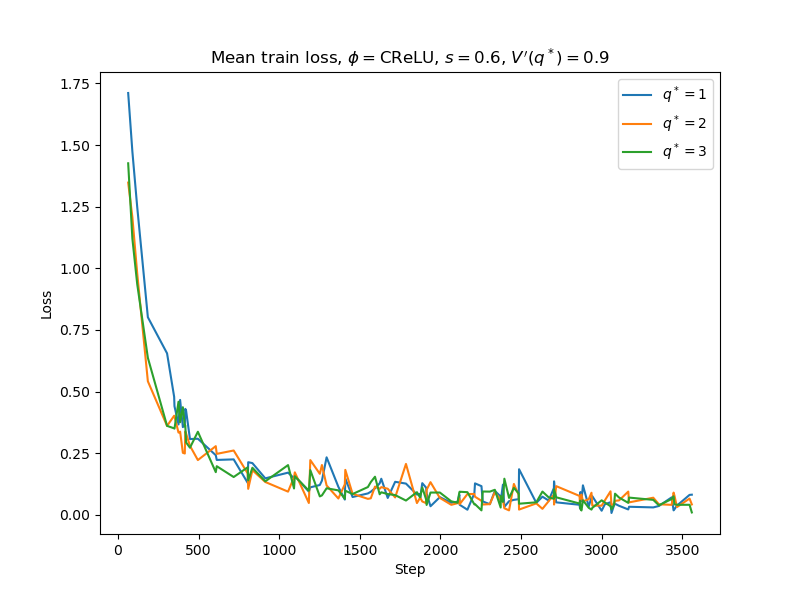}
    }
    \hfill
    \subfigure[$s=0.7$, $V'(q^*)=0.5$]{
        \includegraphics[trim = 0cm 0cm 0cm 1.75cm, clip,height = 2cm, width=0.3\linewidth]{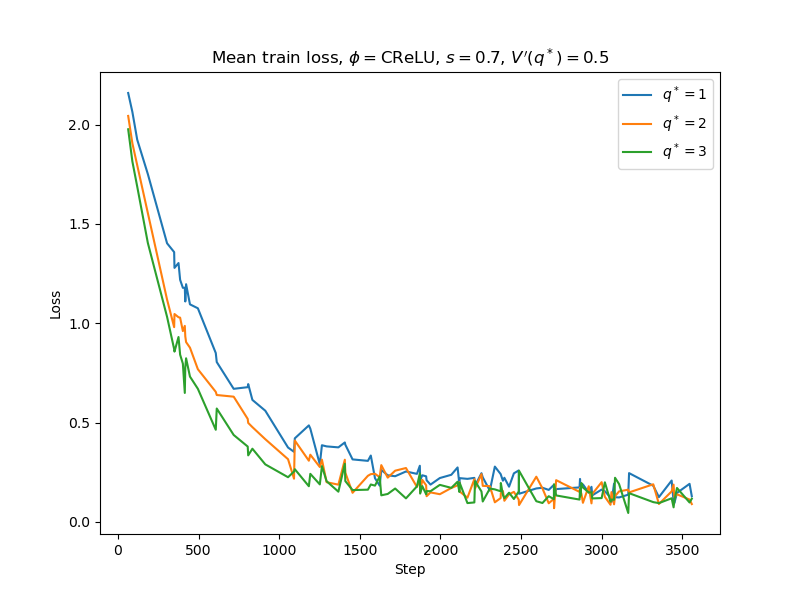}
    }
    \hfill
    \subfigure[$s=0.7$, $V'(q^*)=0.7$]{
        \includegraphics[trim = 0cm 0cm 0cm 1.75cm, clip,height = 2cm, width=0.3\linewidth]{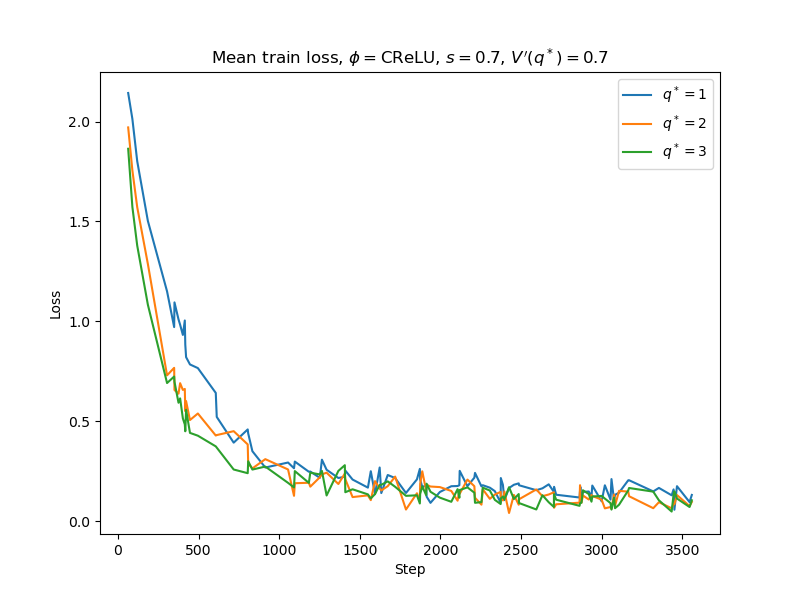}
    }
    \hfill
    \subfigure[$s=0.7$, $V'(q^*)=0.9$]{
        \includegraphics[trim = 0cm 0cm 0cm 1.75cm, clip,height = 2cm, width=0.3\linewidth]{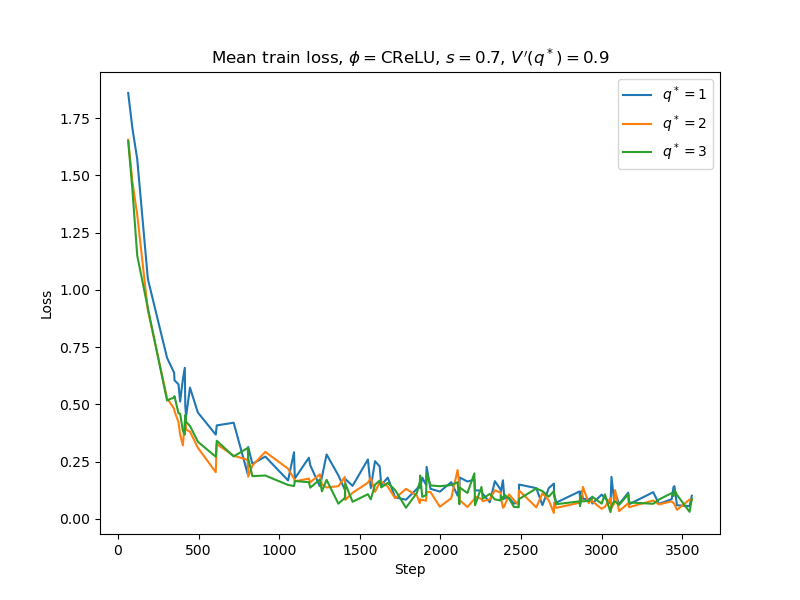}
    }
    \hfill
    \subfigure[$s=0.8$, $V'(q^*)=0.5$]{
        \includegraphics[trim = 0cm 0cm 0cm 1.75cm, clip,height = 2cm, width=0.3\linewidth]{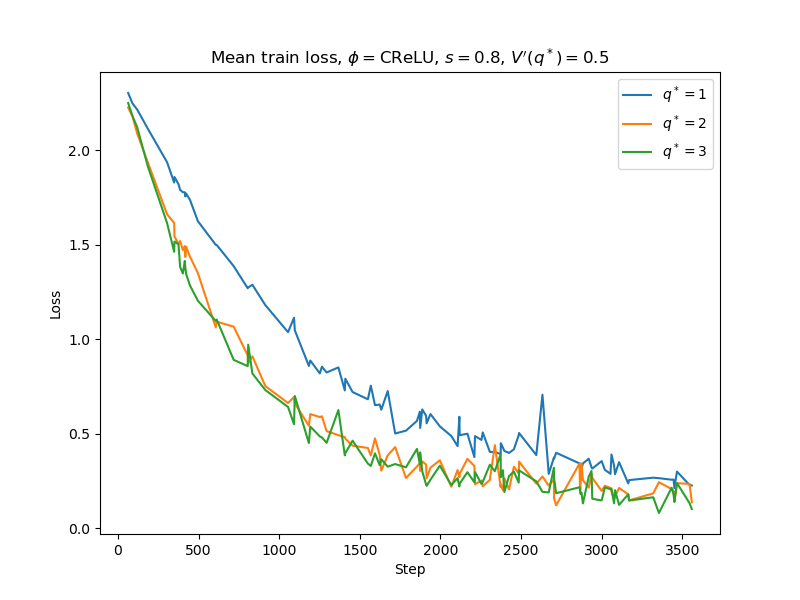}
    }
    \hfill
    \subfigure[$s=0.8$, $V'(q^*)=0.7$]{
        \includegraphics[trim = 0cm 0cm 0cm 1.75cm, clip,height = 2cm, width=0.3\linewidth]{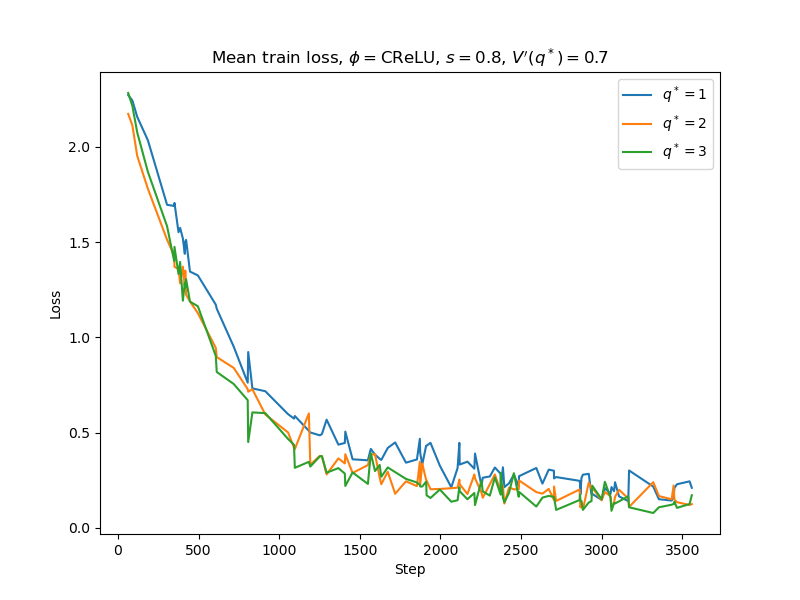}
    }
    \hfill
    \subfigure[$s=0.8$, $V'(q^*)=0.9$]{
        \includegraphics[trim = 0cm 0cm 0cm 1.75cm, clip,height = 2cm, width=0.3\linewidth]{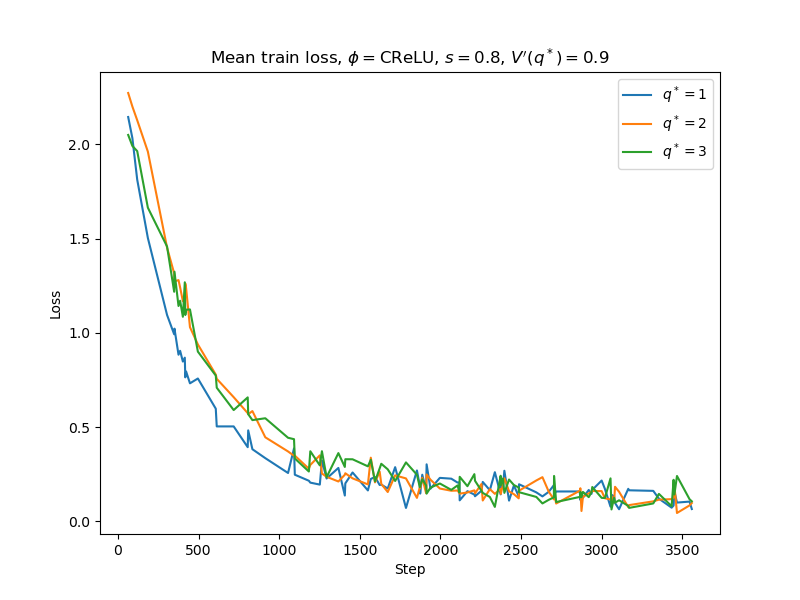}
    }
    \hfill
    \subfigure[$s=0.85$, $V'(q^*)=0.5$]{
        \includegraphics[trim = 0cm 0cm 0cm 1.75cm, clip,height = 2cm, width=0.3\linewidth]{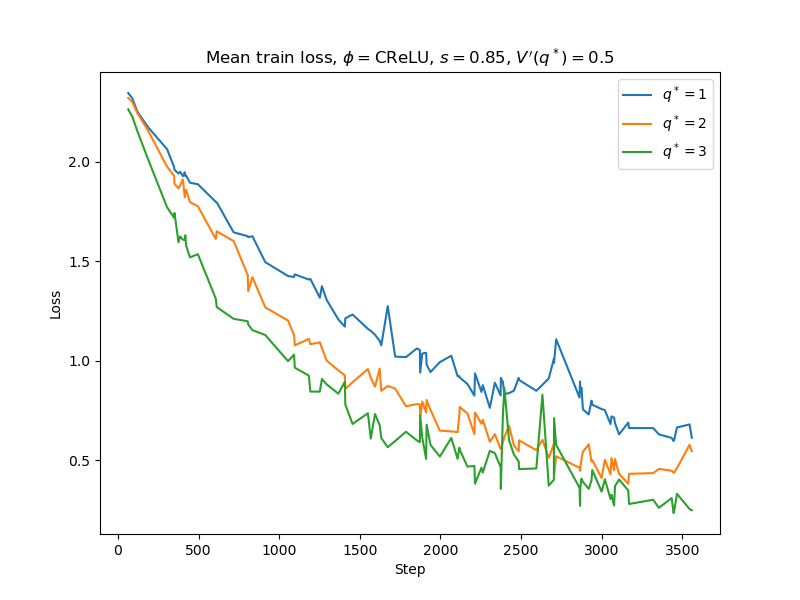}
    }
    \hfill
    \subfigure[$s=0.85$, $V'(q^*)=0.7$]{
        \includegraphics[trim = 0cm 0cm 0cm 1.75cm, clip,height = 2cm, width=0.3\linewidth]{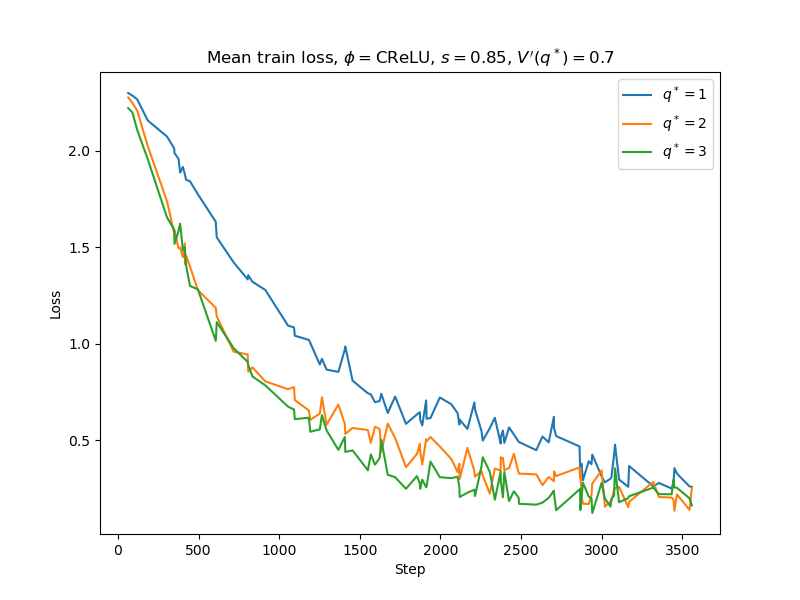}
    }
    \hfill
    \subfigure[$s=0.85$, $V'(q^*)=0.9$]{
        \includegraphics[trim = 0cm 0cm 0cm 1.75cm, clip,height = 2cm, width=0.3\linewidth]{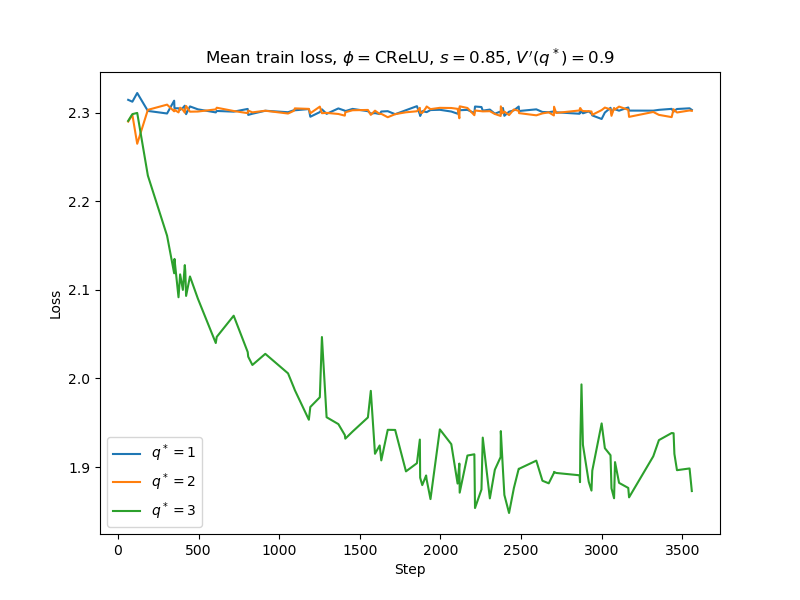}
    }
    \hfill
    \hfill
    \subfigure[$s=0.9$, $V'(q^*)=0.5$]{
        \includegraphics[trim = 0cm 0cm 0cm 1.75cm, clip,height = 2cm, width=0.3\linewidth]{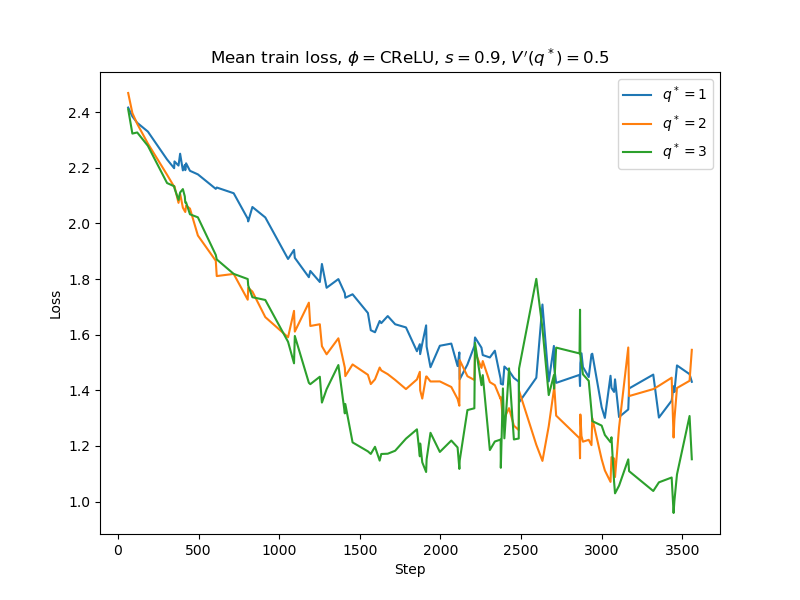}
    }
    \hfill
    \subfigure[$s=0.9$, $V'(q^*)=0.7$]{
        \includegraphics[trim = 0cm 0cm 0cm 1.75cm, clip,height = 2cm, width=0.3\linewidth]{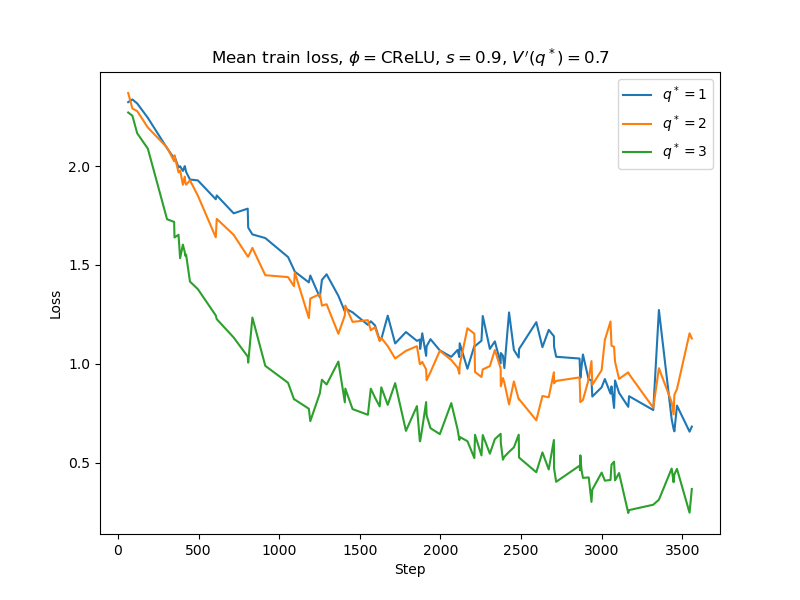}
    }
    \hfill
    \subfigure[$s=0.9$, $V'(q^*)=0.9$]{
        \includegraphics[trim = 0cm 0cm 0cm 1.75cm, clip,height = 2cm, width=0.3\linewidth]{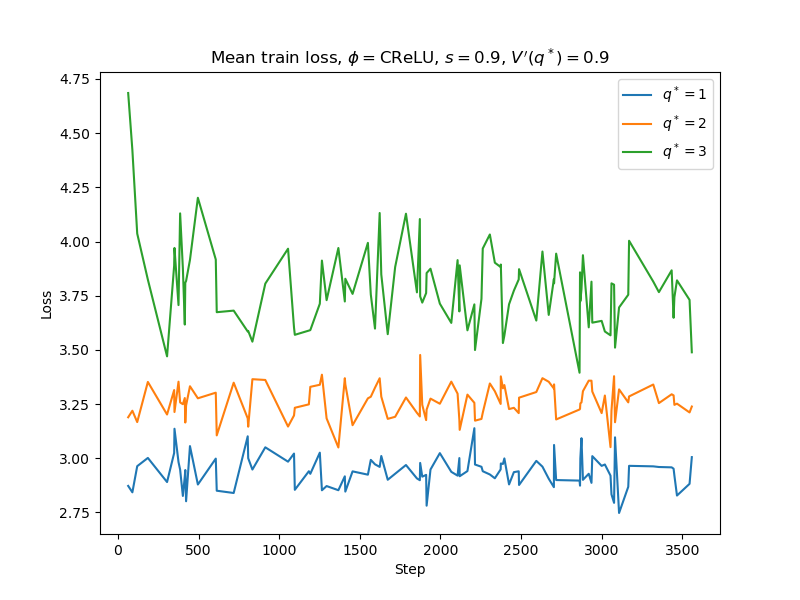}
    }
    \hfill
    \caption{Mean training loss across five seeds for a DNN with activation function $\phi = \text{CReLU}_{\tau, m}$, $\tau$ and $m$ chosen such that $s$, $V'(q^*)$ are as described. Observe the improved training speed for increased $q^*$ and particularly the complete recovery of stability with $q^*=3$ in (l), $s=0.85$, $V'(q^*)=0.9$. }
    \label{fig:mean_train_loss_crelu}
\end{figure}

\pagebreak
\subsubsection{Individual run training loss plots}

Now for a better look at each individual test, and clearer analysis of when recovery to train is found see the Figures \ref{fig:dnn_val_loss_crelu_ind_seed_s_0.6}-\ref{fig:dnn_val_loss_crelu_ind_seed_s_0.9}, which demonstrate the individual run improvement in training speed with increased $q^*$ for DNNs of depth 100 and width 300 trained on MNIST with activation $\text{CReLU}_{\tau, m}$, as well as a recovery in ability to train for $s=0.85$, $V'(q^*)=0.9$ when $q^*=3$.

\begin{figure}[htbp!]
    \centering
    \subfigure[$s=0.6$, $V'(q^*)=0.5$]{
        \includegraphics[trim = 0cm 0cm 0cm 2cm, clip,height = 5cm, width=1\linewidth]{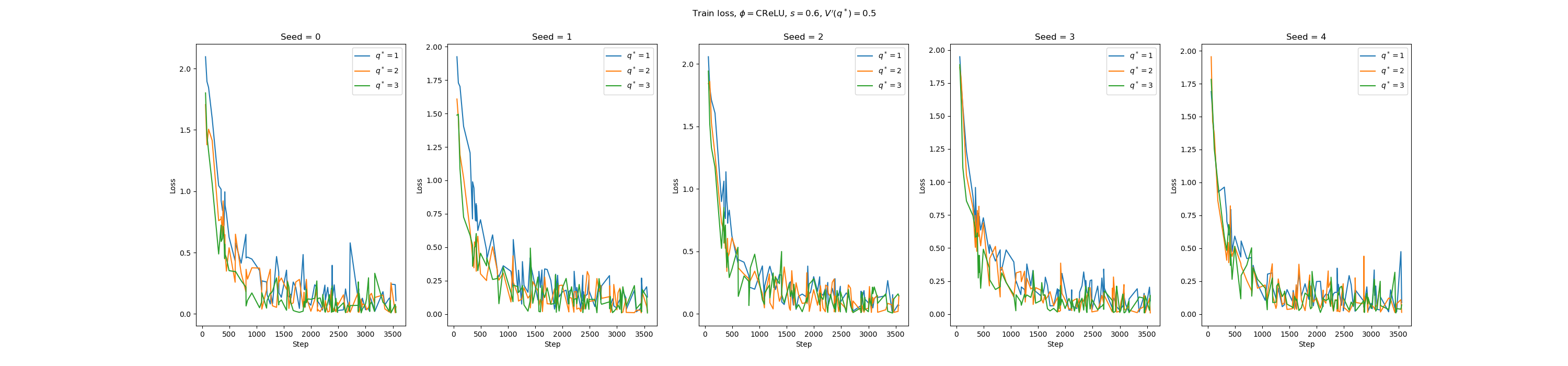}
    }
    \hfill
    \subfigure[$s=0.6$, $V'(q^*)=0.7$]{
        \includegraphics[trim = 0cm 0cm 0cm 2cm, clip,height = 5cm, width=1\linewidth]{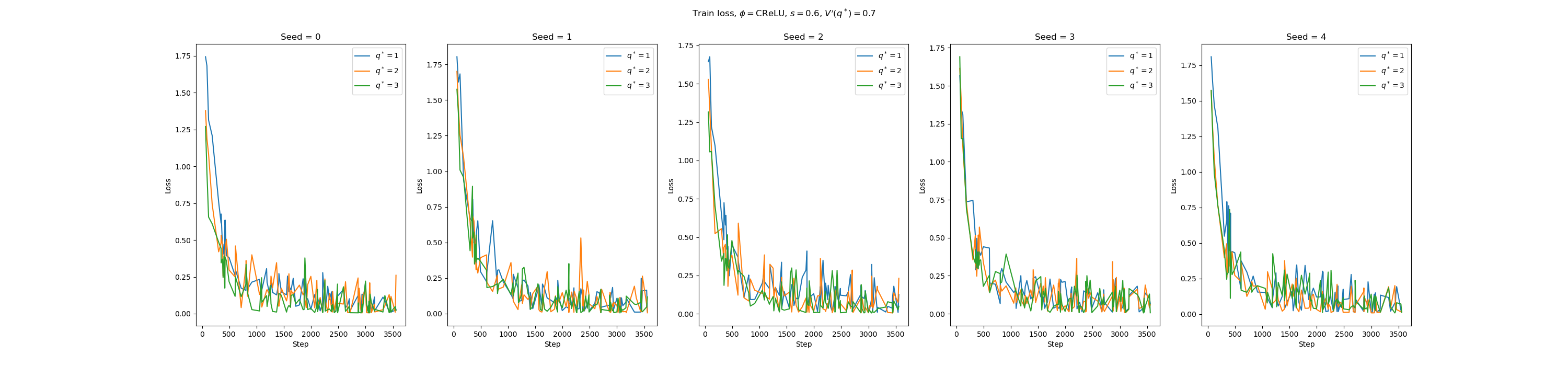}
    }
    \hfill
    \subfigure[$s=0.6$, $V'(q^*)=0.9$]{
        \includegraphics[trim = 0cm 0cm 0cm 2cm, clip,height = 5cm, width=1\linewidth]{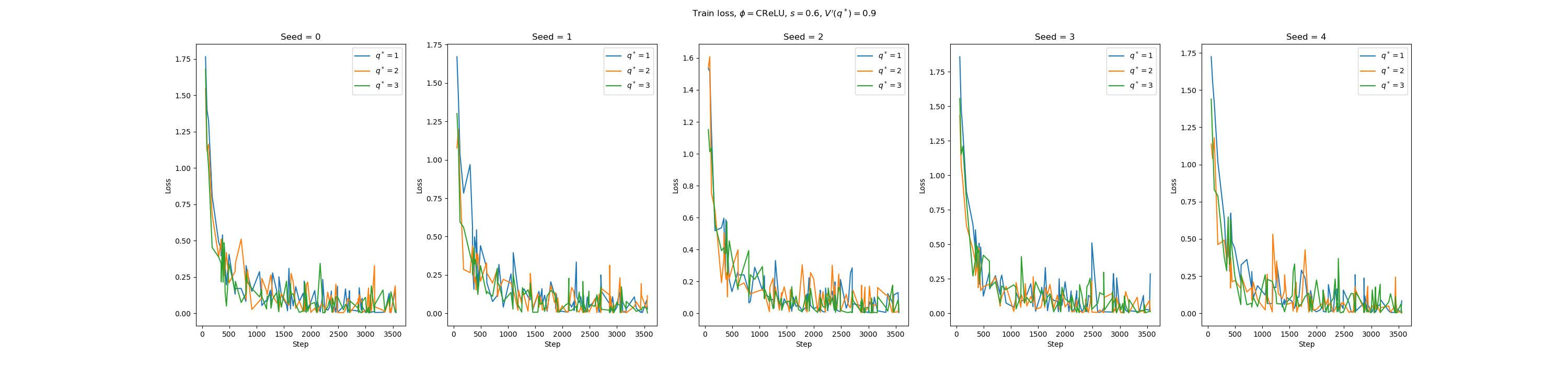}
    }
    \hfill
    
    \caption{Training loss for a DNN with activation function $\phi = \text{CReLU}_{\tau, m}$, with $\tau$ and $m$ chosen such that $s=0.6$, and for a given $V'(q^*)$. Observe improved training speed for increased $q^*$ for all parameter sets.}
    \label{fig:dnn_val_loss_crelu_ind_seed_s_0.6}
\end{figure}

\begin{figure}[htbp!]
    \centering
    \subfigure[$s=0.7$, $V'(q^*)=0.5$]{
        \includegraphics[trim = 0cm 0cm 0cm 2cm, clip,height = 5cm, width=1\linewidth]{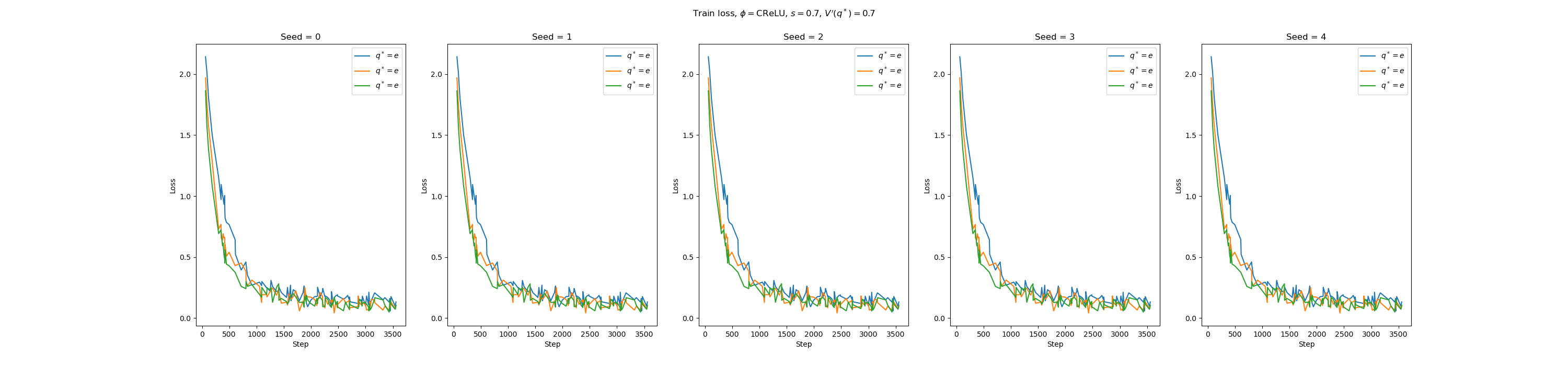}
    }
    \hfill
    \subfigure[$s=0.7$, $V'(q^*)=0.7$]{
        \includegraphics[trim = 0cm 0cm 0cm 2cm, clip,height = 5cm, width=1\linewidth]{figures/mnist/seed_train_loss_plots/train_loss_s_0.7_vprime_0.7_sfatrelu_max_individual_seed.png}
    }
    \hfill
    \subfigure[$s=0.7$, $V'(q^*)=0.9$]{
        \includegraphics[trim = 0cm 0cm 0cm 2cm, clip,height = 5cm, width=1\linewidth]{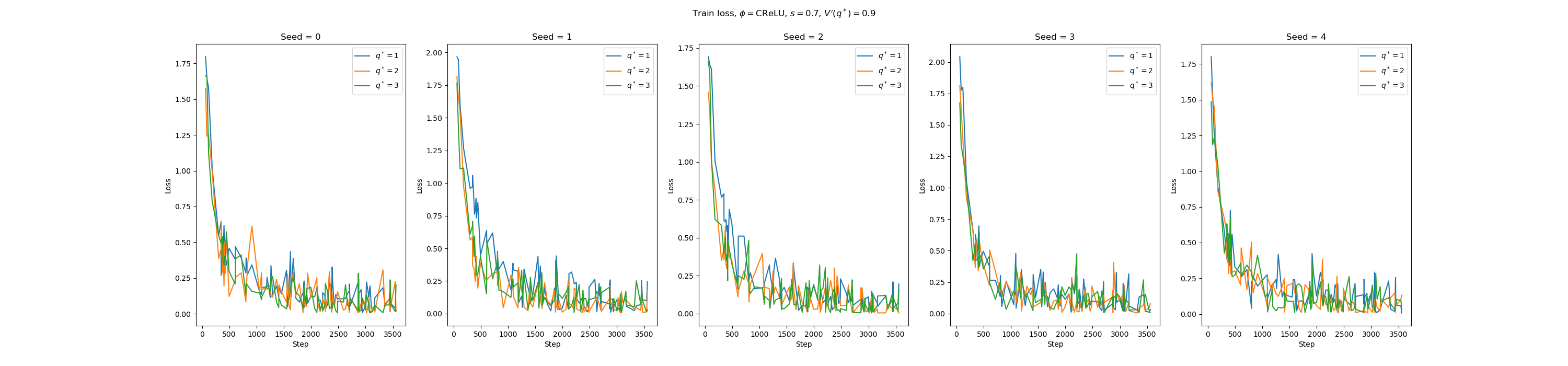}
    }
    \hfill
\caption{Training loss for a DNN with activation function $\phi = \text{CReLU}_{\tau, m}$, with $\tau$ and $m$ chosen such that $s=0.7$, and for a given $V'(q^*)$. Observe improved training speed for increased $q^*$ for all parameter sets.}
    \label{fig:val_loss_crelu_ind_seed_s_0.7}
\end{figure}

\begin{figure}[htbp!]
    \centering
    
    \subfigure[$s=0.8$, $V'(q^*)=0.5$]{
        \includegraphics[trim = 0cm 0cm 0cm 2cm, clip,height = 5cm, width=1\linewidth]{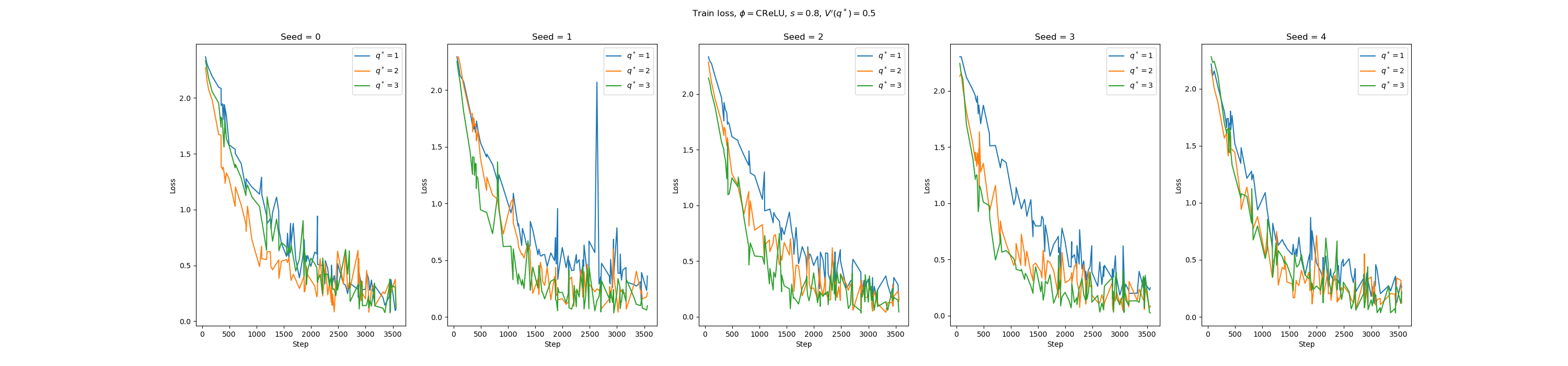}
    }
    \hfill
    \subfigure[$s=0.8$, $V'(q^*)=0.7$]{
        \includegraphics[trim = 0cm 0cm 0cm 2cm, clip,height = 5cm, width=1\linewidth]{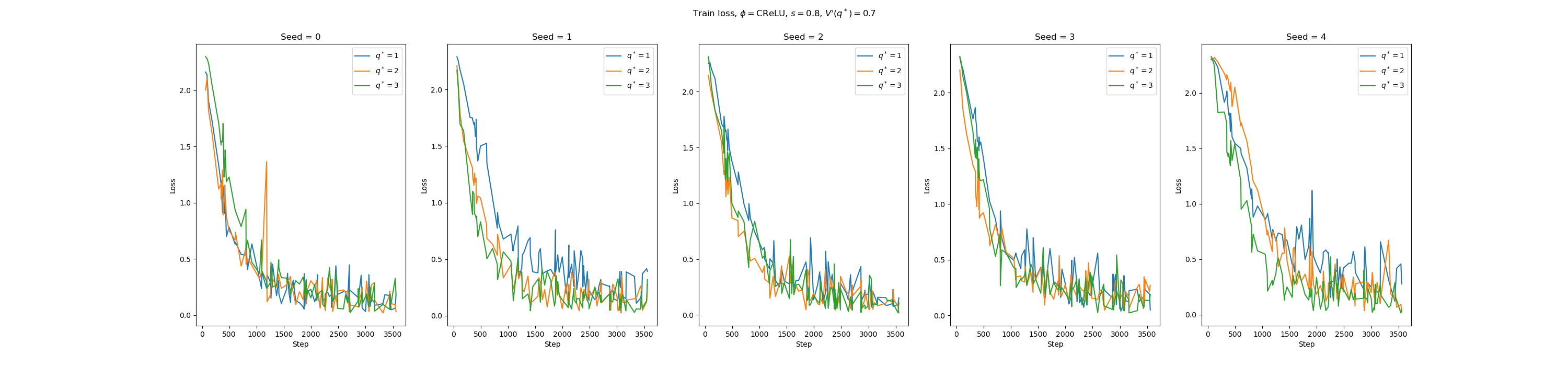}
    }
    \hfill
    \subfigure[$s=0.8$, $V'(q^*)=0.9$]{
        \includegraphics[trim = 0cm 0cm 0cm 2cm, clip,height = 5cm, width=1\linewidth]{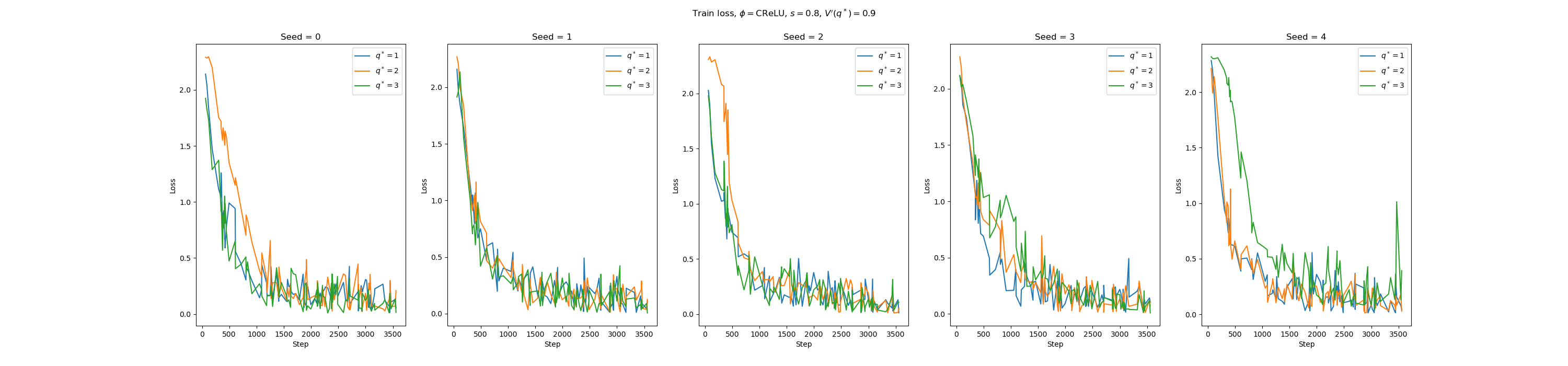}
    }
    \hfill
\caption{Training loss for a DNN with activation function $\phi = \text{CReLU}_{\tau, m}$, $\tau$ and $m$ chosen such that $s=0.8$, and for a given $V'(q^*)$. Observe improved training speed for increased $q^*$ for all parameter sets.}
    \label{fig:val_loss_crelu_ind_seed_s_0.8}
\end{figure}

\begin{figure}[htbp!]
    \centering
    \subfigure[$s=0.85$, $V'(q^*)=0.5$]{
        \includegraphics[trim = 0cm 0cm 0cm 2cm, clip,height = 5cm, width=1\linewidth]{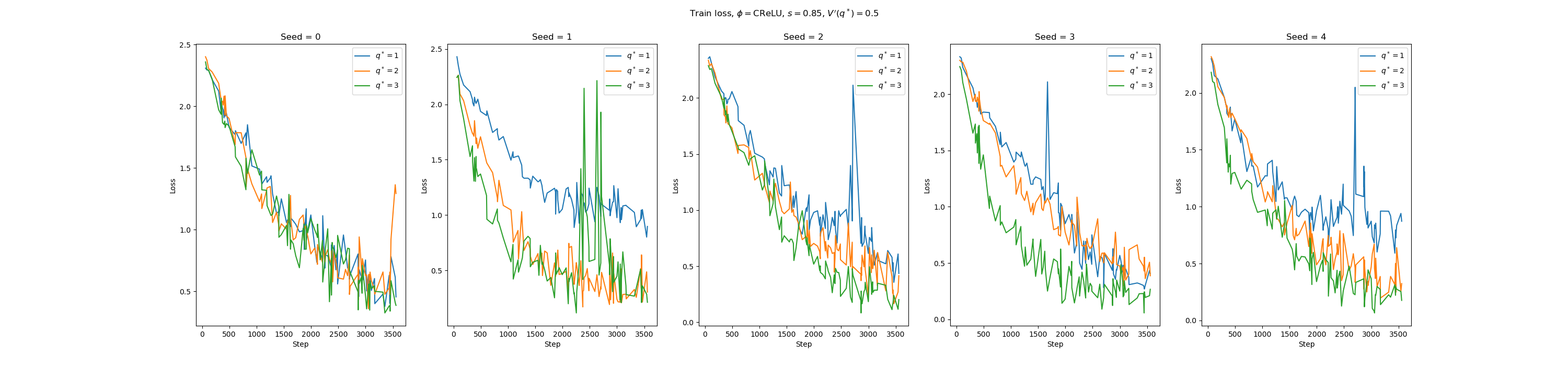}
    }
    \hfill
    \subfigure[$s=0.85$, $V'(q^*)=0.7$]{
        \includegraphics[trim = 0cm 0cm 0cm 2cm, clip,height = 5cm, width=1\linewidth]{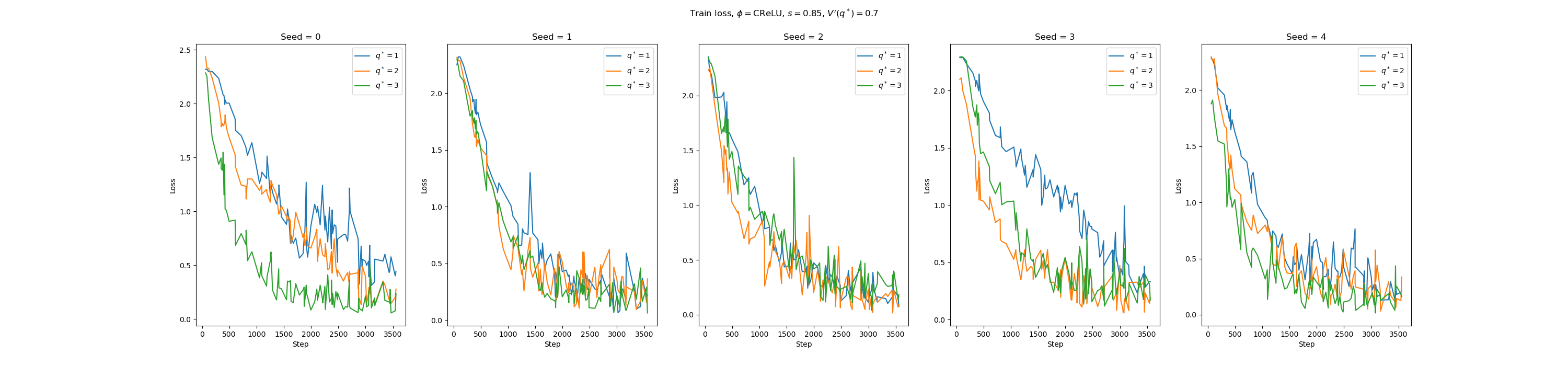}
    }
    \hfill
    \subfigure[$s=0.85$, $V'(q^*)=0.9$]{
        \includegraphics[trim = 0cm 0cm 0cm 2cm, clip,height = 5cm, width=1\linewidth]{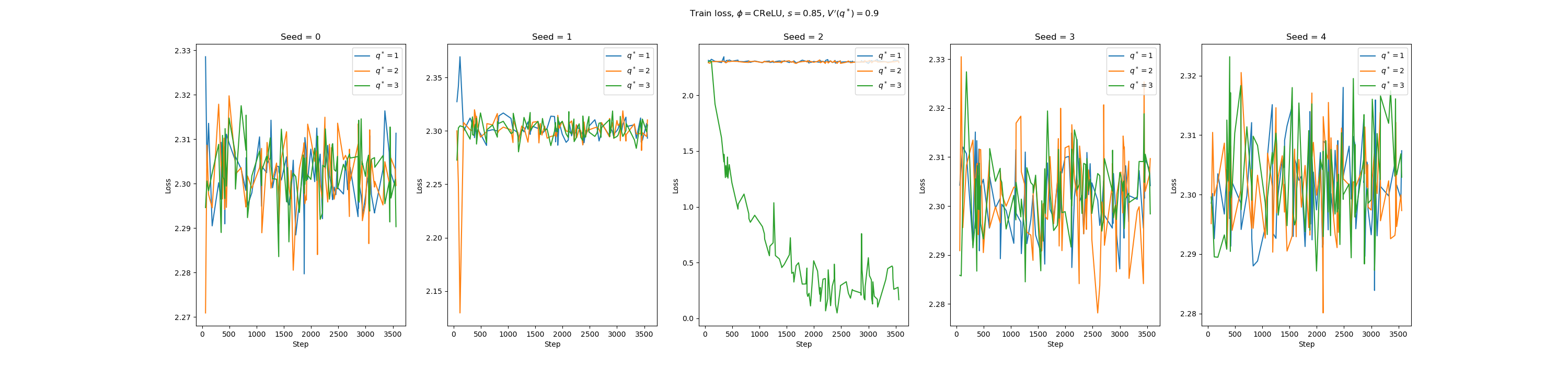}
    }
    \hfill
\caption{Training loss for a DNN with activation function $\phi = \text{CReLU}_{\tau, m}$, $\tau$ and $m$ chosen such that $s=0.85$, and for a given $V'(q^*)$. Observe improved training speed for increased $q^*$ for parameter sets where the network can consistently train, that is $V'(q^*)= \{ 0.5, 0.7\}$, and a regaining of the ability to train at all for $V'(q^*)=0.9$.}
    \label{fig:val_loss_crelu_ind_seed_s_0.85}
\end{figure}

\begin{figure}[htbp!]
    \centering
    \subfigure[$s=0.9$, $V'(q^*)=0.5$]{
        \includegraphics[trim = 0cm 0cm 0cm 2cm, clip, height = 5cm, width=1\linewidth]{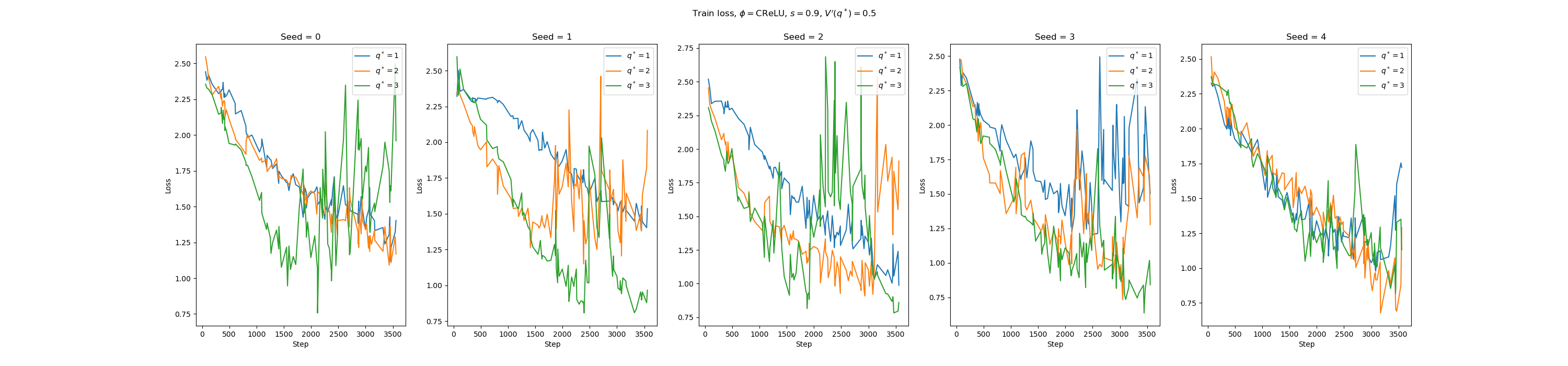}
    }
    \hfill
    \subfigure[$s=0.9$, $V'(q^*)=0.7$]{
        \includegraphics[trim = 0cm 0cm 0cm 2cm, clip, height = 5cm, width=1\linewidth]{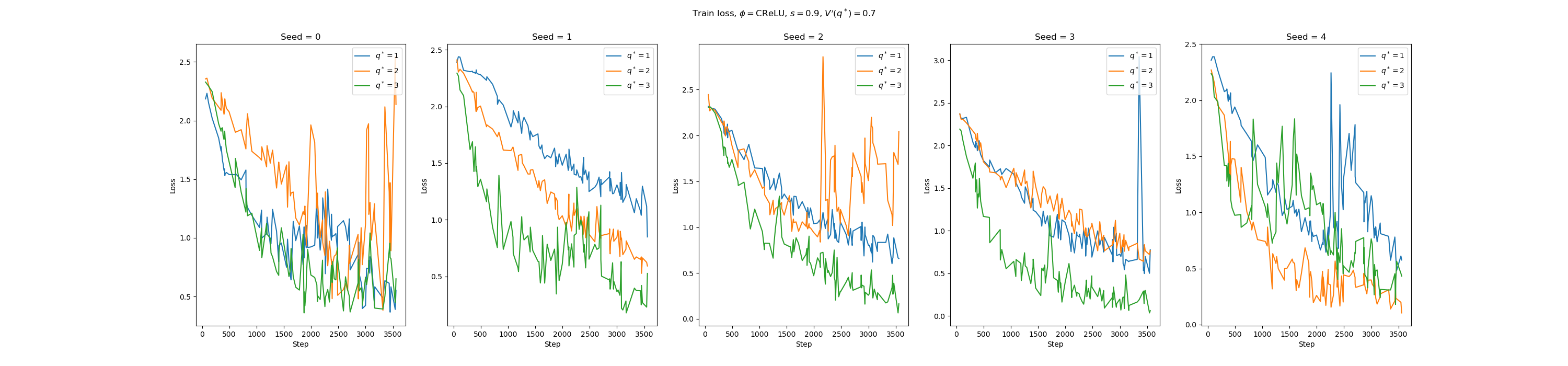}
    }
    \hfill
    \subfigure[$s=0.9$, $V'(q^*)=0.9$]{
        \includegraphics[trim = 0cm 0cm 0cm 2cm, clip,height = 5cm, width=1\linewidth]{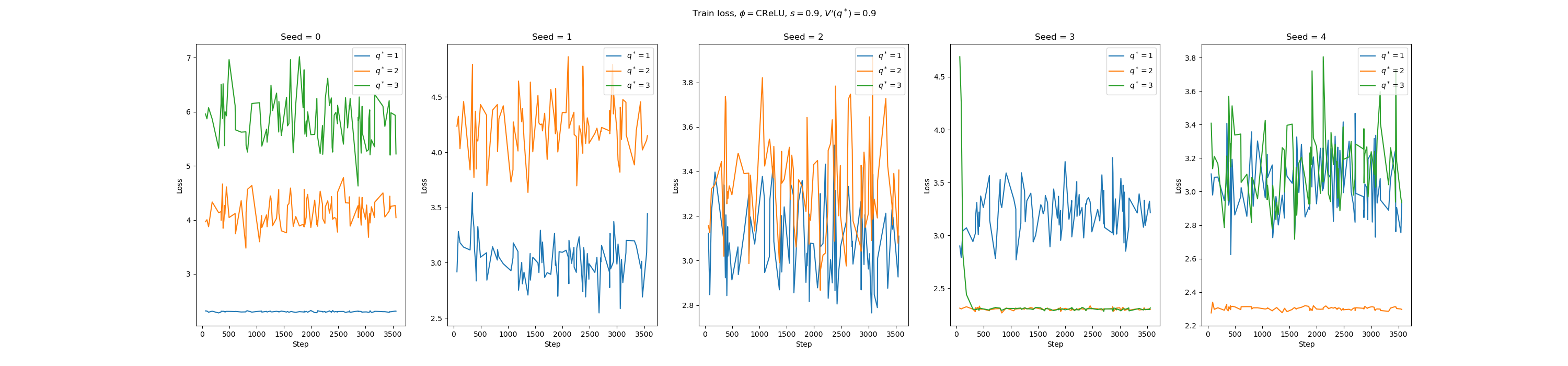}
    }
    \hfill
\caption{Training loss for a DNN with activation function $\phi = \text{CReLU}_{\tau, m}$, $\tau$ and $m$ chosen such that $s=0.9$, and for a given $V'(q^*)$. Observe improved training speed for increased $q^*$ for parameter sets where training is able, that is $V'(q^*)=\{ 0.5, 0.7\}$.}
    \label{fig:dnn_val_loss_crelu_ind_seed_s_0.9}
\end{figure}

\pagebreak
\subsection{\texorpdfstring{$\text{CST}_{\tau, m}$}{CST}}

\subsubsection{Summary table}

This section replicates many of the experiments from Section \ref{sec:experiments} but for the nonlinear activation $\text{CST}_{\tau, m}$ as defined in \eqref{cst_eq}. \cref{table_cst_dnn} is analogous to \cref{table_crelu_dnn}. Note also that for desired sparsity level $s$ with activation $\text{CST}_{\tau, m}$, set $\hat{\tau} = \sqrt{2q^*}\text{erf}^{-1}\left( s\right)$. \cref{table_cst_dnn} shows that generally, increased $q^*$ allows for improved test accuracy where otherwise full accuracy could not be retained.

\begin{table*}[htbp!]
\centering
\tiny
\caption{DNN Trained on MNIST, 100 layers and 300 width, using activation function $\text{CST}_{\tau, m}$, batch size 64, results on 20\% of data retained for testing. Observe the improved accuracy with increased $q^*=3$ for higher sparsity levels $s=\{ 0.7, 0.8, 0.85\}$.}

\begin{tabular}
{p{0.8cm}p{0.3cm}p{0.6cm}|p{0.35cm}p{0.65cm}|p{0.5cm}p{0.6cm}|p{0.35cm}p{0.65cm}|p{0.5cm}p{0.6cm}|p{0.35cm}p{0.65cm}|p{0.5cm}p{0.6cm}}
\toprule
\multicolumn{15}{c}{\textbf{CNN on CIFAR10}} \\ 
\midrule
 &   & & \multicolumn{4}{|c|}{$q^*=1$} & \multicolumn{4}{|c|}{$q^*=2$} & \multicolumn{4}{|c}{$q^*=3$} \\
\midrule
  & $s$ & $V'(q^*)$  & 
 $m$ & $V''(q^*)$  &\textbf{Accuracy} & \textbf{Sparsity} & $m$ & $V''(q^*)$ &\textbf{Accuracy} & \textbf{Sparsity} & $m$ & $V''(q^*)$ &\textbf{Accuracy} & \textbf{Sparsity} \\ 
\midrule
\multirow{12}{*}{$\text{CST}_{\tau, m}$} 
& 0.6& 0.5 & 0.89 & 1.72 & 0.88 & 0.60 & 1.26 & -0.14 & 0.88 & 0.60 & 1.54 & -0.05 & 0.88 & 0.60 \\
\rowcolor{lightgray}\cellcolor{white}&\cellcolor{white}&     0.7 & 1.27 & 2.30 & 0.89 & 0.60 & 1.79 &  0.08 & 0.89 & 0.60 &  2.19 &  0.03 & 0.89 & 0.60 \\
&&     0.9 & 1.85 & 3.18 & 0.89 & 0.60 & 2.62 &  0.40 & 0.90 & 0.57 &  3.21 &  0.13 & 0.74 & 0.53 \\
\cline{2-15}
& 0.7 & 0.5 & 0.81 & 1.49 & 0.89 & 0.70 & 1.14 & -0.05 & 0.87 & 0.70 & 1.40 & -0.02 & 0.88 & 0.70 \\
\rowcolor{lightgray}\cellcolor{white}&\cellcolor{white}&   0.7 & 1.17 & 2.05 & 0.88 & 0.70 & 1.66 &  0.21 & 0.88 & 0.70 &  2.03 &  0.07 & 0.88 & 0.70 \\
&&   0.9 & 1.74 & 2.90 & 0.27 & 0.42 & 2.46 &  0.58 & 0.11 & 0.31 &  3.01 &  0.19 & 0.58 & 0.50 \\
\cline{2-15}
& 0.8 & 0.5 & 0.72 & 1.26 & 0.77 & 0.80 & 1.02 &  0.11 & 0.85 & 0.80 &  1.25 &  0.04 & 0.91 & 0.80 \\
\rowcolor{lightgray}\cellcolor{white}&\cellcolor{white}&   0.7 & 1.06 & 1.79 & 0.89 & 0.80 & 1.50 &  0.41 & 0.89 & 0.80 &  1.84 &  0.14 & 0.89 & 0.80 \\
&&   0.9 & 1.61 & 2.62 & 0.10 & 0.32 & 2.28 &  0.84 & 0.11 & 0.28 &  2.79 &  0.28 & 0.10 & 0.26 \\
\cline{2-15}

&0.85 & 0.5 & 0.67 & 1.14 & 0.62 & 0.85 & 0.95 &  0.22 &  0.85 & 0.85 &  1.17 &  0.07 &0.83 & 0.84 \\
\rowcolor{lightgray}\cellcolor{white}&\cellcolor{white}&    0.7 & 1.00 & 1.66 & 0.87 & 0.85 & 1.42 &  0.56 & 0.82 & 0.84 &  1.73 &  0.19 & 0.90 & 0.85 \\
&&   0.9 & 1.53 & 2.46 & 0.10 & 0.62 & 2.17 &  1.05 & 0.10 & 0.51 &  2.66 & 0.35 & 0.10 & 0.41 \\
\end{tabular}
\label{table_cst_dnn}
\end{table*}

\clearpage

\subsubsection{Mean training loss plots}

\cref{fig:mean_train_loss_cst} shows the mean training loss for each parameter set of experiments of $\text{CST}_{\tau, m}$ across five seeds of a DNN. \cref{fig:mean_train_loss_cst} shows that, similarly to experiments for $\text{CReLU}_{\tau, m}$, there is improved training dynamics on average across parameter sets with increased $q^*$.

\begin{figure}[htbp!]
    \centering
    \subfigure[$s=0.6$, $V'(q^*)=0.5$]{
        \includegraphics[trim = 0cm 0cm 0cm 1.75cm, clip,width=0.3\linewidth]{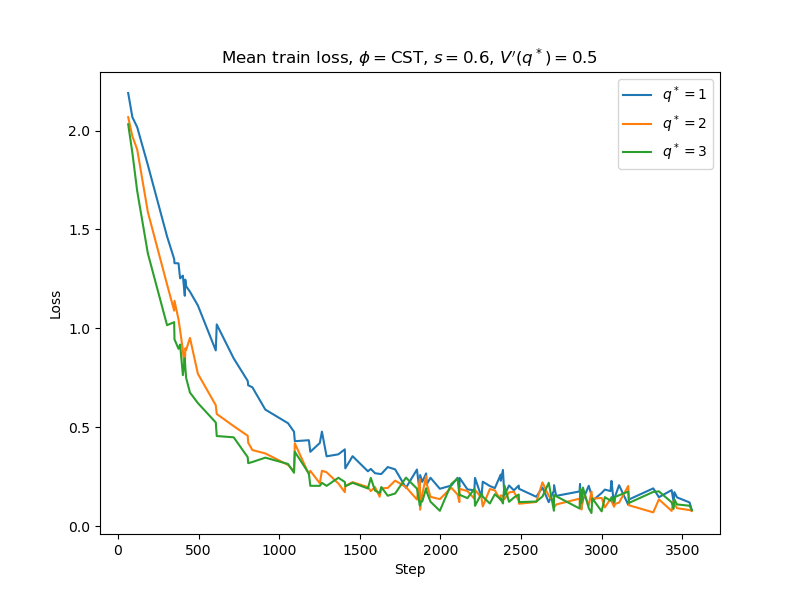}
    }
    \hfill
    \subfigure[$s=0.6$, $V'(q^*)=0.7$]{
        \includegraphics[trim = 0cm 0cm 0cm 1.75cm, clip,width=0.3\linewidth]{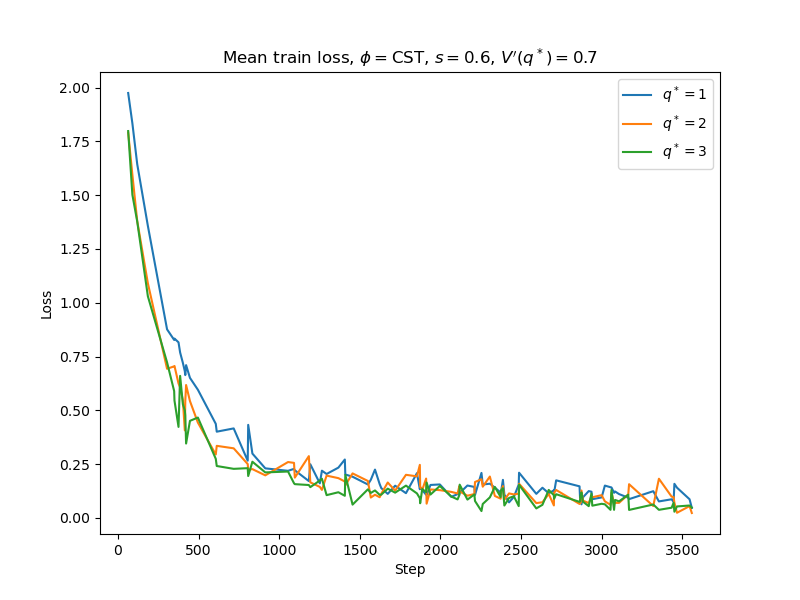}
    }
    \hfill
    \subfigure[$s=0.6$, $V'(q^*)=0.9$]{
        \includegraphics[trim = 0cm 0cm 0cm 1.75cm, clip,width=0.3\linewidth]{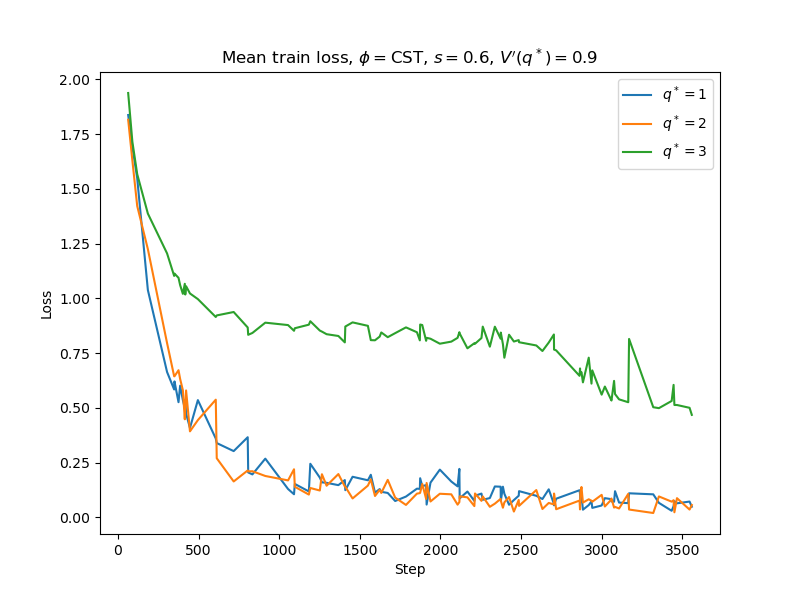}
    }
    \hfill
    \subfigure[$s=0.7$, $V'(q^*)=0.5$]{
        \includegraphics[trim = 0cm 0cm 0cm 1.75cm, clip,width=0.3\linewidth]{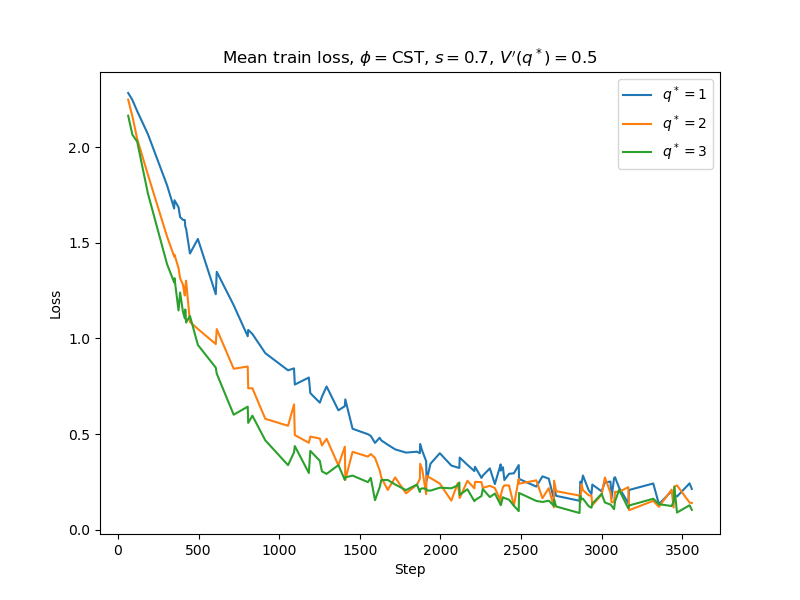}
    }
    \hfill
    \subfigure[$s=0.7$, $V'(q^*)=0.7$]{
        \includegraphics[trim = 0cm 0cm 0cm 1.75cm, clip,width=0.3\linewidth]{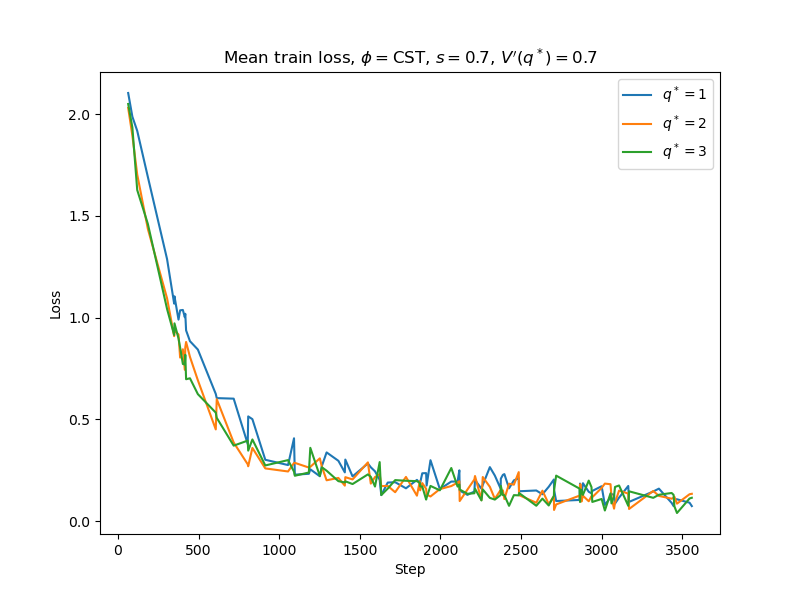}
    }
    \hfill
    \subfigure[$s=0.7$, $V'(q^*)=0.9$]{
        \includegraphics[trim = 0cm 0cm 0cm 1.75cm, clip,width=0.3\linewidth]{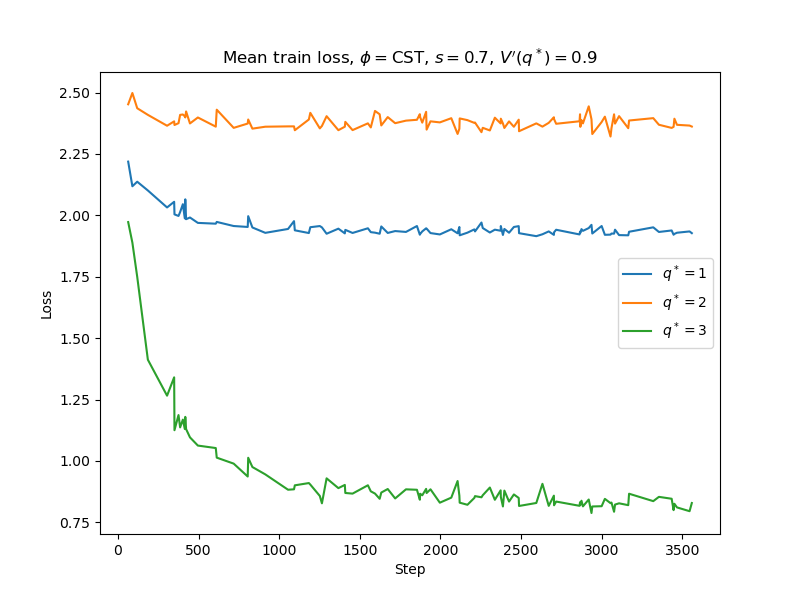}
    }
    \hfill
    \subfigure[$s=0.8$, $V'(q^*)=0.5$]{
        \includegraphics[trim = 0cm 0cm 0cm 1.75cm, clip,width=0.3\linewidth]{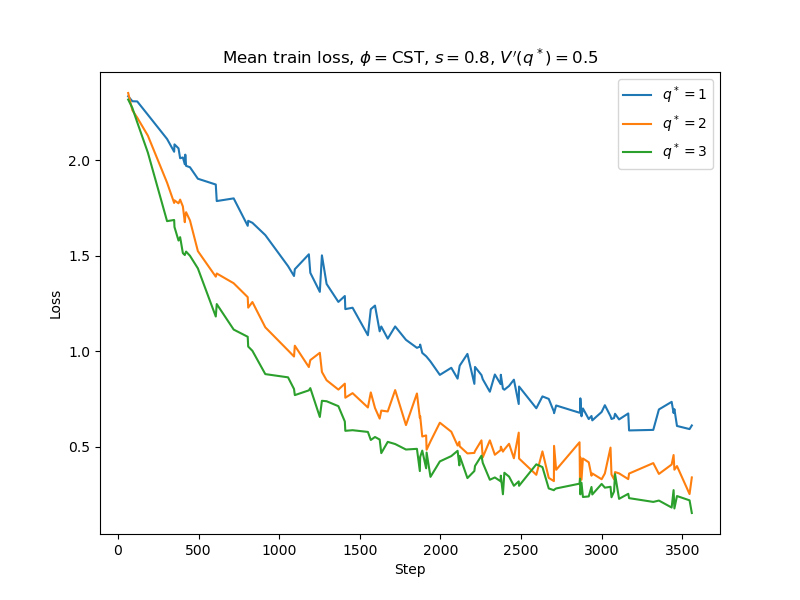}
    }
    \hfill
    \subfigure[$s=0.8$, $V'(q^*)=0.7$]{
        \includegraphics[trim = 0cm 0cm 0cm 1.75cm, clip,width=0.3\linewidth]{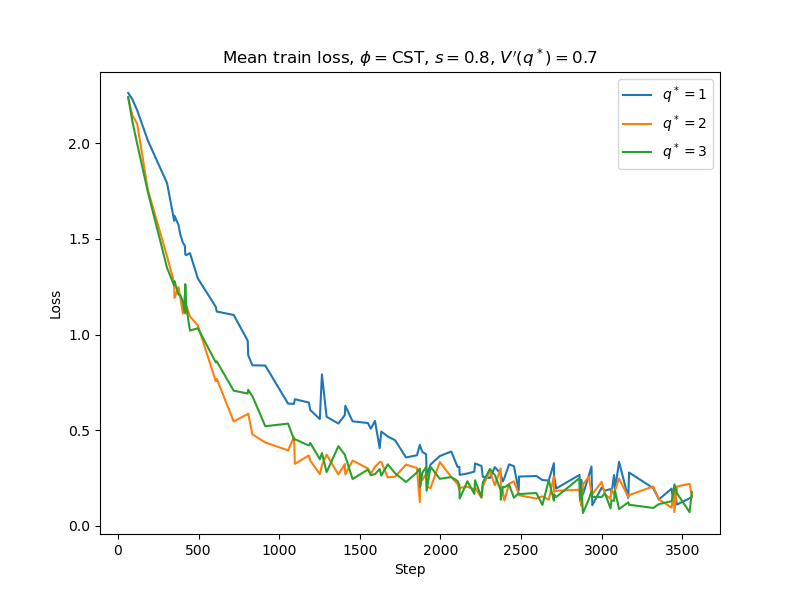}
    }
    \hfill
    \subfigure[$s=0.8$, $V'(q^*)=0.9$]{
        \includegraphics[trim = 0cm 0cm 0cm 1.75cm, clip,width=0.3\linewidth]{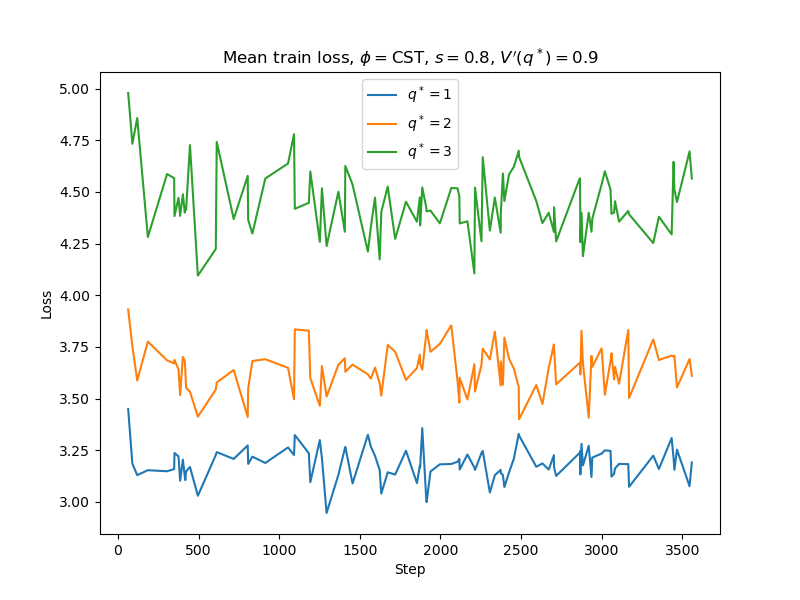}
    }
    \hfill
    \subfigure[$s=0.85$, $V'(q^*)=0.5$]{
        \includegraphics[trim = 0cm 0cm 0cm 1.75cm, clip,width=0.3\linewidth]{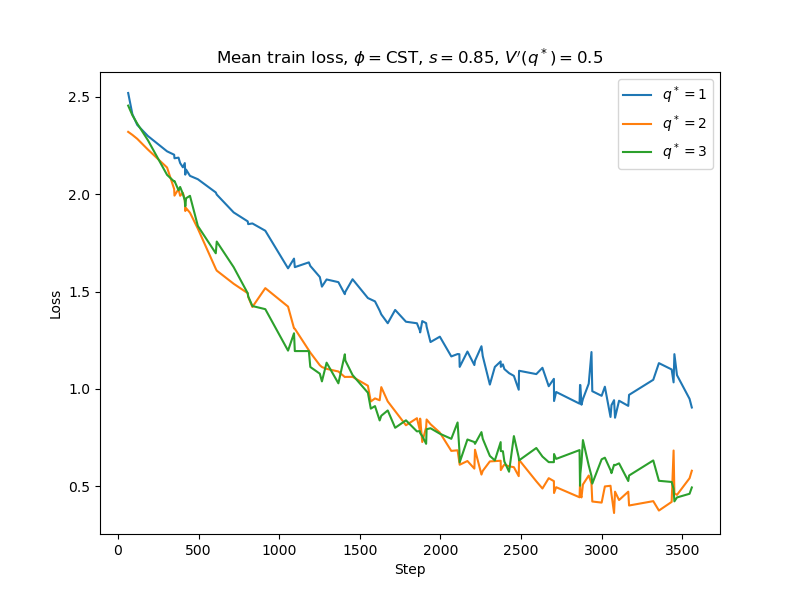}
    }
    \hfill
    \subfigure[$s=0.85$, $V'(q^*)=0.7$]{
        \includegraphics[trim = 0cm 0cm 0cm 1.75cm, clip,width=0.3\linewidth]{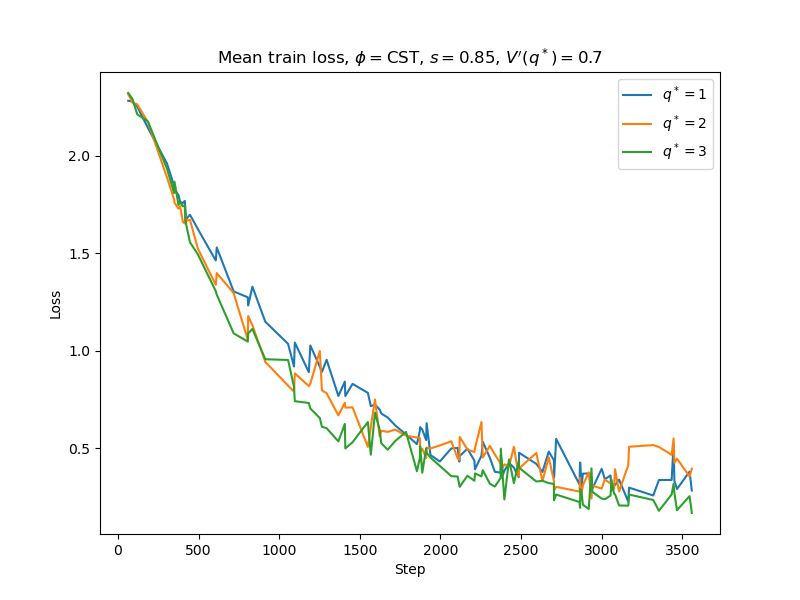}
    }
    \hfill
    \subfigure[$s=0.85$, $V'(q^*)=0.9$]{
        \includegraphics[trim = 0cm 0cm 0cm 1.75cm, clip,width=0.3\linewidth]{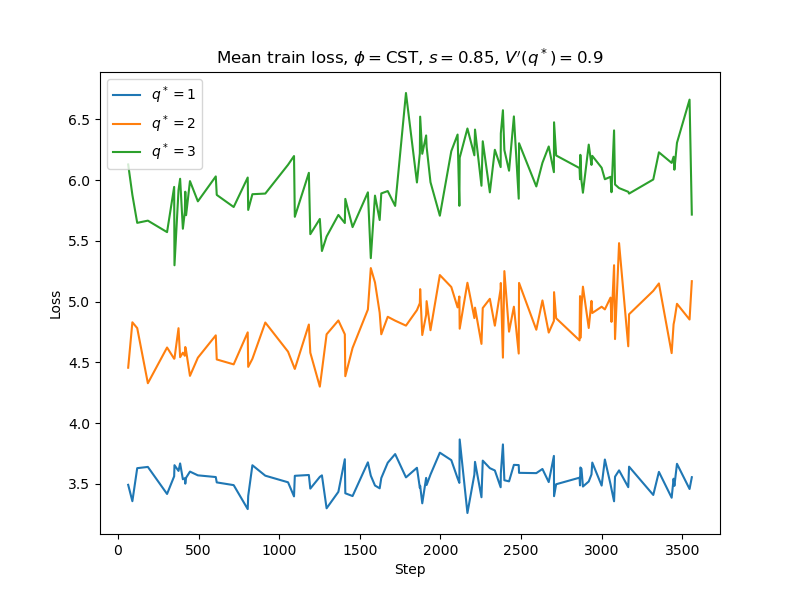}
    }
    \hfill
    \caption{Mean training loss across five seeds for a DNN with activation function $\phi = \text{CST}_{\tau, m}$, with $\tau$ and $m$ chosen such that $s$, $V'(q^*)$ are as described. Observe the improved training speed for larger sparsity levels $s=\{0.7, 0.8, 0.85\}$, with particular recovery of ability to train for $s=0.7$, $V'(q^*)=0.9$.}
    \label{fig:mean_train_loss_cst}
\end{figure}

\clearpage
\subsubsection{Individual run training loss plots}

Now Figures \ref{fig:dnn_val_loss_cst_ind_seed_s_0.6}-\ref{fig:dnn_train_loss_cst_ind_seed_s_0.85} show the training loss for each individual run of a DNN depth 100 and width 300 trained on MNIST. Figures \ref{fig:dnn_val_loss_cst_ind_seed_s_0.6}-\ref{fig:dnn_train_loss_cst_ind_seed_s_0.85} show the improved training speed across each individual run for increased $q^*$, in particular we see a recovery in ability to train for case $s=0.7$, $V'(q^*)=0.9$ for $\text{CST}_{\tau, m}$.

\begin{figure}[htbp!]
    \centering
    \subfigure[$s=0.6$, $V'(q^*)=0.5$]{
        \includegraphics[trim = 0cm 0cm 0cm 2cm, clip,height = 0.22\paperheight, width=1\linewidth]{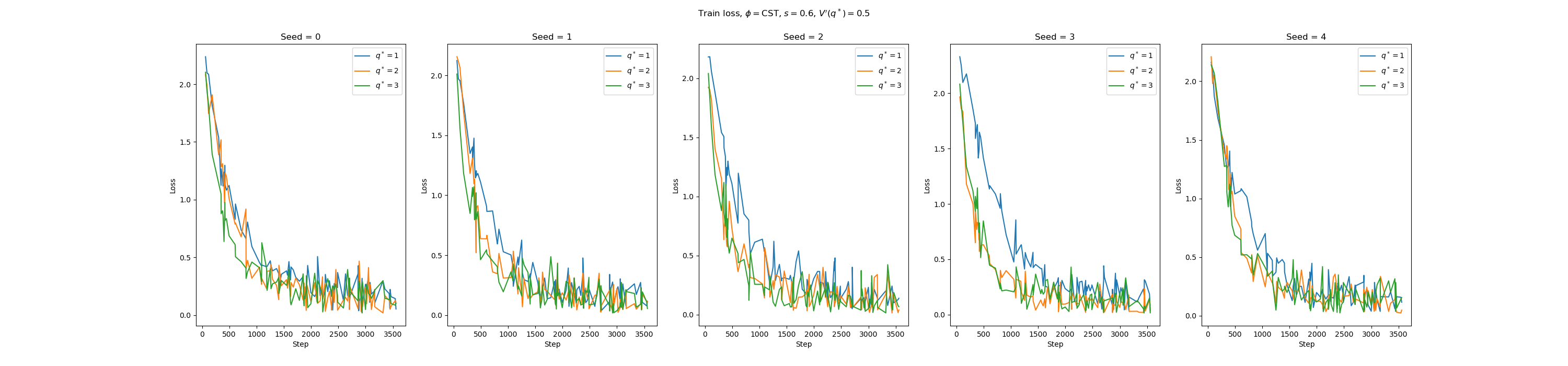}
    }
    \hfill
    \subfigure[$s=0.6$, $V'(q^*)=0.7$]{
        \includegraphics[trim = 0cm 0cm 0cm 2cm, clip,height = 5cm, width=\textwidth]{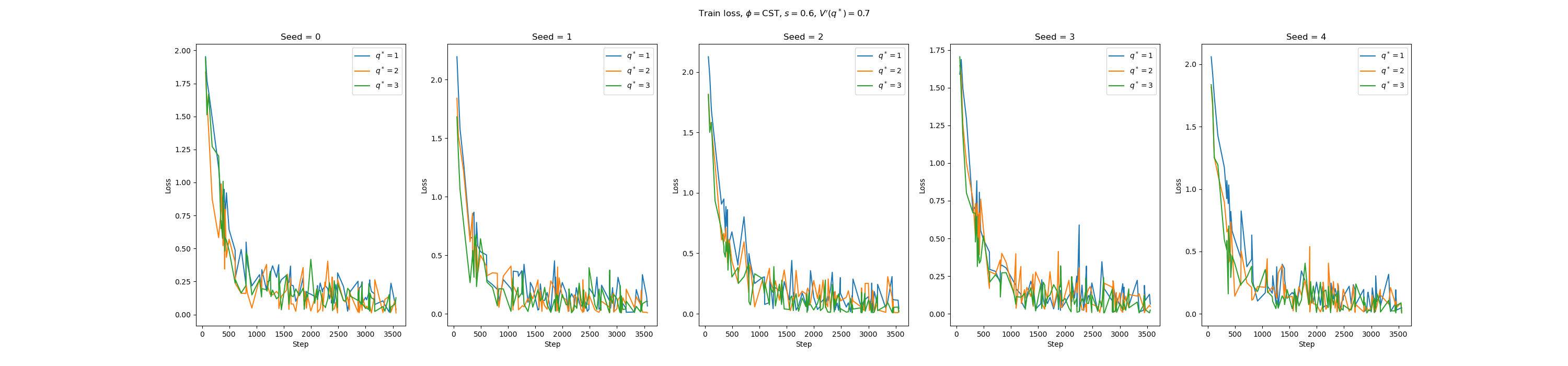}
    }
    \hfill
    \subfigure[$s=0.6$, $V'(q^*)=0.9$]{
        \includegraphics[trim = 0cm 0cm 0cm 2cm, clip,height = 5cm, width=\textwidth]{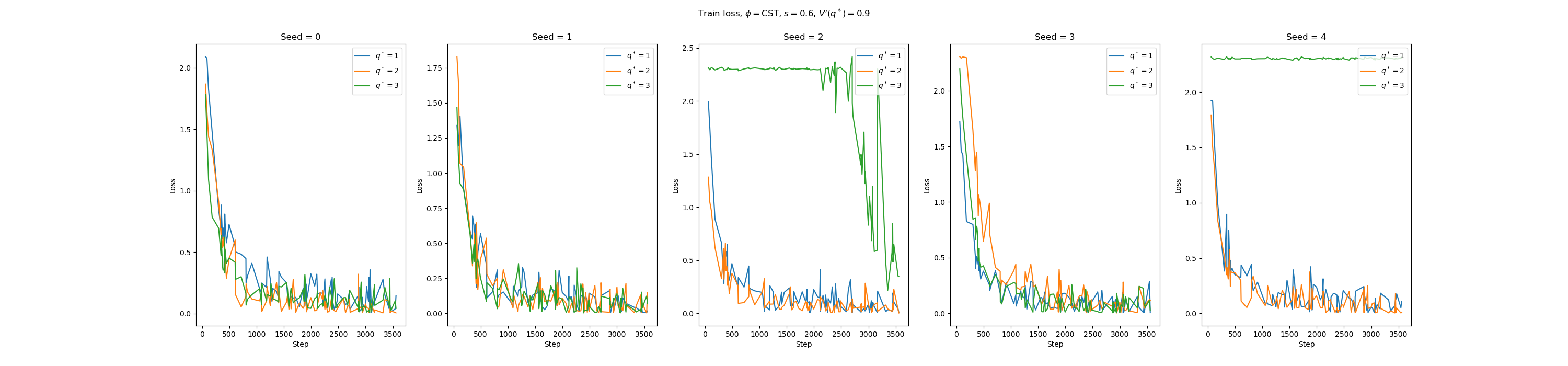}
    }
    \hfill
    
    \caption{Training loss for a DNN with activation function $\phi = \text{CST}_{\tau, m}$, with $\tau$ and $m$ chosen such that $s=0.6$, and for a given $V'(q^*)$. Observe the improved training speed for increased $q^*$ for $V'(q^*)=\{0.5, 0.7\}$}
    \label{fig:dnn_val_loss_cst_ind_seed_s_0.6}
\end{figure}

\begin{figure}[htbp!]
    \centering
    \subfigure[$s=0.7$, $V'(q^*)=0.5$]{
        \includegraphics[trim = 0cm 0cm 0cm 2cm, clip,height = 0.22\paperheight, width=1\linewidth]{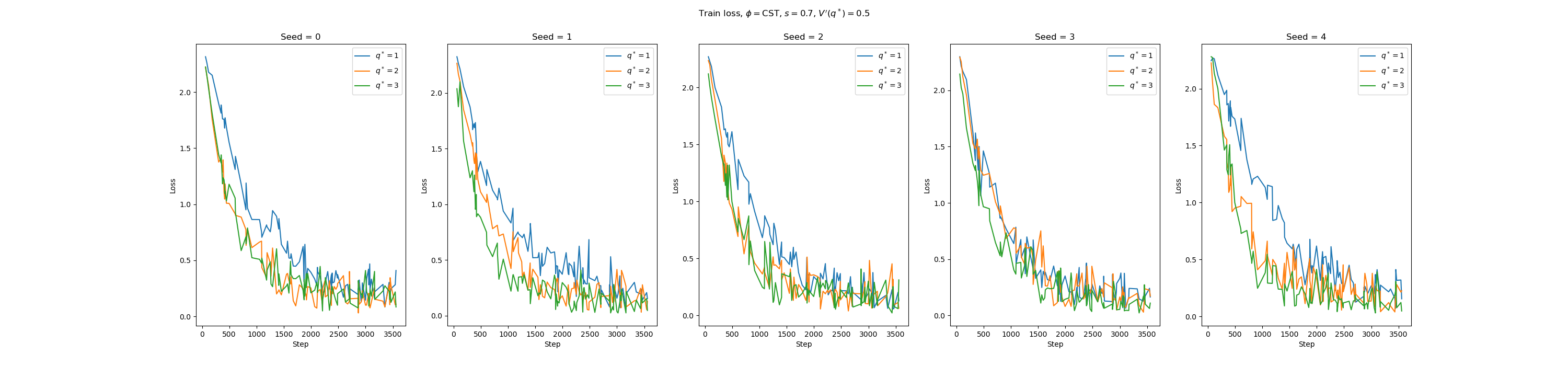}
    }
    \hfill
    \subfigure[$s=0.7$, $V'(q^*)=0.7$]{
        \includegraphics[trim = 0cm 0cm 0cm 2cm, clip,height = 0.22\paperheight, width=1\linewidth]{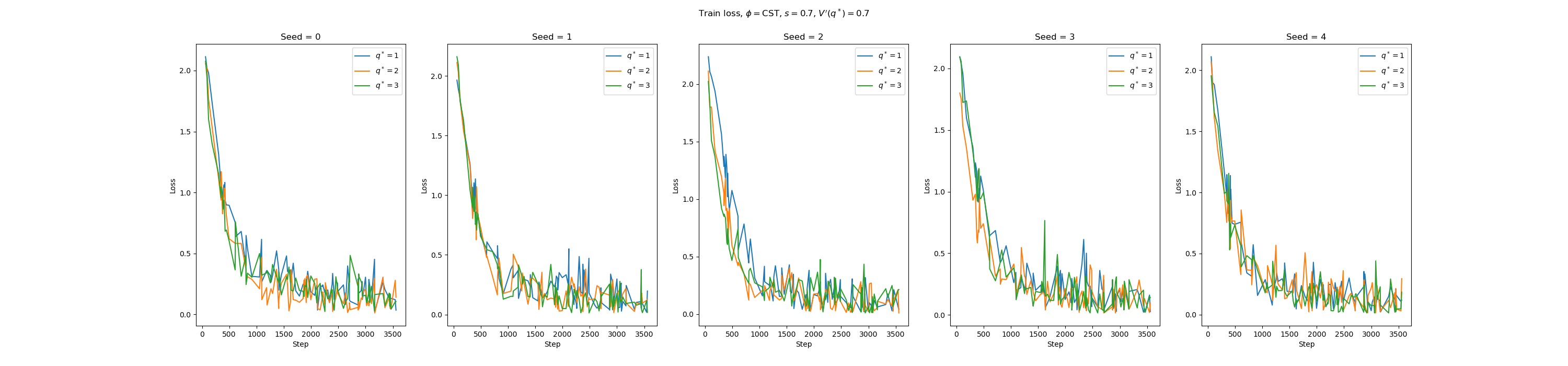}
    }
    \hfill
    \subfigure[$s=0.7$, $V'(q^*)=0.9$]{
        \includegraphics[trim = 0cm 0cm 0cm 2cm, clip,height = 0.22\paperheight, width=1\linewidth]{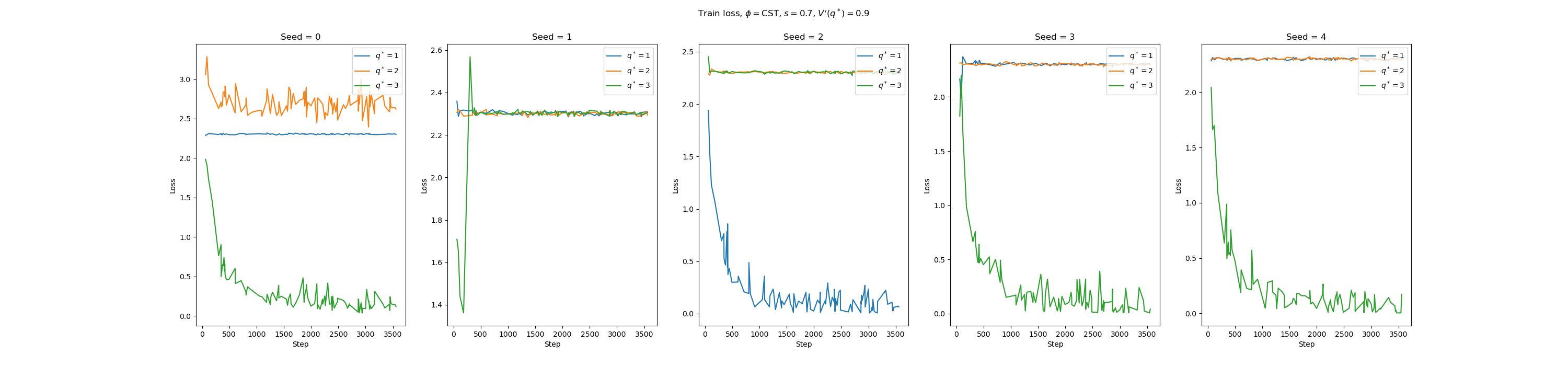}
    }
    \hfill
\caption{Training loss for a DNN with activation function $\phi = \text{CST}_{\tau, m}$, with $\tau$ and $m$ chosen such that $s=0.7$, and for a given $V'(q^*)$. Observe the improved training speed where the network can consistently train, that is $V'(q^*)=\{0.5, 0.7\}$, and recovery of ability to train for three seeds when $V'(q^*)=0.9$.}
    \label{fig:dnn_val_loss_cst_ind_seed_s_0.7}
\end{figure}

\begin{figure}[htbp!]
    \centering
    
    \subfigure[$s=0.8$, $V'(q^*)=0.5$]{
        \includegraphics[trim = 0cm 0cm 0cm 2cm, clip,height = 0.22\paperheight, width=1\linewidth]{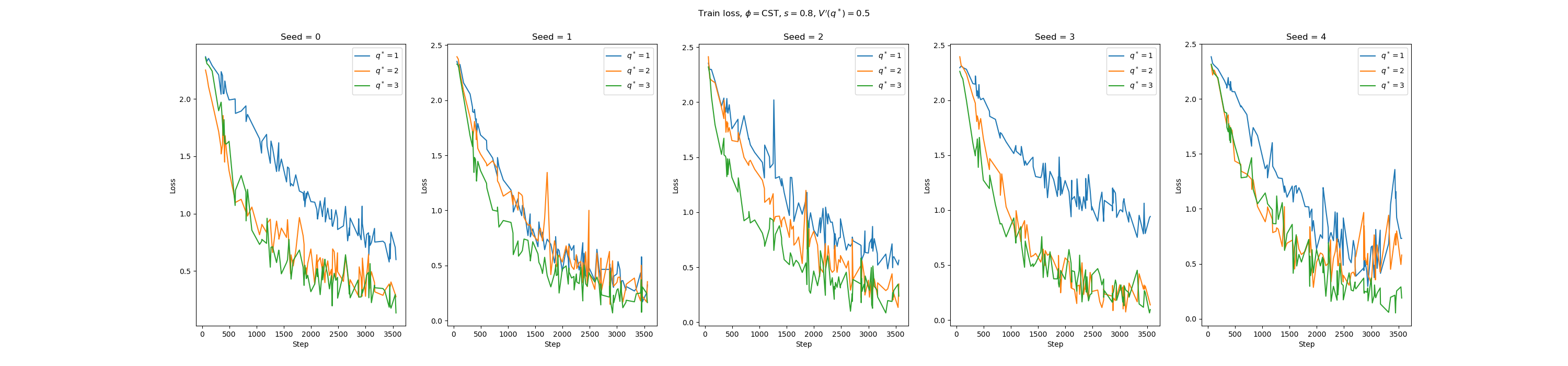}
    }
    \hfill
    \subfigure[$s=0.8$, $V'(q^*)=0.7$]{
        \includegraphics[trim = 0cm 0cm 0cm 2cm, clip,height = 0.22\paperheight, width=1\linewidth]{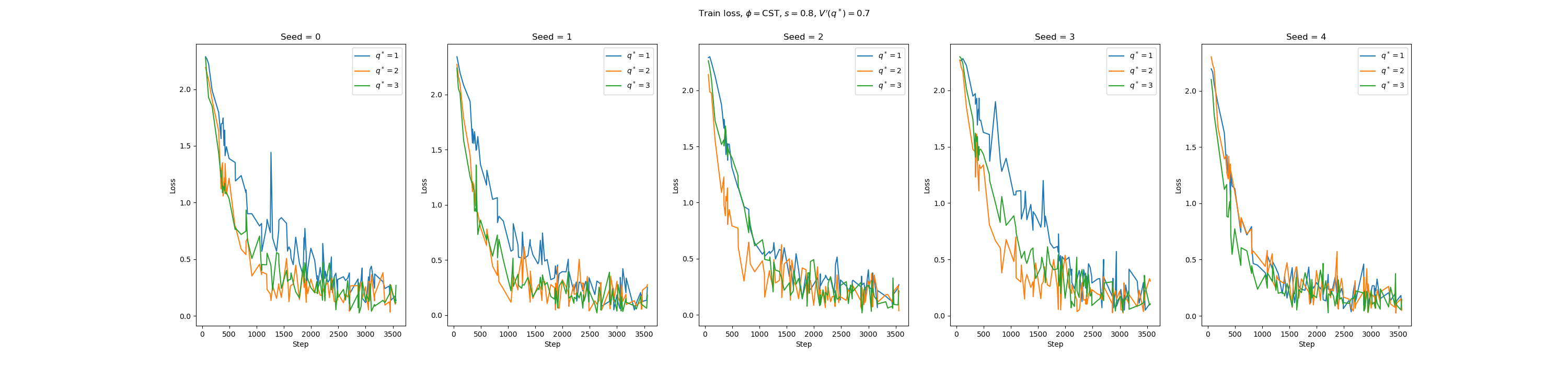}
    }
    \hfill
    \subfigure[$s=0.8$, $V'(q^*)=0.9$]{
        \includegraphics[trim = 0cm 0cm 0cm 2cm, clip,height = 0.22\paperheight, width=1\linewidth]{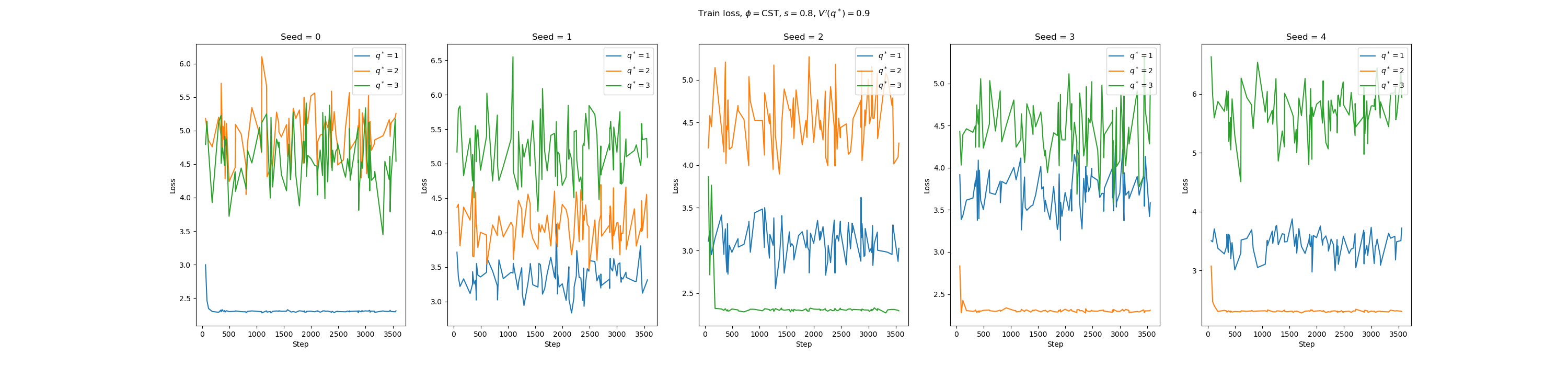}
    }
    \hfill
\caption{Training loss for a DNN with activation function $\phi = \text{CST}_{\tau, m}$, with $\tau$ and $m$ chosen such that $s=0.8$, and for a given $V'(q^*)$. Observe improved training speed for increased $q^*$ where the network can train, that is $V'(q^*)=\{0.5, 0.7\}$.}
    \label{fig:val_loss_cst_ind_seed_s_0.8}
\end{figure}

\begin{figure}[htbp!]
    \centering
    \subfigure[$s=0.85$, $V'(q^*)=0.5$]{
        \includegraphics[trim = 0cm 0cm 0cm 2cm, clip,height = 0.22\paperheight, width=1\linewidth]{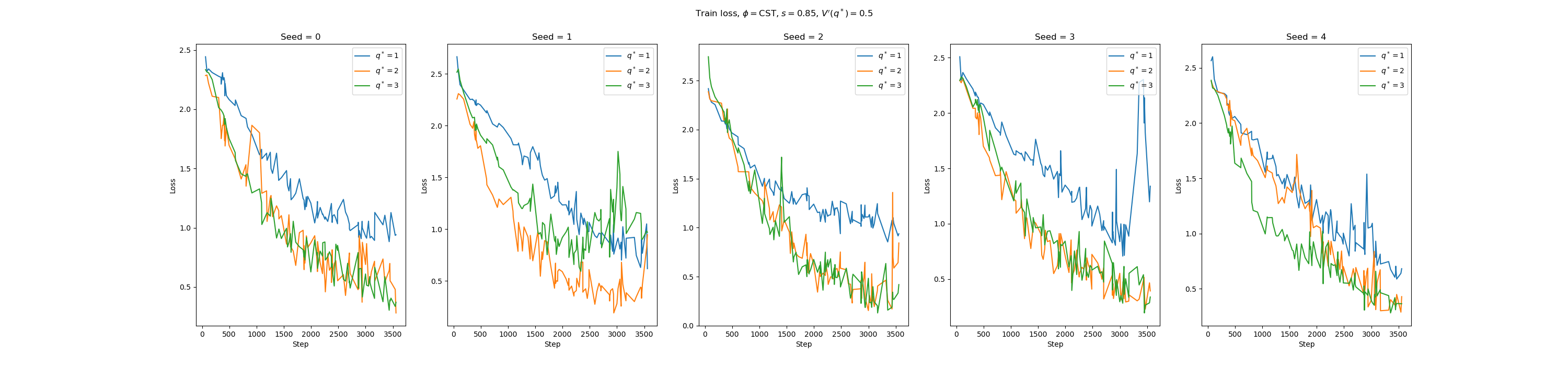}
    }
    \hfill
    \subfigure[$s=0.85$, $V'(q^*)=0.7$]{
        \includegraphics[trim = 0cm 0cm 0cm 2cm, clip,height = 0.22\paperheight, width=1\linewidth]{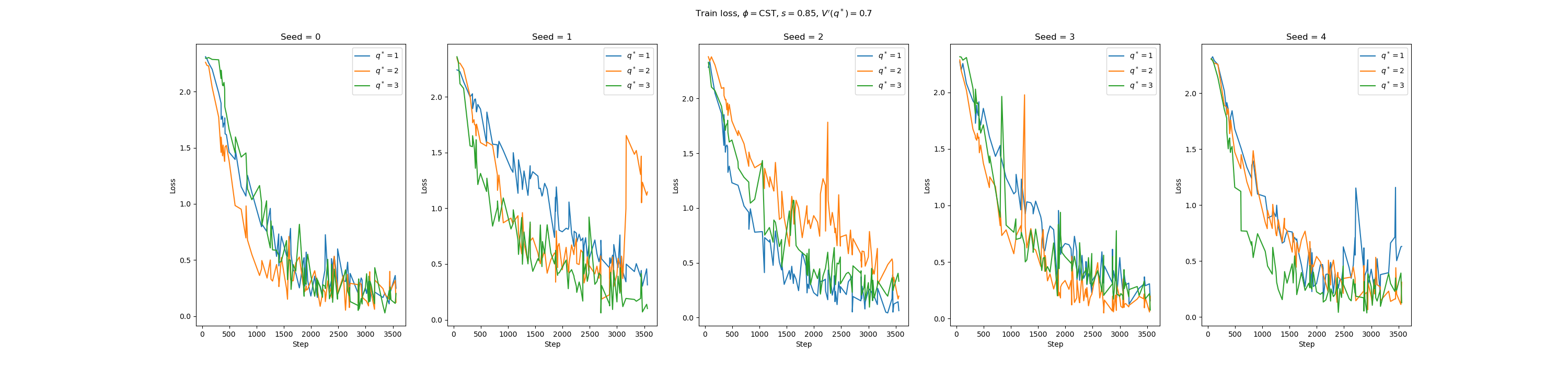}
    }
    \hfill
    \subfigure[$s=0.85$, $V'(q^*)=0.9$]{
        \includegraphics[trim = 0cm 0cm 0cm 2cm, clip,height = 0.22\paperheight, width=1\linewidth]{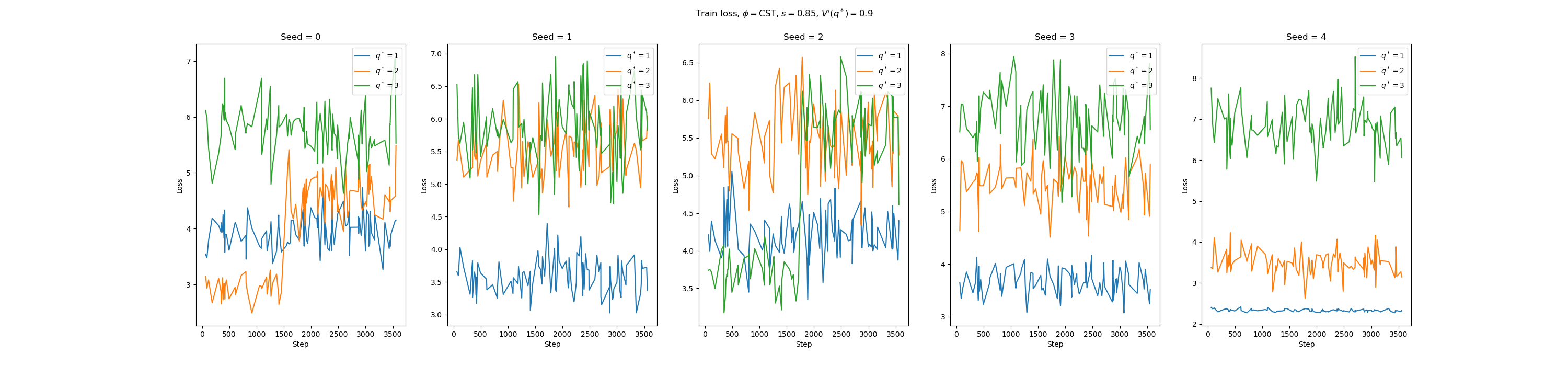}
    }
    \hfill
\caption{Training loss for a DNN with activation function $\phi = \text{CST}_{\tau, m}$, with $\tau$ and $m$ chosen such that $s=0.85$, and for a given $V'(q^*)$. Observe improved training speed for increased $q^*$ where the network can train, that is $V'(q^*)=\{0.5, 0.7\}$.}
    \label{fig:dnn_train_loss_cst_ind_seed_s_0.85}
\end{figure}

\clearpage

\section{Further CNN experiment results}\label{more_experiments_cnn}

This section provides further figures for experiments of CNNs (depth 50, 300 channels) on the training loss behaviour of increasing $q^*$ for $\phi = \text{CReLU}_{\tau, m}$, summarised in \cref{sec:experiments} on CIFAR10, and further on Tiny ImageNet. Each CNN experiment is conducted on a single H100 with 80GB of RAM, each CIFAR10 experiment takes approximately 1 hour to train, each Tiny ImageNet experiment takes approximately 3 hours 30 minutes to train. Total compute time for all experiments in this section is approximately 40 days.

\subsection{CIFAR10 individual run training loss plots}

Due to larger variability in training loss at each step the mean training loss across seeds for CIFAR10 is uninformative. Figures \ref{fig:cnn_val_loss_crelu_ind_seed_s_0.6}-\ref{fig:cnn_val_loss_crelu_ind_seed_s_0.85} show training loss for individual runs, generally demonstrating improved training dynamics for larger $q^*$, with particular recovery of ability to train for $s=0.9$, $V'(q^*)=0.7$ with $q^*=2$.

\begin{figure}[htbp!]
    \centering
    \subfigure[$s=0.6$, $V'(q^*)=0.5$]{
        \includegraphics[trim = 0cm 0cm 0cm 2cm, clip,height = 0.14\paperheight, width=1\linewidth]{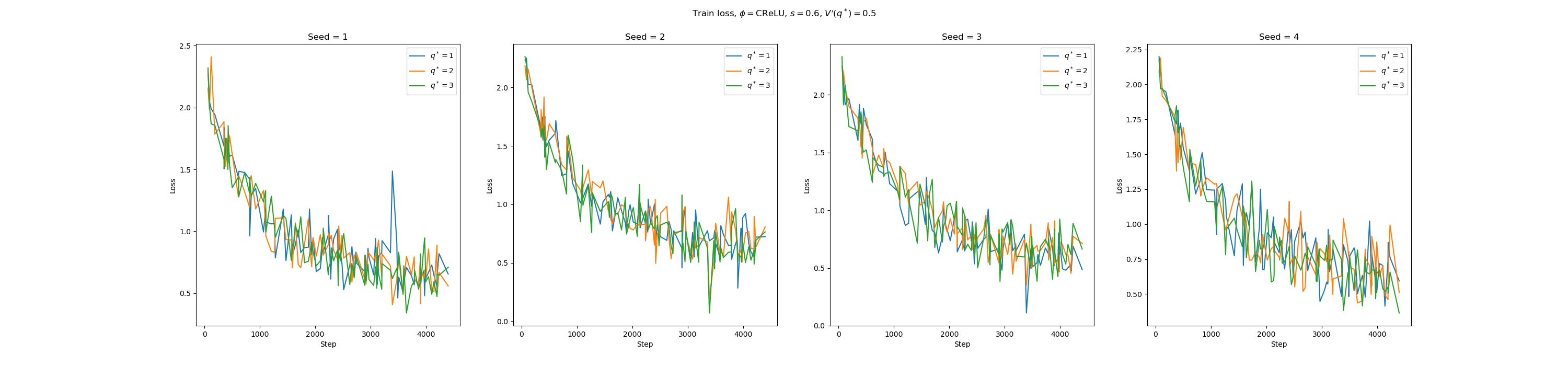}
    }
    \hfill
    \subfigure[$s=0.6$, $V'(q^*)=0.7$]{
        \includegraphics[trim = 0cm 0cm 0cm 2cm, clip,height = 0.14\paperheight, width=1\linewidth]{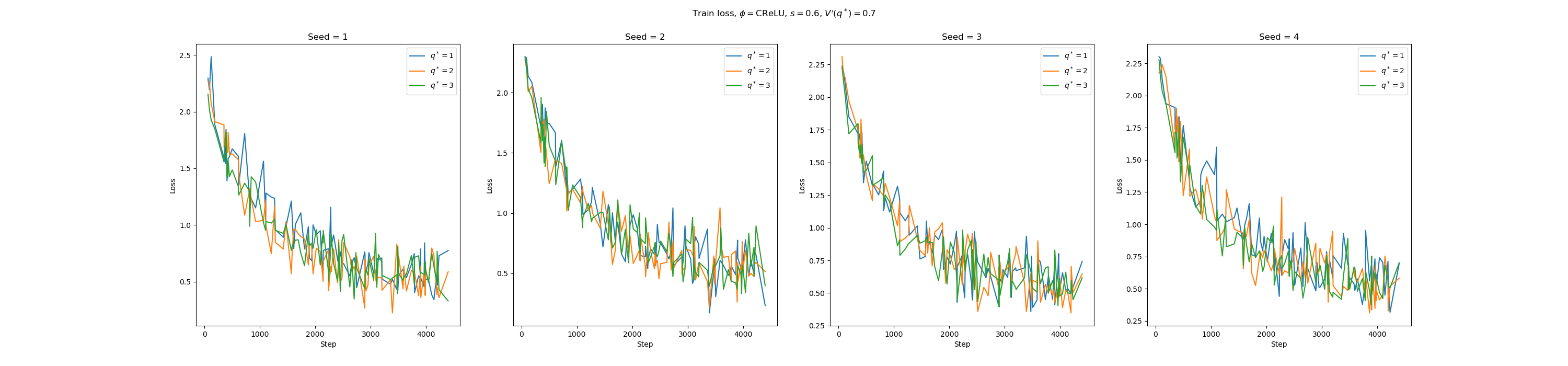}
    }
    \hfill
    \subfigure[$s=0.6$, $V'(q^*)=0.9$]{
        \includegraphics[trim = 0cm 0cm 0cm 2cm, clip,height = 0.14\paperheight, width=1\linewidth]{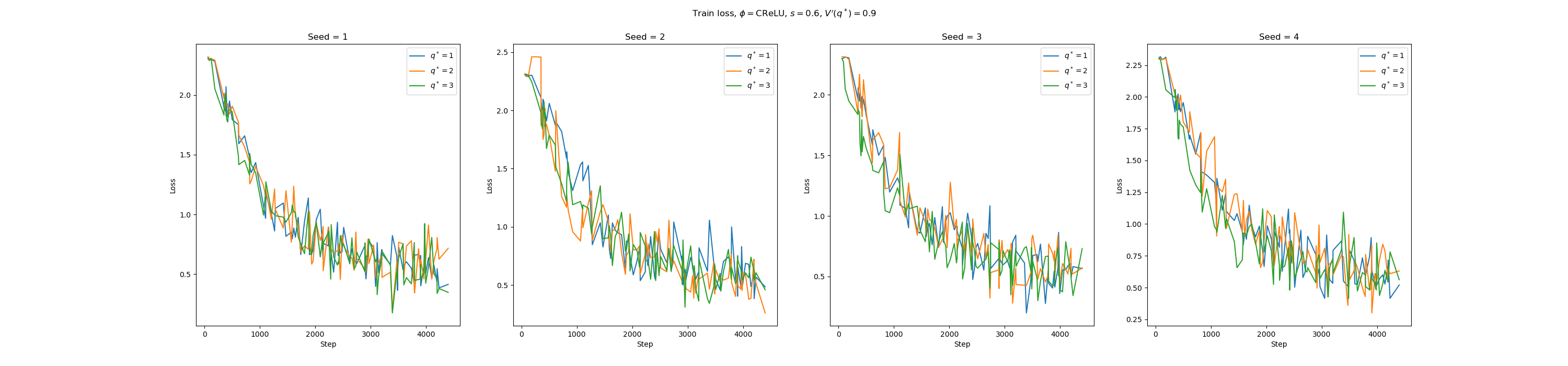}
    }
    \hfill
    
    \caption{Training loss for a CNN with activation function $\phi = \text{CReLU}_{\tau, m}$, $\tau$ and $m$ chosen such that $s=0.6$, and for a given $V'(q^*)$. Observe retained training dynamics for increased $q^*$ across all parameter sets.}
    \label{fig:cnn_val_loss_crelu_ind_seed_s_0.6}
\end{figure}

\begin{figure}[htbp!]
    \centering
    \subfigure[$s=0.7$, $V'(q^*)=0.5$]{
        \includegraphics[trim = 0cm 0cm 0cm 2cm, clip,height = 0.22\paperheight, width=1\linewidth]{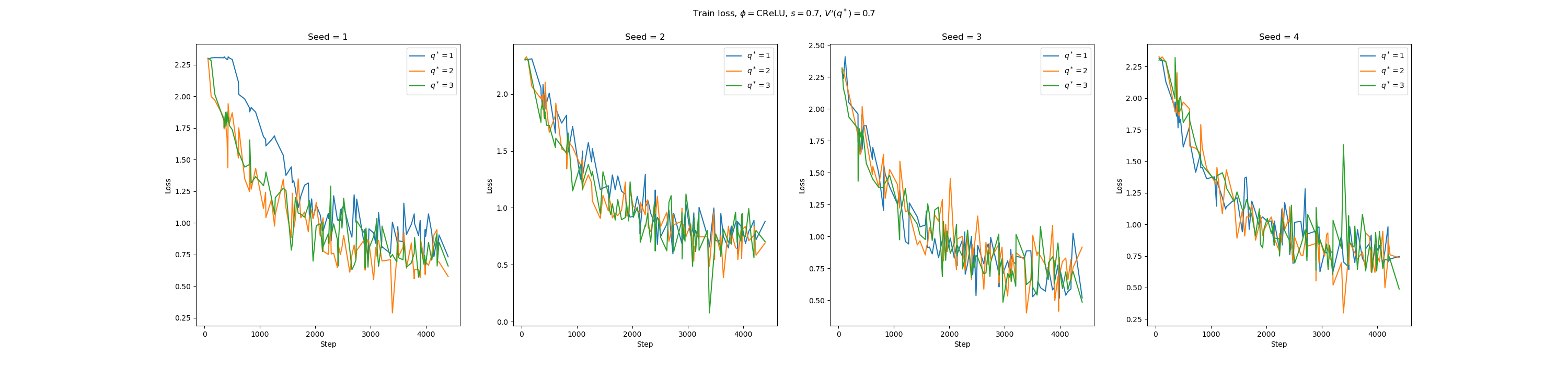}
    }
    \hfill
    \subfigure[$s=0.7$, $V'(q^*)=0.7$]{
        \includegraphics[trim = 0cm 0cm 0cm 2cm, clip,height = 0.22\paperheight, width=1\linewidth]{figures/cifar/cnn_cosine_train_loss_s_0.7_vprime_0.7_sfatrelu_max_individual_seed.png}
    }
    \hfill
    \subfigure[$s=0.7$, $V'(q^*)=0.9$]{
        \includegraphics[trim = 0cm 0cm 0cm 2cm, clip,height = 0.22\paperheight, width=1\linewidth]{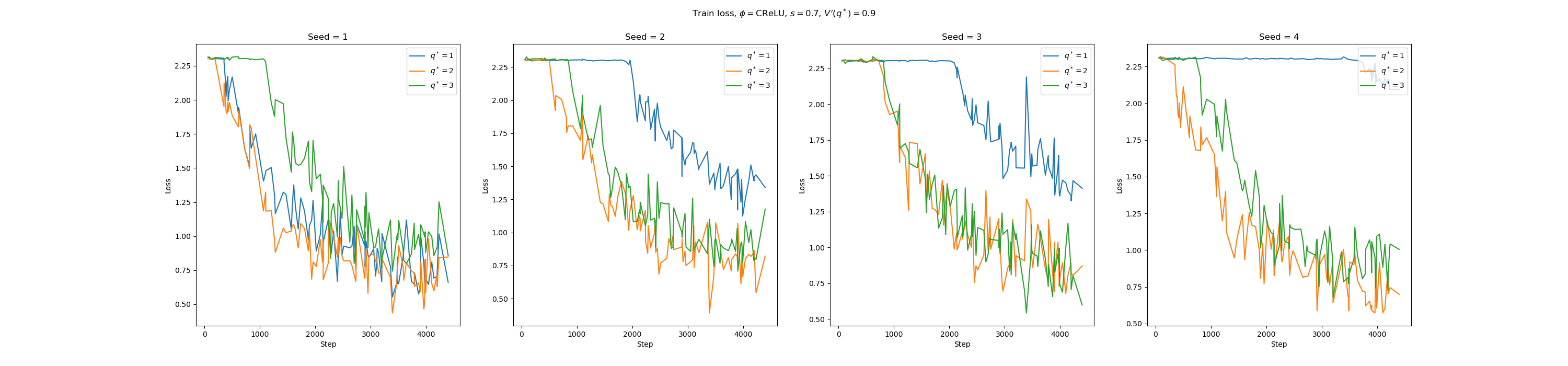}
    }
    \hfill
\caption{Training loss for a CNN with activation function $\phi = \text{CReLU}_{\tau, m}$, with $\tau$ and $m$ chosen such that $s=0.7$, and for a given $V'(q^*)$. Observe improved training speed for increased  $q^*$ across all parameter sets.}
    \label{fig:cnn_val_loss_crelu_ind_seed_s_0.7}
\end{figure}

\begin{figure}[htbp!]
    \centering
    
    \subfigure[$s=0.8$, $V'(q^*)=0.5$]{
        \includegraphics[trim = 0cm 0cm 0cm 2cm, clip,height = 0.22\paperheight, width=1\linewidth]{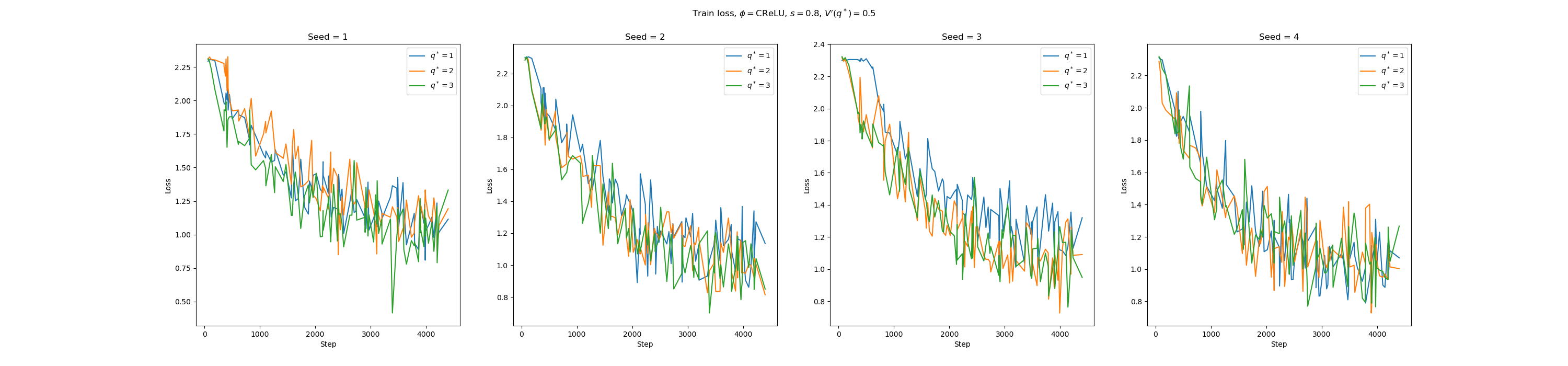}
    }
    \hfill
    \subfigure[$s=0.8$, $V'(q^*)=0.7$]{
        \includegraphics[trim = 0cm 0cm 0cm 2cm, clip,height = 0.22\paperheight, width=1\linewidth]{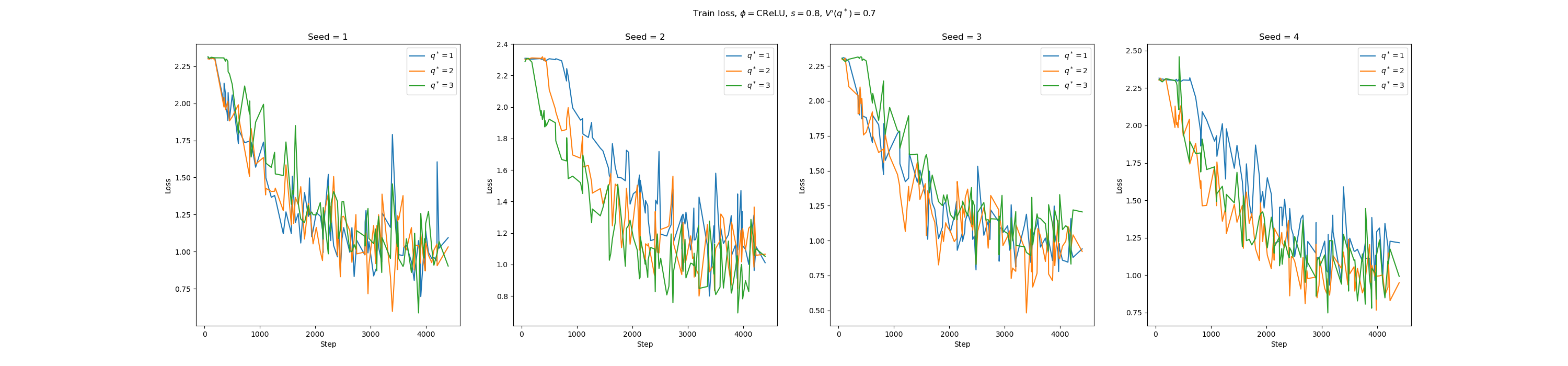}
    }
    \hfill
    \subfigure[$s=0.8$, $V'(q^*)=0.9$]{
        \includegraphics[trim = 0cm 0cm 0cm 2cm, clip,height = 0.22\paperheight, width=1\linewidth]{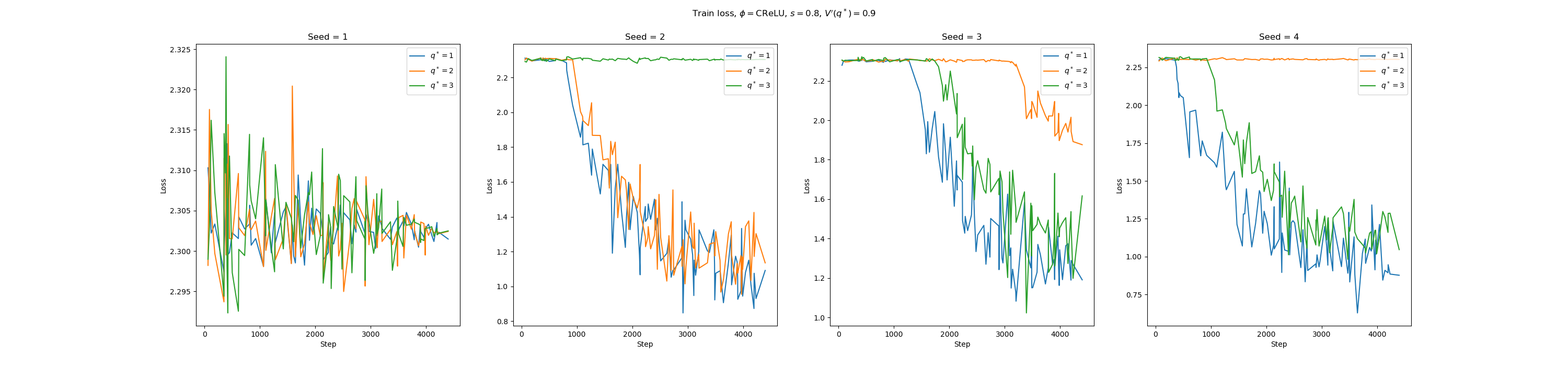}
    }
    \hfill
\caption{Training loss for a CNN with activation function $\phi = \text{CReLU}_{\tau, m}$, with $\tau$ and $m$ chosen such that $s=0.8$, and for a given $V'(q^*)$. Observe improved training speed for increased $q^*$ for cases where the network consistently trains, that is $V'(q^*)=\{0.5, 0.7\}$.}
    \label{fig:cnn_val_loss_crelu_ind_seed_s_0.8}
\end{figure}

\begin{figure}[htbp!]
    \centering
    \subfigure[$s=0.85$, $V'(q^*)=0.5$]{
        \includegraphics[trim = 0cm 0cm 0cm 2cm, clip,height = 0.22\paperheight, width=1\linewidth]{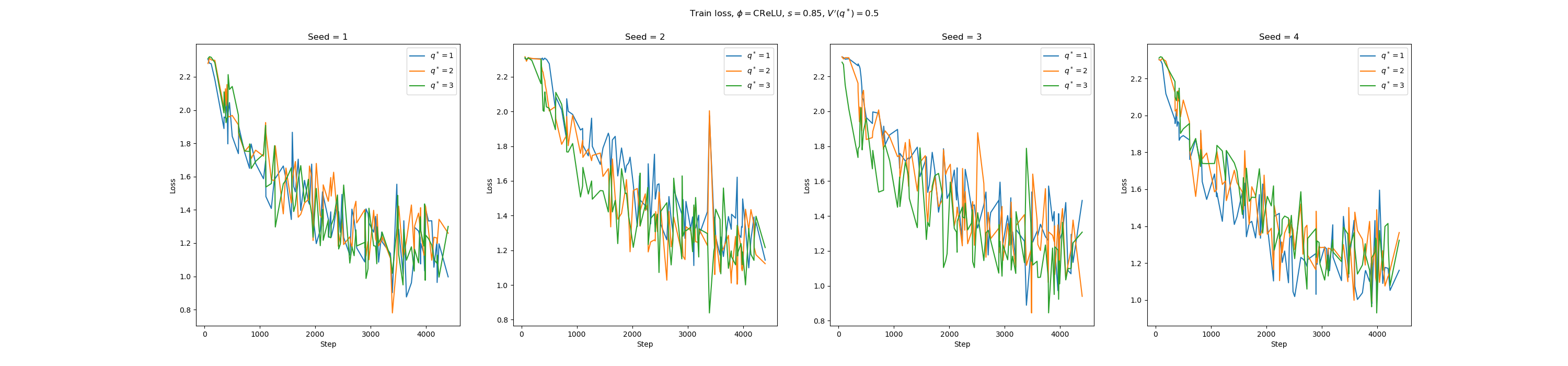}
    }
    \hfill
    \subfigure[$s=0.85$, $V'(q^*)=0.7$]{
        \includegraphics[trim = 0cm 0cm 0cm 2cm, clip,height = 0.22\paperheight, width=1\linewidth]{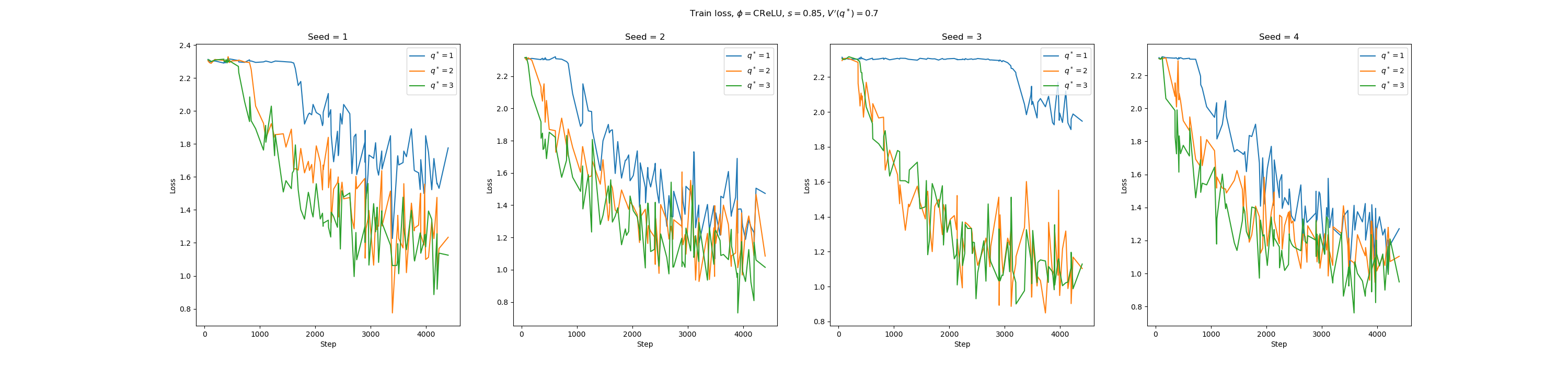}
    }
    \hfill
    \subfigure[$s=0.85$, $V'(q^*)=0.9$]{
        \includegraphics[trim = 0cm 0cm 0cm 2cm, clip,height = 0.22\paperheight, width=1\linewidth]{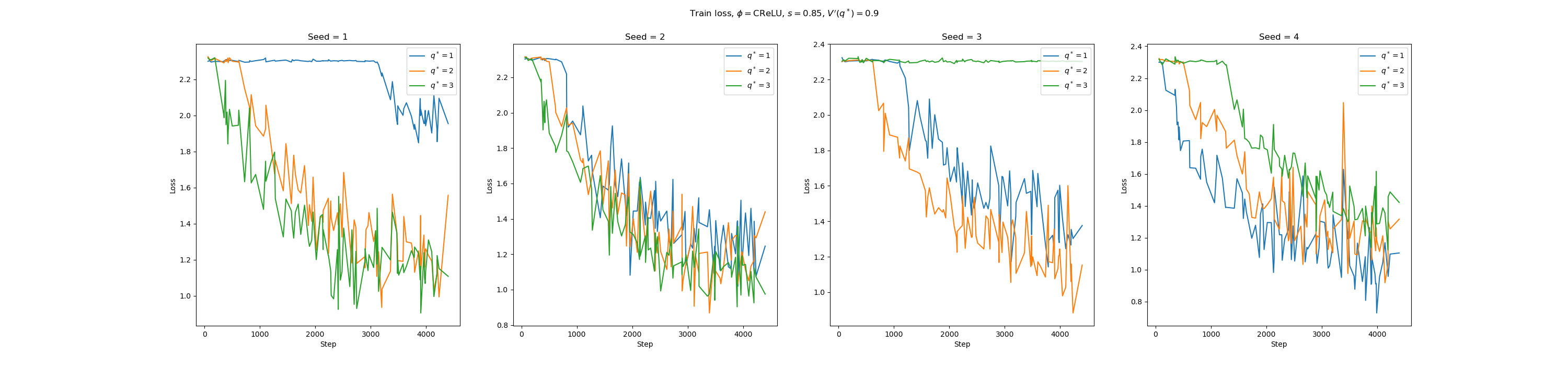}
    }
    \hfill
\caption{Training loss for a CNN with activation function $\phi = \text{CReLU}_{\tau, m}$, with $\tau$ and $m$ chosen such that $s=0.85$, and for a given $V'(q^*)$. Observe improved training speed for increased $q^*$ across all parameter sets.}
    \label{fig:cnn_val_loss_crelu_ind_seed_s_0.85}
\end{figure}

\clearpage

\subsection{Tiny ImageNet experiments}\label{app:tiny_imagenet}

Further tests of the CNN architecture of depth 50 and 300 layers were performed on the Tiny ImageNet dataset for $\phi= \text{CReLU}_{\tau, m}$. The CNN was trained using SGD with cosine learning rate $\eta=10^{-3}$ and batch size 64. Below, \cref{fig:tinyimagenet_training_mean_plots} are the mean training loss and validation accuracy across five seeds for the parameters $s=0.85$, $V'(q^*)=0.7$, see the improved training for larger $q^*=3$ versus $q^*=1$. Further complete tests of the mean training loss for the parameter set of $\text{CReLU}_{\tau, m}$ can be seen in \cref{fig:mean_train_loss_crelu_tinyimagenet}.

\begin{figure}[htbp!]
    \centering
    \subfigure[]{\includegraphics[width=0.45\linewidth]{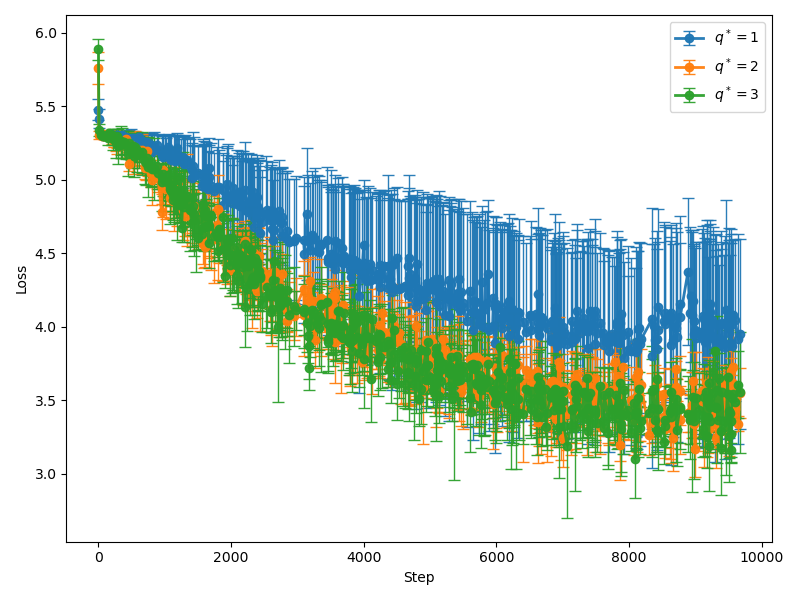}
    }
    \subfigure[]{\includegraphics[width=0.45\linewidth]{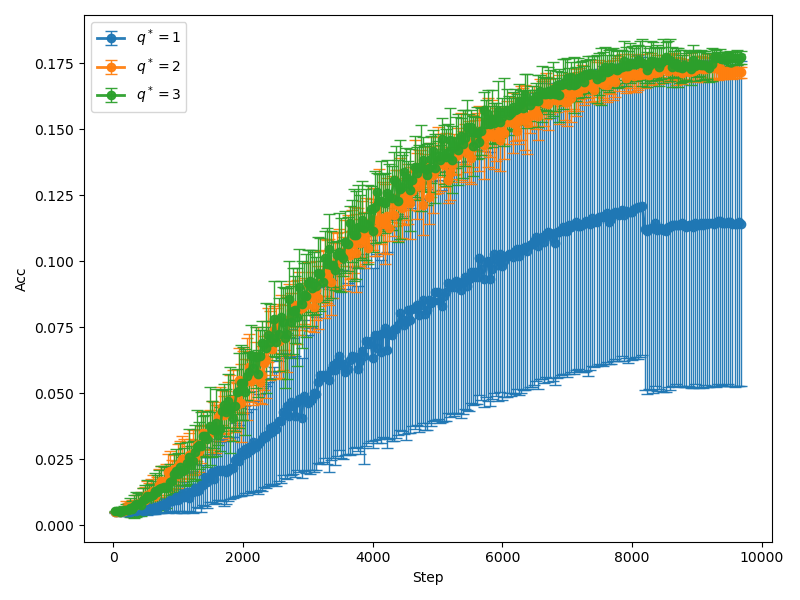}}
     \caption{Mean training loss, (a), and mean validation accuracy, (b), against step, for a CNN trained on Tiny ImageNet across five seeds, with standard deviation error bars. The CNN has activation function $\phi = \text{CReLU}_{\tau, m}$, $\tau$ and $m$ are chosen such that $s=0.85$ and $V'(q^*) = 0.7$, across $q^*=\{ 1, 2, 3 \}$. For increasingly larger of $q^*=\{ 1, 2, 3 \}$, the training loss and validation accuracy converge faster.}
    \label{fig:tinyimagenet_training_mean_plots}
\end{figure}

\clearpage
\subsubsection{Mean training loss plots}
\begin{figure}[htbp!]
    \centering
    \subfigure[$s=0.6$, $V'(q^*)=0.5$]{
        \includegraphics[height = 0.145\paperheight, width=0.25\linewidth]{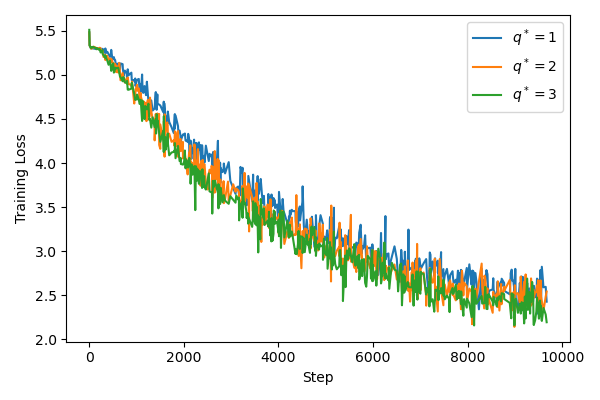}
    }
    \hfill
    \subfigure[$s=0.6$, $V'(q^*)=0.7$]{
        \includegraphics[height = 0.145\paperheight, width=0.25\linewidth]{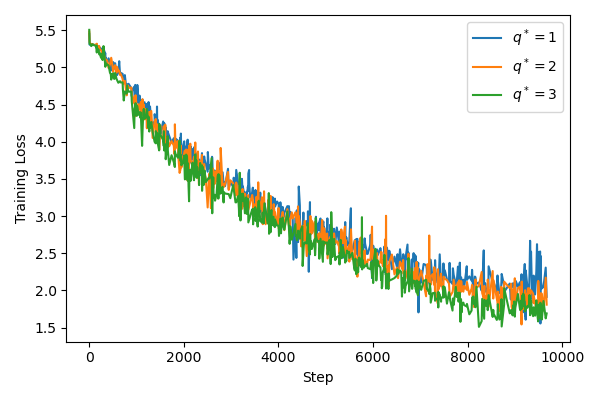}
    }
    \hfill
    \subfigure[$s=0.6$, $V'(q^*)=0.9$]{
        \includegraphics[height = 0.145\paperheight, width=0.25\linewidth]{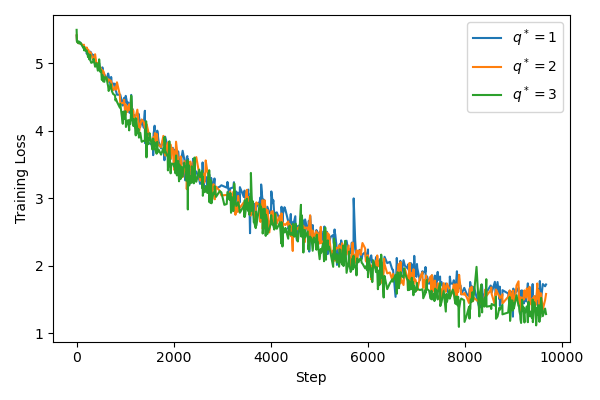}
    }
    \hfill
    \subfigure[$s=0.7$, $V'(q^*)=0.5$]{
        \includegraphics[height = 0.145\paperheight, width=0.25\linewidth]{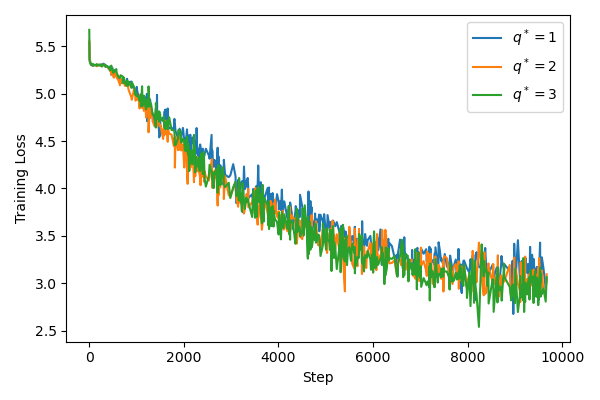}
    }
    \hfill
    \subfigure[$s=0.7$, $V'(q^*)=0.7$]{
        \includegraphics[height = 0.145\paperheight, width=0.25\linewidth]{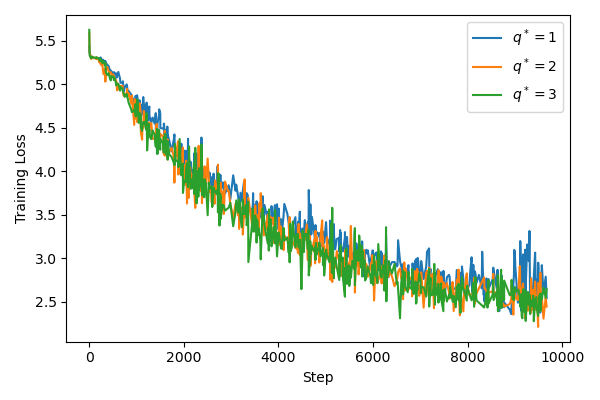}
    }
    \hfill
    \subfigure[$s=0.7$, $V'(q^*)=0.9$]{
        \includegraphics[height = 0.145\paperheight, width=0.25\linewidth]{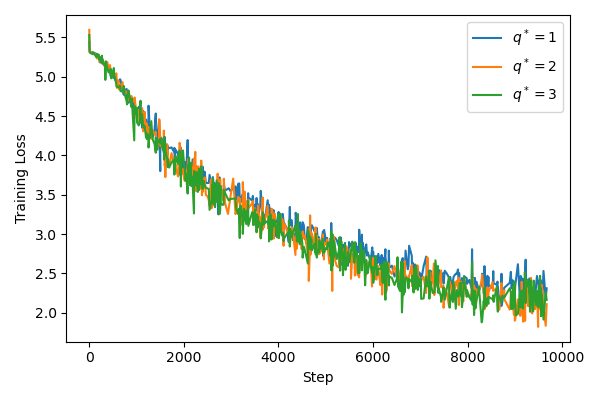}
    }
    \hfill
    \subfigure[$s=0.8$, $V'(q^*)=0.5$]{
        \includegraphics[height = 0.145\paperheight, width=0.25\linewidth]{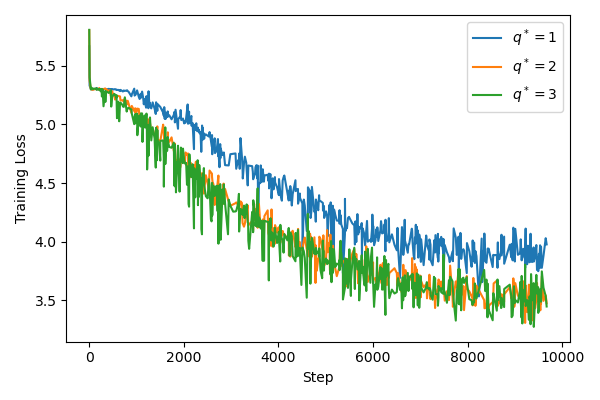}
    }
    \hfill
    \subfigure[$s=0.8$, $V'(q^*)=0.7$]{
        \includegraphics[height = 0.145\paperheight, width=0.25\linewidth]{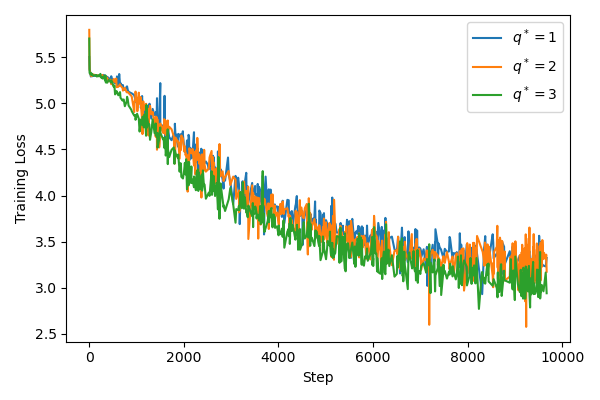}
    }
    \hfill
    \subfigure[$s=0.8$, $V'(q^*)=0.9$]{
        \includegraphics[height = 0.145\paperheight, width=0.25\linewidth]{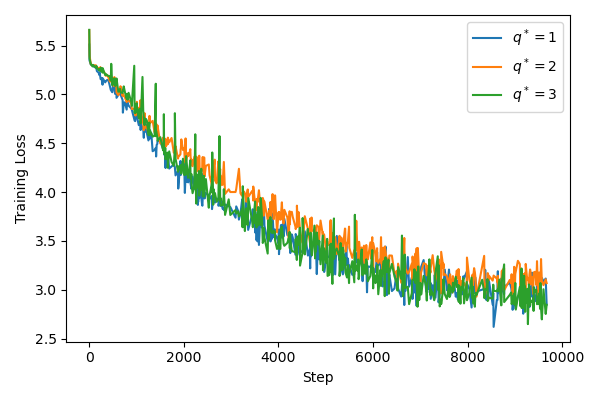}
    }
    \hfill
    \subfigure[$s=0.85$, $V'(q^*)=0.5$]{
        \includegraphics[height = 0.145\paperheight, width=0.25\linewidth]{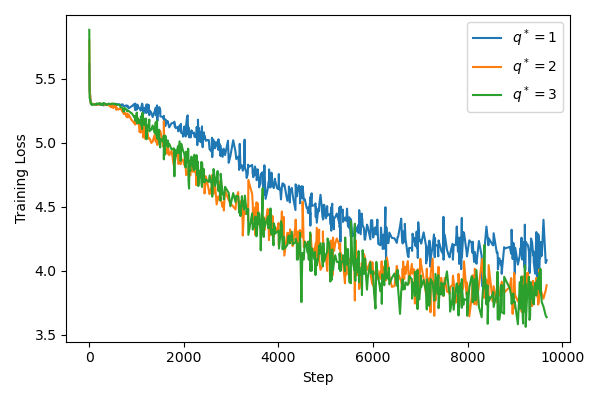}
    }
    \hfill
    \subfigure[$s=0.85$, $V'(q^*)=0.7$]{
        \includegraphics[height = 0.145\paperheight, width=0.25\linewidth]{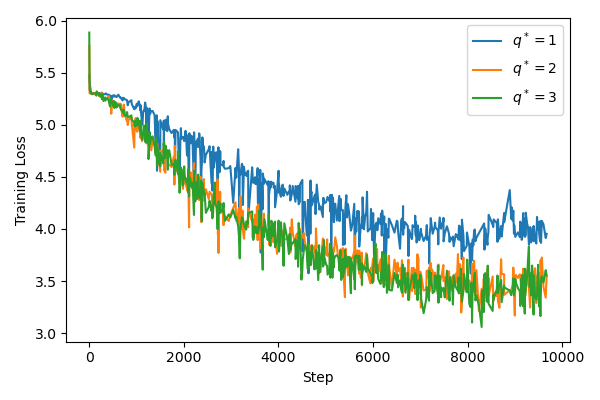}
    }
    \hfill
    \subfigure[$s=0.85$, $V'(q^*)=0.9$]{
        \includegraphics[height = 0.145\paperheight, width=0.25\linewidth]{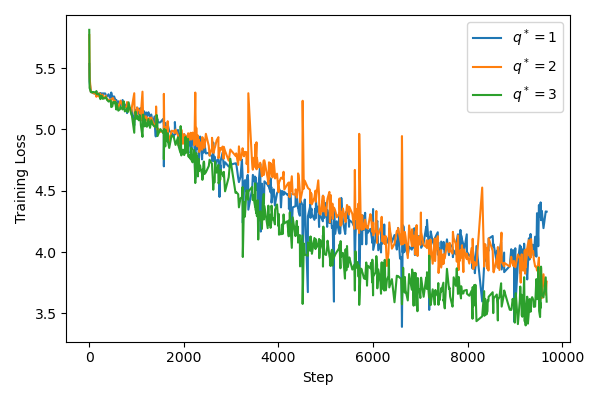}
    }
    \hfill
    \caption{Mean training loss across five seeds for a CNN trained on Tiny ImageNet, with activation function $\phi = \text{CReLU}_{\tau, m}$, $\tau$ and $m$ chosen such that $s$, $V'(q^*)$ are as described. Observe improved training speed for increased $q^*$ across all parameter sets.}
    \label{fig:mean_train_loss_crelu_tinyimagenet}
\end{figure}

\end{document}